\documentclass[journal]{IEEEtran}

\usepackage{graphicx} % more modern
\usepackage{subfigure}

\usepackage[numbers]{natbib}
\usepackage{color}
\usepackage[colorlinks=true,backref=page,citecolor=blue,linkcolor=magenta]{hyperref}

\newif\ifDoubleColumn
\DoubleColumntrue
\newif\ifIEEETN
\IEEETNfalse

\newcommand{\revised}[1]{#1}

\ifCLASSINFOpdf
\else
\fi
\usepackage[cmex10]{amsmath}
\usepackage{amssymb}
\usepackage{amsfonts}
\usepackage{amsthm}

\usepackage{algorithm}
\usepackage{algorithmic}
\newtheorem{theorem}{Theorem}% \ifDissertation[chapter]\fi
\newtheorem{lemma}[theorem]{Lemma}
\newtheorem{definition}{Definition}%\ifDissertation[chapter]\fi

\newtheorem{proposition}[theorem]{Proposition}
\theoremstyle{remark}
\newcounter{assumption}%[section]
\renewcommand{\theassumption}{A\arabic{assumption}}

\newenvironment{assumption}[1][]{\begin{trivlist}\item[] \refstepcounter{assumption}%
 \textbf{Assumption\ \theassumption\ #1} }{%\par\nobreak\noindent\sl\ignorespaces}{%
 \ifvmode\smallskip\fi\end{trivlist}}

\renewcommand{\AA}{{\mathcal{A}}}

\newcommand{\D}{{\mathcal D}}

\newcommand{\XX}{{\mathcal{X}}}

\newcommand{\XA}{\XX\times\AA}

\newcommand{\kfun}{\textsc{k}}

\newcommand{\beq}{\begin{equation}}
\newcommand{\eeq}{\end{equation}}
\newcommand{\beqa}{\begin{eqnarray}}
\newcommand{\eeqa}{\end{eqnarray}}
\newcommand{\beqan}{\begin{eqnarray*}}
\newcommand{\eeqan}{\end{eqnarray*}}
\newcommand{\ben}{\begin{eqnarray*}}
\newcommand{\een}{\end{eqnarray*}}

\newcommand{\norm}[1]{\left\Vert#1\right\Vert}
\newcommand{\smallnorm}[1]{\Vert#1\Vert}

\newcommand{\Real}{\mathbb R}

\newcommand{\Prob}[1]{{\mathbb P}\left\{#1\right\}}
\newcommand{\EE}[1]{{\mathbb E}\left[#1\right]}

\newcommand{\Var}[1]{{\mathrm{Var}}\left[#1\right]}

\newcommand{\one}[1]{{\mathbb I}_{\{#1\}}}

\newcommand{\eps}{\varepsilon}

\newcommand{\ra}{\rightarrow}

\newcommand{\FF}{{\mathcal{F}}}
\newcommand{\GG}{{\mathcal{G}}}

\newcommand{\cN}{{\mathcal{N}}}

\newcommand{\argmin}{\mathop{\textrm{argmin}}}
\newcommand{\argmax}{\mathop{\textrm{argmax}}}

\newcommand{\Qmax}{Q_{\textrm{max}}}
\newcommand{\Rmax}{R_{\textrm{max}}}

\newcommand{\eqdef}{\triangleq}

\newcommand{\actionnum}{{|\AA|}}

\newcommand{\MM}{\mathcal{M}}
\newcommand{\Qpi}{Q^\pi}

\newcommand{\Qhat}{\hat{Q}}

\newcommand{\Dn}{\mathcal{D}_n}

\newcommand{\Qopt}{Q^*}
\newcommand{\Vpi}{V^\pi}
\newcommand{\Vopt}{V^*}

\newcommand{\cset}[2]{\left\{\,#1\,:\,#2\,\right\}}

\newcommand{\gap}{\mathbf{g}}

\newcommand{\PKernel}{\mathcal{P}}

\newcommand{\Argmax}{\mathop{\textrm{Argmax}}}

\newcommand{\ProbWRT}[2]{{\mathbb P}_{#1}\left(#2\right)}
\newcommand{\One}[1]{{\mathbb I}{\{#1\}}}
\newcommand{\pihat}{{\hat{\pi}}}
\newcommand{\piopt}{{\pi^*}}

\newcommand{\QpiPrime}{{Q^{\pi'}}}
\newcommand{\QpiPrimehat}{{\hat{Q}^{\pi'}}}
\newcommand{\gapQPrime}{{\mathbf{g}_\QpiPrime}}
\newcommand{\gapQPrimehat}{{\mathbf{g}_\QpiPrimehat}}

\newcommand{\pioptInPi}{ {\pi^*_\Pi} }

\def\Pr{{\mathbf{P}}}

\newcommand{\Dkx}{\mathcal{D}^{(k)}(x)}

\newcommand{\mnec}{\ensuremath{\eta_{\mathrm{v}}}}
\newcommand{\mnea}{\ensuremath{\eta_{\mathrm{\pi}}}}

\graphicspath{{./fig/}}

\begin{document}
%
% paper title
% can use linebreaks \\ within to get better formatting as desired
% Do not put math or special symbols in the title.
\ifIEEETN
\title{Classification-based Approximate Policy Iteration}
\else
\title{Classification-based Approximate Policy Iteration: Experiments and Extended Discussions}
\fi
%
%
% author names and IEEE memberships
% note positions of commas and nonbreaking spaces ( ~ ) LaTeX will not break
% a structure at a ~ so this keeps an author's name from being broken across
% two lines.
% use \thanks{} to gain access to the first footnote area
% a separate \thanks must be used for each paragraph as LaTeX2e's \thanks
% was not built to handle multiple paragraphs
%

\author{Amir-massoud Farahmand,
        Doina Precup, \ifDoubleColumn \else \\ \fi
        Andr\'e~M.S. Barreto,
        Mohammad Ghavamzadeh% <-this % stops a space
\thanks{A.M. Farahmand was with the School of Computer Science, McGill University, Montreal, Canada. He is currently with the Robotics Institute, Carnegie Mellon University, Pittsburgh, USA (email: amirmf@andrew.cmu.edu).}
\thanks{D. Precup is with the School of Computer Science, McGill University, Montreal, Canada (email: dprecup@cs.mcgill.ca).}
\thanks{A.M.S. Barreto was with the School of Computer Science, McGill University, Montreal, Canada. He is currently with the National Laboratory for Scientific Computing (LNCC), Petr\'opolis, Brazil (e-mail: amsb@lncc.br).}
\thanks{M.~Ghavamzadeh is with Adobe Research, USA on leave of absence from INRIA Lille, France (email: mohammad.ghavamzadeh@inria.fr).}
%
%\thanks{A.M. Farahmand, D. Precup, and A.M.S. Barreto are with the School of Computer Science, McGill University, Montreal, Canada. A.M.S. Barreto is also with the National Laboratory for Scientific Computing (LNCC), Petr\'opolis, Brazil.
%M.~Ghavamzadeh is with Adobe Research, USA \& INRIA Lille - Team SequeL, France.}
\ifIEEETN \thanks{Manuscript received on November 18, 2013, and revised and resubmitted on July 1, 2014.}\fi}

\ifIEEETN
% The paper headers
\markboth{IEEE Transactions on Automatic Control,~Vol.~XX, No.~XX, XX~XXXX}
{Farahmand \MakeLowercase{\textit{et al.}}: Classification-based Approximate Policy Iteration}
\fi

% The only time the second header will appear is for the odd numbered pages
% after the title page when using the twoside option.
% 
% *** Note that you probably will NOT want to include the author's ***
% *** name in the headers of peer review papers.                   ***
% You can use \ifCLASSOPTIONpeerreview for conditional compilation here if
% you desire.

% If you want to put a publisher's ID mark on the page you can do it like
% this:
%\IEEEpubid{0000--0000/00\$00.00~\copyright~2012 IEEE}
% Remember, if you use this you must call \IEEEpubidadjcol in the second
% column for its text to clear the IEEEpubid mark.

% use for special paper notices
%\IEEEspecialpapernotice{(Invited Paper)}

% make the title area
\maketitle

% As a general rule, do not put math, special symbols or citations
% in the abstract or keywords.
\begin{abstract}
\ifIEEETN
Tackling large approximate dynamic programming or reinforcement learning problems requires methods that can exploit regularities of the problem in hand.
Most current methods are geared towards exploiting the regularities of either the value function or the policy.
We introduce a general classification-based approximate policy iteration (CAPI) framework that can exploit regularities of both.
We establish theoretical guarantees for the sample complexity of CAPI-style algorithms, which allow the policy evaluation step to be performed by a wide variety of algorithms, and can handle nonparametric representations of policies. Our bounds on the estimation error of the performance loss are tighter than existing results.\footnote{The CAPI framework has previously been presented at the European Workshop on Reinforcement Learning (no proceedings)~\citep{FarahmandEWRL2012} and the Multidisciplinary Conference on Reinforcement Learning and Decision Making (extended abstract)~\citep{FarahmandCAPIRLDM2013}. The current version includes the proofs and a significantly more detailed discussion of the results. 
An extended version, including experimental results, is available on arXiv~\citep{FarahmandCAPIExtended2014}. 
}
\else
Tackling large approximate dynamic programming or reinforcement learning problems requires methods that can exploit regularities, or intrinsic structure, of the problem in hand.
Most current methods are geared towards exploiting the regularities of either the value function or the policy.
We introduce a general classification-based approximate policy iteration (CAPI) framework, which encompasses a large class of algorithms that can exploit regularities of both the value function and the policy space, depending on what is advantageous.
This framework has two main components: a generic value function estimator and a classifier that learns a policy based on the estimated value function.
We establish theoretical guarantees for the sample complexity of CAPI-style algorithms, which allow the policy evaluation step to be performed by a wide variety of algorithms (including temporal-difference-style methods), and can handle nonparametric representations of policies. Our bounds on the estimation error of the performance loss are tighter than existing results. 
We also illustrate this approach empirically on several problems, including a large HIV control task.
\footnote{The CAPI framework has previously been presented at the European Workshop on Reinforcement Learning (no proceedings)~\citep{FarahmandEWRL2012} and the Multidisciplinary Conference on Reinforcement Learning and Decision Making (extended abstract)~\citep{FarahmandCAPIRLDM2013}. The current version includes the proofs, a significantly more detailed discussion of the results, and extensive experiments. The theoretical analysis part of this work has been submitted for publication~\citep{FarahmandCAPI2013}. }
\fi
\end{abstract}

% Note that keywords are not normally used for peerreview papers.

% For peer review papers, you can put extra information on the cover
% page as needed:
%\ifCLASSOPTIONpeerreview
% \begin{center} \bfseries EDICS Category: 3-BBND \end{center}
% \fi
%
% For peerreview papers, this IEEEtran command inserts a page break and
% creates the second title. It will be ignored for other modes.
\IEEEpeerreviewmaketitle

%%%%%%%%%%%%%%%%%%%%%%%%%%%%%%%%%%%%%%%%%%%%%%%
%%%%%%%%%%%%%%%%%%%%%%%%%%%%%%%%%%%%%%%%%%%%%%%
%%%%%%%%%%%%%%%%%%%%%%%%%%%%%%%%%%%%%%%%%%%%%%%
\section{Introduction}
\label{sec:CAPI-Introduction}
%!TEX root = CAPI-TAC.tex

%%%%%%%%%%%%%%%%%%%%%%%%%%%%%%%%%%%%%%%%%%%%%%%%
%%%%%%%%%%% Shortened version for IEEE Technical Note %%%%%%%%%%%%%%%
%%%%%%%%%%%%%%%%%%%%%%%%%%%%%%%%%%%%%%%%%%%%%%%%
\ifIEEETN

\IEEEPARstart{W}{e} consider the problem of finding a near-optimal policy (i.e., controller) for discounted Markov Decision Processes (MDPs) with large state space and finite action space~\citep{SzepesvariBook10} with an unknown model.
% We focus on the scenario where the MDP model is not known and we only have access to a batch of interaction data.
For problems with large state spaces (e.g., when the state space is $\Real^d$ with large $d$), 
finding a close-to-optimal policy is difficult (due to the so-called curse of dimensionality) unless one benefits from regularities, or special structure, of the problem in hand.
One group of algorithms developed in reinforcement learning (RL) and approximate dynamic programming (ADP) focuses on exploiting regularities of the \emph{value} function~\citep{FarahmandNIPS08,TaylorParr09,FarahmandACC09,GhavamzadehLazaricMunosHoffmanICML2011,FarahmandVPI2012}, while another group tries to benefit from regularities of the \emph{policy}~\citep{MarbachTsitsiklis2001,Cao2005,GhavamzadehEngel07PG}.
% These methods, however, exploit only one type of regularity (either value or policy).
The goal of this paper is to introduce and analyze a class of algorithms, which we call Classification-based Approximate Policy Iteration (CAPI), that can potentially benefit simultaneously from both types of regularities. %, instead of just one. %from both types of regularities.

Our approach is inspired by existing classification-based RL algorithms~\citep{LagoudakisParrICML2003,FernYoonGivan2006,LazaricGhavamzadehMunosDPI2010}.
These methods use rollout (i.e., Monte Carlo trajectories) to roughly estimate the action-value function of the current policy at several states. \revised{The estimates define a set of (noisy) greedy actions (positive examples) and non-greedy actions (negative examples), which are then fed to a classifier.}
The classifier ``generalizes'' the greedy action choices over the state space.
The procedure is repeated.

Classification-based methods can be interpreted as variants of Approximate Policy Iteration (API) that use rollouts to estimate the action-value function (policy evaluation) and then \emph{project} the greedy policy obtained at those points onto the predefined space of controllers (policy improvement).
% An interpretation of this approach is that it \emph{projects} the greedy policy of a Monte Carlo estimate of the action-value function onto the predefined space of controllers. The projection is performed by the classifier.

Although classification-based RL methods can benefit from regularities of the policy, the use of rollouts prevents generalization through the value function, which reduces data efficiency. % their use of rollouts prevents them to benefit from regularities of the value function.
This lack of generalization makes rollout-based estimators data-inefficient.
This is a concern in real problems, in which new samples may be expensive, e.g., in adaptive treatment strategies.
Moreover, one cannot easily use rollouts when only access to a batch of data is allowed and a generative model or simulator of the environment is not available.

\revised{
To address the limitation of rollout-based estimators, we propose the CAPI framework. CAPI generalizes the current classification-based algorithms by allowing the use any policy evaluation method including, but not limited to, rollout-based estimators (as in previous work~\citep{LagoudakisParrICML2003,LazaricGhavamzadehMunosDPI2010}), LSTD~\citep{LagoudakisParr03}, the policy evaluation version of Fitted Q-Iteration~\citep{Ernst05}, and their regularized variants~\citep{FarahmandNIPS08,FarahmandACC09}, as well as online methods for policy evaluation such as Temporal Difference learning.
This is a significant generalization of the existing classification-based RL algorithms, which become special cases of CAPI. Our theoretical results indicate that this extension is indeed sound.
CAPI uses a weighted loss instead of the conventional $0/1$-loss of classification, which may lead to surprisingly bad policies~\citep{FarahmandCAPIExtended2014}.
% The use of weighted loss ensures that the resulting policy closely follows the greedy policy in regions of the state space where the difference between the best action and the rest is considerable (so choosing the wrong action is costly), but pays less attention to regions where all actions are almost the same.
}

The main theoretical contribution of this paper is the finite-sample error analysis of CAPI-style algorithms, which allows \emph{general policy evaluation} algorithms, handles \emph{nonparametric} (in the sense used by e.g.,~\citep{Gyorfi02,Wasserman07}) policy spaces,
and provides a \emph{faster convergence rate for the estimation error} than existing results.
Using nonparametric policies is a significant extension of the work by~\citet{FernYoonGivan2006}, which is limited to finite policy spaces, and of~\citet{LazaricGhavamzadehMunosDPI2010} and~\citet{GabLazGhaSch2011}, which are limited to policy spaces with finite Vapnik-Chervonenkis (VC) dimension.
Our faster convergence rates are due to using a concentration inequality based on the powerful notion of \emph{local Rademacher complexity}~\cite{BartlettBousquetMendelson05}, which is known to lead to fast rates in supervised learning.

We also leverage the notion of \emph{action-gap regularity}~\citep{FarahmandNIPS2011}, which implies that choosing the right action at each state may not require a precise estimate of the action-value function.
When the action-gap regularity of a problem is favourable, the convergence rate of CAPI is faster than the convergence rate of the estimate of the action-value function (and without any such assumption, the convergence rate is the same).

Another theoretical contribution of this work is a new \emph{error propagation} result that shows 
that the errors at later iterations of CAPI play a more important role on the performance of the resulting policy.

\else
%%%%%%%%%%%%%%%%%%%%%%%%%%%%%%%%%%%%%%%%%%%%%%%%
%%%%%%%%%%%%%%%%%% Longer arXiv version %%%%%%%%%%%%%%%%%%%
%%%%%%%%%%%%%%%%%%%%%%%%%%%%%%%%%%%%%%%%%%%%%%%%
\IEEEPARstart{W}{e} consider the problem of finding a near-optimal policy (i.e., controller) for discounted Markov Decision Processes (MDP) with large state space and finite action space~\citep{Bertsekas96,Sutton98,SzepesvariBook10}.
We focus on the scenario where the MDP model is not known and we only have access to a batch of interaction data.
For problems with large state spaces (e.g., when the state space is $\Real^d$ with large $d$), 
finding a close-to-optimal policy is difficult (due to the so-called curse of dimensionality) unless one benefits from regularities, or special structure, of the problem in hand, e.g., smoothness or sparsity of the value function or the optimal policy.
Many successful algorithms developed in reinforcement learning (RL) and approximate dynamic programming (ADP) focus on exploiting regularities of the \emph{value} function, e.g.,~\citet{FarahmandNIPS08,FarahmandACC09,KolterNg09,TaylorParr09,GhavamzadehLazaricMunosHoffmanICML2011,FarahmandVPI2012}.
However, useful structure can also arise in the policy space. For instance, in many control problems, simple policies such as bang-bang or PID controllers can perform quite well if tuned properly.
Direct policy search algorithms and various policy gradient algorithms try to exploit such structure, e.g.,~\citet{BaxterBartlett01,MarbachTsitsiklis2001,Kakade01,Cao2005,GhavamzadehEngel07PG}.

The aforementioned methods exploit only one type of regularity (either value or policy), therefore they do not benefit from all potential regularities of a problem.
The goal of this paper is to introduce a class of algorithms, which we call Classification-based Approximate Policy Iteration (CAPI), that can potentially benefit from the regularities of both value function and policy.

The inspiration for our approach comes from existing classification-based RL algorithms, e.g.,~\citet{LagoudakisParrICML2003,FernYoonGivan2006,LiBulitkoGreiner2007,LazaricGhavamzadehMunosDPI2010}.
These methods use Monte Carlo trajectories to roughly estimate the action-value function of the current policy (i.e., the value of choosing a particular action at the current state and then following the policy) at several states. This approach is called a \emph{rollout}-based estimate by~\citet{TesauroGalperin1996} and is closely related, but not equivalent, to the rollout algorithms of~\citet{Bertsekas2005}.
\revised{In these methods, the rollout estimates at several points in the state space define a set of (noisy) greedy actions (positive examples) as well as non-greedy actions (negative examples), which are then fed to a classifier.}
The classifier ``generalizes'' the greedy action choices over the entire state space.
The procedure is repeated.

Classification-based methods can be interpreted as variants of Approximate Policy Iteration (API) that use rollouts to estimate the action-value function (policy evaluation step) and then \emph{project} the greedy policy obtained at those points onto the predefined space of controllers (policy improvement step).

In many problems, this approach is helpful for three main reasons.
First, good policies are sometimes simpler to represent and learn than good value functions.
Second, even a rough estimate of the value function is often sufficient to separate the best action from the rest, especially when the gap between the value of the greedy actions and the rest is large.
And finally, even if the best action estimates are noisy (due to value function imprecision), one can take advantage of powerful classification methods to smooth out the noise.

\revised{
Rollout-based estimator of the value function, however, does not generalize the value function over the state space, and instead produces estimates at a finite collection of points. This lack of generalization makes rollout-based estimators data-inefficient.
This is a big concern in real problems, in which new samples may be expensive or impossible to generate, e.g., adaptive treatment strategies or user dialogue systems.
Moreover, one cannot easily use rollouts when only access to a batch of data is allowed and a generative model or simulator of the environment is not available.

To address the limitation of rollout-based estimators, we propose the CAPI framework. CAPI generalizes the current classification-based algorithms by allowing the use any policy evaluation method including, but not limited to, rollout-based estimators (as in previous work~\citep{LagoudakisParrICML2003,LazaricGhavamzadehMunosDPI2010}), LSTD~\citep{LagoudakisParr03,LazaricGhavamzadehMunosLSPI2012}, modified Bellman Residual Minimization~\citep{AntosSzepesvariML08}, the policy evaluation version of Fitted Q-Iteration~\citep{Ernst05,Riedmiller05,BusoniuErnstDeSchutterBabuska10,FarahmandVPI2012}, and their regularized variants~\citep{FarahmandNIPS08,GhavamzadehLazaricMunosHoffmanICML2011,FarahmandACC09}, as well as online methods for policy evaluation such as Temporal Difference learning~\cite{Sutton98,TsitsiklisVanRoy97} and GTD~\cite{SuttonMaei09}.
This is a significant generalization of existing classification-based RL algorithms, which become special cases of CAPI. Our theoretical results indicate that this extension is indeed sound.

On a more technical note, the loss function used for the classification step of CAPI is different from the conventional $0/1$-loss of classification, and is weighted according to the difference between the value of the greedy actions and the selected action.
The $0/1$-loss penalizes all mistakes equally and does not consider the relative importance of different regions of the state space, which may lead to surprisingly bad policies (cf. Section~\ref{sec:CAPI-Experiments}). In contrast, the use of weighted loss ensures that the resulting policy closely follows the greedy policy in regions of the state space where the difference between the best action and the rest is considerable (so choosing the wrong action is costly), but pays less attention to regions where all actions are almost the same.
The choice of weighted loss in RL/ADP is not entirely new and has been used in the context of classification-based RL (\citet{LiBulitkoGreiner2007,LazaricGhavamzadehMunosDPI2010,GabLazGhaSch2011,ScherrerGhavamzadehGabillonGeist2012}) and elsewhere (Conservative Policy Iteration approach of~\citet{KakadeLangfordCPI2002} and a variant of Policy Search by Dynamic Programming of~\citet{BagnellKakadeNgSchneider2003}).
}

The main theoretical contribution of this paper is the finite-sample error analysis of CAPI-style algorithms, which allows \emph{general policy evaluation} algorithms, handles \emph{nonparametric}\footnote{\revised{In the sense used by e.g.,~\citet{Gyorfi02,Wasserman07}.}} policy spaces, and provides a \emph{faster convergence rate for the estimation error} than existing results.
Using nonparametric policies is a significant extension of the work by~\citet{FernYoonGivan2006}, which is limited to finite policy spaces, and of~\citet{LazaricGhavamzadehMunosDPI2010} and~\citet{GabLazGhaSch2011}, which are limited to policy spaces with finite Vapnik-Chervonenkis (VC) dimension.
Our faster convergence rates are due to using a concentration inequality based on the powerful notion of \emph{local Rademacher complexity}~\cite{BartlettBousquetMendelson05}, which is known to lead to fast rates in supervised learning.

We also leverage the notion of \emph{action-gap regularity}, recently introduced by~\citet{FarahmandNIPS2011}, which implies that choosing the right action at each state may not require a precise estimate of the action-value function.
When the action-gap regularity of the problem is favourable, the convergence rate of CAPI is faster than the convergence rate of the estimate of the action-value function (and without any assumption on that regularity, the convergence rate is the same).

Another theoretical contribution of this work is a new \emph{error propagation} result that shows that the errors at later iterations of CAPI play a more important role on the performance of the resulting policy. So, if one has finite resources (samples or computational time), it is better to spend effort on the estimation at later iterations (by using a better function approximator, more samples, etc).

We illustrate CAPI's flexibility on some standard toy problems, as well as on a large HIV control domain, which is known to be difficult.
%%%%%%%%%%%%%%%%%%%%%%%%%%%%%%%%%%%%%%%%%%%%%%%%
%%%%%%%%%%%%%%%%%%%%%%%%%%%%%%%%%%%%%%%%%%%%%%%%
%%%%%%%%%%%%%%%%%%%%%%%%%%%%%%%%%%%%%%%%%%%%%%%%
\fi %%% The if in beginning switching between IEEE Technical Note and arXiv

%%%%%%%%%%%%%%%%%%%%%%%%%%%%%%%%%%%%%%%%%%%%%%%
%%%%%%%%%%%%%%%%%%%%%%%%%%%%%%%%%%%%%%%%%%%%%%%
%%%%%%%%%%%%%%%%%%%%%%%%%%%%%%%%%%%%%%%%%%%%%%%
\section{Background and Notation}
\label{sec:CAPI-MDP}
%!TEX root = CAPI-TAC.tex

%%%%%%%%%%%%%%%%%%%%%%%%%%%%%%%%%%%%%%%%%%%%%%%
%%%%%%%%%%%%%%%%%%%%%%%%%%%%%%%%%%%%%%%%%%%%%%%
%%%%%%%%%%%%%%% IEEE TECHNICAL NOTE VERSION %%%%%%%%%%%%%
%%%%%%%%%%%%%%%%%%%%%%%%%%%%%%%%%%%%%%%%%%%%%%%
%%%%%%%%%%%%%%%%%%%%%%%%%%%%%%%%%%%%%%%%%%%%%%%
\ifIEEETN
We consider a {\em finite-action discounted MDP} $(\XX,\AA,\PKernel, \mathcal{R},\gamma)$, where $\XX$ is a measurable state space, $\AA$ is a finite set of actions, 
$\PKernel: \XA \to \MM(\XX)$ is the transition probability kernel, \revised{$\mathcal{R}: \XA \to \MM(\Real)$} is the reward kernel (with expected reward uniformly bounded by $\Rmax$), and $\gamma\in[0,1)$ is a discount factor.
We use rather standard notations and definitions (see e.g., \citep{FarahmandCAPIExtended2014,SzepesvariBook10}): $\pi: \XX \ra \AA$ is a (deterministic Markov stationary) policy, $\Vpi$ and $\Qpi$ are its value and action-value functions, and $\Vopt$ and $\Qopt$ are the optimal value and action-value functions (bounded by $\Qmax$).
A policy $\pi$ is \emph{greedy} w.r.t. an action-value function $Q$, denoted by 
$
\pi=\hat{\pi}(\cdot;Q),
$
if $\pi(x) = \argmax_{a\in\AA}Q(x,a)$ holds for all $x\in\XX$
(if there exist multiple maximizers, one of them is chosen in an arbitrary deterministic manner).
% A greedy policy w.r.t. $\Qopt$ is an optimal policy.

Our theoretical analysis will rely on the notion of action-gap regularity of an MDP~\citep{FarahmandNIPS2011}, which characterizes the complexity of a control problem.  For simplicity, we define and analyze the two-action case, \revised{but the CAPI framework naturally accommodates MDPs with more actions, as we explain below}.

Consider an MDP with two actions.
For any $Q:\XX\times\AA \rightarrow \Real$, the action-gap function is defined as
$\gap_Q(x) \eqdef | Q(x,1) - Q(x,2) |$ for all $x\in\XX$.
To understand why the action-gap function is informative, suppose that we have an estimate $\Qhat^{\pi}$ of $\Qpi$ and we want to perform policy improvement based on $\Qhat^\pi$.
The greedy policy w.r.t. \revised{$\Qhat^\pi$}, i.e., $\pihat(\cdot;\Qhat^\pi)$, should ideally be close to the greedy policy w.r.t. $\Qpi$, i.e., $\pihat(\cdot;\Qpi)$.
If the action-gap $\gap_{\Qpi}(x)$ is large for some state $x$, the regret of choosing an action different from $\pihat(x;\Qpi)$, roughly speaking, is large; however, confusing the best action with the other one is also less likely. If the action-gap is small, a confusion is more likely to arise, but the regret stemming from the wrong choice will be small.
To characterize how difficult a problem is, we need to summarize the behaviour of the action-gap function over the entire state space.
This is done in the following assumption.
%%%%%%%%%%%%%%%%%%%%%%%%%%%%%%%%%%%%%%%%%%%%%%%
%%%%%%%%%%%%%%%%%%% Assumption %%%%%%%%%%%%%%%%%%%%%
%%%%%%%%%%%%%%%%%%%%%%%%%%%%%%%%%%%%%%%%%%%%%%%
\begin{assumption}[(Action-Gap).]
\label{ass:CAPI-ActionGap}
For a fixed MDP $(\XX, \AA, \PKernel, \mathcal{R}, \gamma)$ with $\actionnum = 2$ and a fixed distribution over states $\nu\in\MM(\XX)$, there exist constants $c_g > 0$ and $\zeta \geq 0$ such that for any $\pi \in \Pi$ and all $\eps > 0$, we have
\ifIEEETN
$	\ProbWRT
		{\nu} 
		{0 < \gap_{\Qpi} (X) \leq \eps}
		 \eqdef
		\int_{\XX} \One{0 < \gap_{\Qpi} (x) \leq \eps} \, \mathrm{d} \nu(x)
	\leq c_g \, \eps^\zeta$.
\else
	\begin{align*}
		\ProbWRT
		{\nu} 
		{0 < \gap_{\Qpi} (X) \leq \eps}
		& \eqdef
		\int_{\XX} \One{0 < \gap_{\Qpi} (x) \leq \eps} \, \mathrm{d} \nu(x)
	%	\\& 
		\leq c_g \, \eps^\zeta.
	\end{align*}
\fi
\end{assumption}
The value of $\zeta$ controls the distribution of the action-gap $\gap_{\Qpi}(X)$.
A large value of $\zeta$ indicates that the probability of $\Qpi(X,1)$ being very close to $\Qpi(X,2)$ is small.
This implies that the estimate $\hat{Q}^\pi$ can be quite inaccurate in a large subset of the state space (measured according to $\nu$), 
but $\pihat(\cdot;\hat{Q}^{\pi})$ would still be the same as $\pihat(\cdot;\Qpi)$.
Note that any MDP satisfies the inequality when $\zeta = 0$ and $c_g = 1$, so the class of MDPs satisfying this property is not restricted in any way.

Finally, the $L_\infty$-norm on $\XA$ is defined as $\norm{Q}_\infty \eqdef \sup_{(x,a) \in \XA} |Q(x,a)|$. We also use a definition of supremum norm that holds only on a set of points from $\XX$. Let $\Dn = \{X_1, \dotsc, X_n\}$; then, $\norm{Q}_{\infty,\Dn} \eqdef \max_{x \in \Dn, a \in \AA} |Q(x,a)|$.

%%%%%%%%%%%%%%%%%%%%%%%%%%%%%%%%%%%%%%%%%%%%%%%
%%%%%%%%%%%%%%%%%%%%%%%%%%%%%%%%%%%%%%%%%%%%%%%
%%%%%%%%%%%%%% EXTENDED VERSION on arXiV %%%%%%%%%%%%%%%%
%%%%%%%%%%%%%%%%%%%%%%%%%%%%%%%%%%%%%%%%%%%%%%%
%%%%%%%%%%%%%%%%%%%%%%%%%%%%%%%%%%%%%%%%%%%%%%%
\else
In this section, we summarize necessary definitions and notation.
For more information, we refer the reader to~\citet{Bertsekas96,Sutton98,SzepesvariBook10}.

\subsection{Markov Decision Processes}
For a space $\Omega$ with $\sigma$-algebra $\sigma_\Omega$, $\MM(\Omega)$ denotes the set of all probability measures over $\sigma_\Omega$.
The space of bounded measurable functions with respect to (w.r.t.) $\sigma_\Omega$ is denoted by $B(\Omega)$ and $B(\Omega,L)$ denotes the subset of $B(\Omega)$ with bound $0 < L < \infty$.

A {\em finite-action discounted MDP} is a 5-tuple $(\XX,\AA,\PKernel, \mathcal{R},\gamma)$, where $\XX$ is a measurable state space, $\AA$ is a finite set of actions, 
$\PKernel: \XA \to \MM(\XX)$ is the transition probability kernel, \revised{$\mathcal{R}: \XA \to \MM(\Real)$} is the reward kernel, and $\gamma\in[0,1)$ is a discount factor.
Assume that the expected reward is uniformly bounded by $\Rmax$.
A measurable mapping $\pi: \XX \ra \AA$ is called a deterministic Markov stationary policy, or just {\em policy} for short. Following a policy $\pi$ means that at each time step, $A_t = \pi(X_t)$.

A policy $\pi$ induces the transition probability kernel $\PKernel^\pi: \XX \to \MM(\XX)$.
For a measurable subset $S \subseteq \XX$, we define
$(\PKernel^\pi)(S|x) \eqdef \int \PKernel(\mathrm{d}y|x,\pi(x)) \one{y \in S}$, in which $\one{\cdot}$ is the indicator function.
The $m$-step transition probability kernels $(\PKernel^\pi)^m: \XX \to \MM(\XX)$ for $m = 2,3,\cdots$ are inductively defined as
$(\PKernel^\pi)^m( S |x) \eqdef \int_{\XX} \PKernel(\mathrm{d}y|x,\pi(x)) (\PKernel^\pi)^{m-1}(S | y)$.

Given a transition probability kernel \revised{$\PKernel': \XX \to \MM(\XX)$}, the right-linear operator $\PKernel'\cdot: B(\XX) \to B(\XX)$ is defined as
$(\PKernel' V)(x) \eqdef \int_{\XX} \PKernel'(\mathrm{d}y|x) V(y)$.
In other words, $(\PKernel' V)(x)$ is the expected value of $V$ w.r.t. the distribution induced by following $\PKernel'$ from state $x$.
For a probability distribution $\rho \in \MM(\XX)$ and a measurable subset $S \subseteq \XX$, 
let the left-linear operators $\cdot \PKernel': \MM(\XX) \to \MM(\XX)$ be
$(\rho \PKernel')(S) = \int \rho(dx) \PKernel'(\mathrm{d}y|x) \one{y \in S}$.
In other words, $(\rho \PKernel')$ is the distribution induced by $\PKernel'$ when the initial distribution is $\rho$.
In this paper, $\PKernel'$ is usually $(\PKernel^\pi)^m: \MM(\XX) \to \MM(\XX)$, for $m=1,2,\dotsc$.

The value function $\Vpi$ and the action-value function $\Qpi$ of a policy $\pi$ are defined as follows: 
Let $(R_t;t\ge 1)$ be the sequence of rewards when the Markov chain is started from state $X_1$ (or state-action $(X_1,A_1)$ for $\Qpi$) drawn from a positive probability distribution over $\XX$ ($\XA$) and the agent follows the policy $\pi$.
Then
$V^\pi(x)  \eqdef  \EE{\sum_{t=1}^\infty\gamma^{t-1} R_t \,\Big|\, X_1=x}$
and
$Q^\pi(x,a) \eqdef  \EE{\sum_{t=1}^\infty\gamma^{t-1} R_t \,\Big|\,X_1=x,A_1=a}$.
The functions $\Vpi$ and $\Qpi$ are uniformly bounded by $\Qmax = \Rmax/(1-\gamma)$, independent of the choice of $\pi$.

The {\em optimal value} and {\em optimal action-value} functions are defined as $V^*(x)  = \sup_\pi V^\pi(x)$ for all $x \in \XX$ and $Q^*(x,a) = \sup_\pi Q^\pi(x,a)$ for all $(x,a) \in \XA$.
A policy $\piopt$ is {\em optimal} if $V^{\piopt} = \Vopt$.
A policy $\pi$ is \emph{greedy} w.r.t. an action-value function $Q$, denoted by 
$
\pi=\hat{\pi}(\cdot;Q),
$
if $\pi(x) = \argmax_{a\in\AA}Q(x,a)$ holds for all $x\in\XX$
(if there exist multiple maximizers, one of them is chosen in an arbitrary deterministic manner).
A greedy policy w.r.t. the optimal action-value function $\Qopt$ is an optimal policy.

\subsection{Action-Gap Characterization of the MDP}

The \emph{action-gap} regularity~\citep{FarahmandNIPS2011} is a recently introduced complexity measure of a control problem, inspired by the low-noise condition in the classification literature~\citep{AudibertTsybakov2007}.
Our theoretical analysis will rely on the notion of action-gap regularity of an MDP~\citep{FarahmandNIPS2011}, which characterizes the complexity of a control problem.  For simplicity, we define and analyze the two-action case, \revised{but the CAPI framework naturally accommodates MDPs with more actions, as we explain below}.

Consider an MDP with two actions.
For any $Q:\XX\times\AA \rightarrow \Real$, the action-gap function is defined as
\[
	\gap_Q(x) \eqdef | Q(x,1) - Q(x,2) | \qquad \text{for all }x\in\XX.
\]
To understand why the action-gap function is informative, suppose that we have an estimate $\Qhat^{\pi}$ of $\Qpi$ and we want to perform policy improvement based on $\Qhat^\pi$.
The greedy policy w.r.t. \revised{$\Qhat^\pi$}, i.e., $\pihat(\cdot;\Qhat^\pi)$, should ideally be close to the greedy policy w.r.t. $\Qpi$, i.e., $\pihat(\cdot;\Qpi)$.
If the action-gap $\gap_{\Qpi}(x)$ is large for some state $x$, the regret of choosing an action different from $\pihat(x;\Qpi)$, roughly speaking, is large; however, confusing the best action with the other one is also less likely. If the action-gap is small, a confusion is more likely to arise, but the regret stemming from the wrong choice will be small.

To characterize how difficult a problem is, we need to summarize the behaviour of the action-gap function over the entire state space.
This is done in the following assumption.
%%%%%%%%%%%%%%%%%%%%%%%%%%%%%%%%%%%%%%%%%%%%%%%
%%%%%%%%%%%%%%%%%%% Assumption %%%%%%%%%%%%%%%%%%%%%
%%%%%%%%%%%%%%%%%%%%%%%%%%%%%%%%%%%%%%%%%%%%%%%
\begin{assumption}[(Action-Gap).]
\label{ass:CAPI-ActionGap}
For a fixed MDP $(\XX, \AA, \PKernel, \mathcal{R}, \gamma)$ with $\actionnum = 2$ and a fixed distribution over states $\nu\in\MM(\XX)$, there exist constants $c_g > 0$ and $\zeta \geq 0$ such that for any $\pi \in \Pi$ and all $\eps > 0$, we have
\ifIEEETN
$	\ProbWRT
		{\nu} 
		{0 < \gap_{\Qpi} (X) \leq \eps}
		 \eqdef
		\int_{\XX} \One{0 < \gap_{\Qpi} (x) \leq \eps} \, \mathrm{d} \nu(x)
	\leq c_g \, \eps^\zeta$.
\else
	\begin{align*}
		\ProbWRT
		{\nu} 
		{0 < \gap_{\Qpi} (X) \leq \eps}
		& \eqdef
		\int_{\XX} \One{0 < \gap_{\Qpi} (x) \leq \eps} \, \mathrm{d} \nu(x)
	%	\\& 
		\leq c_g \, \eps^\zeta.
	\end{align*}
\fi
\end{assumption}
The value of $\zeta$ controls the distribution of the action-gap $\gap_{\Qpi}(X)$.
A large value of $\zeta$ indicates that the probability of $\Qpi(X,1)$ being very close to $\Qpi(X,2)$ is small.
This implies that the estimate $\hat{Q}^\pi$ can be quite inaccurate in a large subset of the state space (measured according to $\nu$), 
but $\pihat(\cdot;\hat{Q}^{\pi})$ would still be the same as $\pihat(\cdot;\Qpi)$.
Note that any MDP satisfies the inequality when $\zeta = 0$ and $c_g = 1$, so the class of MDPs satisfying this property is not restricted in any way. \revised{MDPs with $\zeta = 0$, however, are quite ``boring'' as it implies that $\gap_{\Qpi}(x) = 0$ and so $\Qpi(x,1) = \Qpi(x,2)$ for all $x \in \XX$ ($\nu$-almost surely). Trying to find policies to control these MDPs is futile after all.}
Also, note that one could characterize the distribution of the action-gap function in other ways too, e.g., upper bounds in a form other than $O(\eps^\zeta$). The current form, however, simplifies the analysis while succinctly showing the effect of the action-gap distribution.

The $L_\infty$-norm on $\XA$ is defined as $\norm{Q}_\infty \eqdef \sup_{(x,a) \in \XA} |Q(x,a)|$. We also use a definition of supremum norm that holds only on a set of points from $\XX$. Let $\Dn = \{X_1, \dotsc, X_n\}$ with $X_i \in \XX$; then, $\norm{Q}_{\infty,\Dn} \eqdef \max_{x \in \Dn, a \in \AA} |Q(x,a)|$.

\fi

%%%%%%%%%%%%%%%%%%%%%%%%%%%%%%%%%%%%%%%%%%%%%%%
%%%%%%%%%%%%%%%%%%%%%%%%%%%%%%%%%%%%%%%%%%%%%%%
%%%%%%%%%%%%%%%%%%%%%%%%%%%%%%%%%%%%%%%%%%%%%%%
\section{CAPI Framework}
\label{sec:CAPI-Algorithm}
%!TEX root = CAPI-TAC.tex

%%%%%%%%%%%%%%%%%%%%%%%%%%%%%%%%%%%%%%%%%%%%%%%%
%%%%%%%%%%% Shortened version of IEEE Technical Note %%%%%%%%%%%%%%%
%%%%%%%%%%%%%%%%%%%%%%%%%%%%%%%%%%%%%%%%%%%%%%%%
\ifIEEETN
CAPI is an approximate policy iteration framework that takes a policy space $\Pi$, a distribution over states $\nu\in\MM(\XX)$, and the number of iterations $K$ as inputs, and returns a policy whose performance should be close to the best policy in $\Pi$ (Figure~\ref{alg:CAPI}).
\text{PolicyEval} can be any algorithm that computes an estimate $\hat{Q}^\pi$ of $Q^\pi$, including all policy evaluation methods mentioned in the Introduction.

\begin{figure}[t]
{\bf Algorithm} CAPI$(\Pi,\nu,K)$
\begin{small}
\begin{algorithmic}
\STATE \textbf{Input:} Policy space $\Pi$, State distribution $\nu$, Number of iterations $K$
\STATE \textbf{Initialize:} Let $\pi_{(0)}\in\Pi$ be an arbitrary policy
\FOR{$k = 0,1,\dotsc,K-1$}
\STATE Construct a dataset $\D_n^{(k)}=\{X_i\}_{i=1}^n,\;X_i\stackrel{\text{i.i.d.}}{\sim}\nu$
\STATE $\hat{Q}^{\pi_{k}} \leftarrow  \text{PolicyEval}(\pi_{k})$
\STATE $\pi_{k+1} \leftarrow \argmin_{\pi\in\Pi}\hat{L}^{\pi_{k}}_n(\pi)$ {\footnotesize (action-gap-weighted classification)}
\ENDFOR
\end{algorithmic}
\end{small}
%\vspace*{-0.7cm}
\caption{CAPI pseudocode}
\label{alg:CAPI}
\end{figure}

\revised{
Exploiting the intuition given by the action-gap phenomenon~\citep{FarahmandNIPS2011}, which entails that when $\gap_{\Qpi}(x)$ is large at some state $x$, the regret of choosing an action different from $\pihat(x;\Qpi)$ is also large,}
the approximate policy improvement step of CAPI at each iteration $k$ is performed by minimizing the following action-gap-weighted empirical loss function in policy space $\Pi$:%%% %%
\revised{
\begin{align}
\label{eq:CAPI-emp-loss}
\hat{L}^{\pi_{k}}_n(\pi) & \eqdef \int_{\XX} \gap_{\hat{Q}^{\pi_{k}}}(x) \One{\pi(x) \neq \argmax_{a \in \AA} \hat{Q}^{\pi_{k}}(x,a) } \, \mathrm{d} \nu_n \\
\nonumber
	& = 
\sum_{X_i \in \Dn^{(k)} }
\gap_{\hat{Q}^{\pi_{k}}}(X_i) \One{\pi(X_i) \neq \argmax_{a \in \AA} \hat{Q}^{\pi_{k}}(X_i,a) }, 
\end{align}
}
where $\nu_n$ is the empirical distribution induced by the samples in $\Dn^{(k)} =\{X_i\}_{i=1}^n$ with $X_i \sim \nu$, \revised{i.e., $\nu_n  = \frac{1}{n} \sum_{X_i \in \Dn^{(k)} } \delta_{X_i}$ with $\delta_{X_i}$ being a point mass at $X_i$ for $i=1, \dotsc, n$.
This loss function emphasizes states in which the regret of choosing a non-greedy action is large.}
The policy improvement step of CAPI is defined by
\begin{equation}
\label{eq:CAPI-OptimizationProblem}
	\pi_{k+1} \leftarrow \argmin_{\pi\in\Pi}\hat{L}^{\pi_{k}}_n(\pi)
\end{equation}
Policy $\pi_{k+1}$ is the projection of the greedy policy $\pihat(\cdot;\Qhat^{\pi_{k}})$, defined only at points $\Dn^{(k)}$, onto policy space $\Pi$ when the distance measure is weighted according to the estimated action-gap function $\gap_{\Qhat^{\pi_{k}}}$.
This should be contrasted with the conventional classification-based approaches~\citep{LagoudakisParrICML2003}, which use a uniform weight for all states, i.e., they minimize $\int_{\XX} \One{\pi(x) \neq \argmax_{a \in \AA} \hat{Q}^{\pi_{k}}(x,a) }  \mathrm{d} \nu_n$.
Note that the loss~\eqref{eq:CAPI-emp-loss} is also used by~\citep{LazaricGhavamzadehMunosDPI2010,GabLazGhaSch2011}.
In~\citep{FarahmandCAPIExtended2014}, we discuss why uniformly weighted loss might lead to a bad choice of policies and provide some empirical evidence too.

\revised{
The flexibility in the choice of policy space $\Pi$ and \text{PolicyEval} allows benefitting from regularities of both policy and value function.
The policy space can be a parametric function space, which is described by a fixed finite number of parameters, or a nonparametric space, which grows with data~\citep{Gyorfi02,Wasserman07}.
The flexibility in the choice of \text{PolicyEval} enables CAPI to exploit regularities of the value function, such as smoothness, which is impossible with a rollout-based estimator.
The optimal choices for PolicyEval and $\Pi$ are problem-dependent and should ideally be determined by a model selection method~\citep{FarahmandSzepesvariMLJ11}.
}

\revised{The dataset used by PolicyEval to generate $\Qhat^{\pi_k}$, in general, is different from $\Dn^{(k)}$ used in~\eqref{eq:CAPI-emp-loss}.
In practice, however, one might use the same dataset for both. It is also possible to change the sampling distribution $\nu$ at each iteration, e.g., similar to~\citep{RossGordonBagnell2011}.
Reusing the same dataset or changing the sampling distribution is not analyzed here.
}

\revised{
To extend the current loss function to problems with $\actionnum > 2$, one can define the action-gap function as $\gap_{Q}(x,a) \eqdef \max_{a' \in \AA} Q(x,a') - Q(x,a)$.
The empirical loss function would be
%
%\begin{align*}
$
\hat{L}^{\pi_{k}}_n(\pi) \eqdef \int_{\XX} \gap_{\hat{Q}^{\pi_{k}}}(x,\pi(x)) \One{\pi(x) \neq \argmax_{a \in \AA} \hat{Q}^{\pi_{k}}(x,a) } \, \mathrm{d} \nu_n$.
%\end{align*}
Our theoretical analysis, however, does not cover this case.
}

%%%%%%%%%%%%%%%%%%%%%%%%%%%%%%%%%%%%%%%%%%%%%%%
%%%%%%%%%% IEEE TAC Version -- No detail on KNN-CAPI %%%%%%%%%%%%%
%%%%%%%%%%%%%%%%%%%%%%%%%%%%%%%%%%%%%%%%%%%%%%%
\revised{
Since the loss function~\eqref{eq:CAPI-emp-loss} is non-convex, solving~\eqref{eq:CAPI-OptimizationProblem} is computationally difficult for some policy spaces. But for local methods such as action-gap-weighted K-Nearest Neighbour or decision trees, one can get simple and computationally efficient rules~\citep{FarahmandCAPIExtended2014}. Another possibility is to relax the non-convex loss with a convex surrogate such as action-gap-weighted hinge or exponential loss.}

\else
%%%%%%%%%%%%%%%%%%%%%%%%%%%%%%%%%%%%%%%%%%%%%%%%
%%%%%%%%%%%%%%%%%% Longer arXiv version %%%%%%%%%%%%%%%%%%%
%%%%%%%%%%%%%%%%%%%%%%%%%%%%%%%%%%%%%%%%%%%%%%%%
CAPI is an approximate policy iteration framework that takes a policy space $\Pi$, a distribution over states $\nu\in\MM(\XX)$, and the number of iterations $K$ as inputs, and returns a policy whose performance should be close to the best policy in $\Pi$. Its outline is presented in Figure~\ref{alg:CAPI}.

\begin{figure}
{\bf Algorithm} CAPI$(\Pi,\nu,K)$
%\begin{small}
\begin{algorithmic}
\STATE \textbf{Input:} Policy space $\Pi$, State distribution $\nu$, Number of iterations $K$
\STATE \textbf{Initialize:} Let $\pi_{(0)}\in\Pi$ be an arbitrary policy
\FOR{$k = 0,1,\dotsc,K-1$}
\STATE Construct a dataset $\D_n^{(k)}=\{X_i\}_{i=1}^n,\;X_i\stackrel{\text{i.i.d.}}{\sim}\nu$
\STATE $\hat{Q}^{\pi_{k}} \leftarrow  \text{PolicyEval}(\pi_{k})$
\STATE $\pi_{k+1} \leftarrow \argmin_{\pi\in\Pi}\hat{L}^{\pi_{k}}_n(\pi)$ { (action-gap-weighted classification)}
\ENDFOR
\end{algorithmic}
%\end{small}
%\vspace*{-0.7cm}
\caption{CAPI pseudocode}
\label{alg:CAPI}
\end{figure}

\text{PolicyEval} can be any algorithm that computes an estimate $\hat{Q}^\pi$ of $Q^\pi$, including: rollout-based estimation~\cite{LagoudakisParrICML2003,LazaricGhavamzadehMunosDPI2010}, LSTD-Q~\cite{LagoudakisParr03,FarahmandNIPS08}, modified Bellman Residual Minimization~\citep{AntosSzepesvariML08}, and Fitted Q-Iteration~\cite{Ernst05,Riedmiller05,FarahmandACC09}, or a combination of rollouts and function approximation~\cite{GabLazGhaSch2011}, \revised{as well as online algorithms such as TD~\cite{TsitsiklisVanRoy97} and GTD~\cite{SuttonMaei09}}.

\revised{
Exploiting the intuition given by the action-gap phenomenon~\citep{FarahmandNIPS2011}, which entails that when $\gap_{\Qpi}(x)$ is large at some state $x$, the regret of choosing an action different from $\pihat(x;\Qpi)$ is also large,}
the approximate policy improvement step of CAPI at each iteration $k$ is performed by minimizing the following action-gap-weighted empirical loss function in policy space $\Pi$:\footnote{\revised{The set of actions maximizing $\Qhat^{\pi_k}(x,\cdot)$ might have more than one element, so it would be more accurate to write $\One{\pi(X_i) \notin \Argmax_{a \in \AA} \hat{Q}^{\pi_{k}}(X_i,a) }$. To keep the notation simple, we suppose that there is only one maximizer.}
}
\revised{
\begin{align}
\label{eq:CAPI-emp-loss}
\hat{L}^{\pi_{k}}_n(\pi) & \eqdef \int_{\XX} \gap_{\hat{Q}^{\pi_{k}}}(x) \One{\pi(x) \neq \argmax_{a \in \AA} \hat{Q}^{\pi_{k}}(x,a) } \, \mathrm{d} \nu_n \\
\nonumber
	& = 
\sum_{X_i \in \Dn^{(k)} }
\gap_{\hat{Q}^{\pi_{k}}}(X_i) \One{\pi(X_i) \neq \argmax_{a \in \AA} \hat{Q}^{\pi_{k}}(X_i,a) }, 
\end{align}
}
where $\nu_n$ is the empirical distribution induced by the samples in $\Dn^{(k)} =\{X_i\}_{i=1}^n$ with $X_i \sim \nu$, \revised{i.e., $\nu_n  = \frac{1}{n} \sum_{X_i \in \Dn^{(k)} } \delta_{X_i}$ where $\delta_{X_i}$ is a point mass at $X_i$ for $i=1, \dotsc, n$.
This loss function emphasizes states in which the regret of choosing a non-greedy action is large.}
The policy improvement step of CAPI is defined by the following optimization problem:
\begin{equation}
\label{eq:CAPI-OptimizationProblem}
	\pi_{k+1} \leftarrow \argmin_{\pi\in\Pi}\hat{L}^{\pi_{k}}_n(\pi)
\end{equation}
Policy $\pi_{k+1}$ is the projection of the greedy policy $\pihat(\cdot;\Qhat^{\pi_{k}})$, defined only at points $\Dn^{(k)}$, onto policy space $\Pi$ when the distance measure is weighted according to the estimated action-gap function $\gap_{\Qhat^{\pi_{k}}}$.
This should be contrasted with the conventional classification-based approaches~\citep{LagoudakisParrICML2003}, which use a uniform weight for all states, i.e., they minimize $\int_{\XX} \One{\pi(x) \neq \argmax_{a \in \AA} \hat{Q}^{\pi_{k}}(x,a) }  \mathrm{d} \nu_n$.
Note that the loss~\eqref{eq:CAPI-emp-loss} is also used by~\citet{LazaricGhavamzadehMunosDPI2010,GabLazGhaSch2011}.

A uniformly weighted loss might lead to a bad choice of policies, as it does not take into account the relative importance of different regions in the state space. This is especially a concern if the greedy policy $\pihat(\cdot;\Qhat^{\pi_{k}})$ does not belong to $\Pi$, so that there are some points in the dataset for which $\One{\pi(x) \neq \argmax_{a \in \AA} \hat{Q}^{\pi_{k}}(x,a) }$ is nonzero. To simplify the discussion, suppose there are only two points $x_1$ and $x_2$. 
The uniformly weighted loss does not differentiate between these two points, regardless of their action-gap. Nonetheless, the regret of not following the greedy policy at a point with a large action-gap is worse.

\revised{
The CAPI framework is flexible in the choice of policy space $\Pi$. The policy space can be a parametric function space, which is described by a fixed finite number of parameters, or a nonparametric space, which grows with data. Examples of latter are spaces defined by local methods (K-Nearest Neighbourhood) and decision trees that grow by data, and reproducing kernel Hilbert spaces. Refer to~\citet{Gyorfi02,Wasserman07} for the detailed discussion of nonparametric estimators in statistics.
}

\revised{
The choice of \text{PolicyEval} allows benefitting from regularities of the value function, such as its smoothness or its sparsity in a certain basis functions.
By the right choice of PolicyEval, which is determined by the function approximation architecture and the estimation method, we can provide a better estimate of $Q^{\pi_k}$ compared to what is achievable by a rollout-based policy evaluation algorithm. Rollout-based estimate does not generalize the estimate of the value function over the state space, so is incapable of benefiting from, e.g., the smoothness of the value function.
The accuracy of policy evaluation estimation, which is used in the approximate policy improvement step defined in~\eqref{eq:CAPI-OptimizationProblem}, affects the overall performance of the algorithm (cf. Theorems~\ref{thm:CAPI-ErrorInEachIteration} and \ref{thm:CAPI-PerformanceLoss}).
We will also see that another important factor in the performance of the algorithm is the choice of policy space $\Pi$. If $\Pi$ matches the regularity of the policy, we achieve better error upper bounds.
PolicyEval and $\Pi$ should ideally be chosen by an automatic model selection algorithm~\cite{FarahmandSzepesvariMLJ11}.
}

\revised{Note that we have not specified what dataset PolicyEval uses to generate $\Qhat^{\pi_k}$.
In general, that dataset is different from $\Dn^{(k)}$ used in~\eqref{eq:CAPI-emp-loss}, though in practice one might use the same dataset for both (except that $\Dn^{(k)}$ as we define here does not have reward and state transition information).
It is also possible to change the sampling distribution at each iteration, e.g., at the $k^\text{th}$ iteration, we generate new samples by following $\pi_k$ and add them to the samples generated in earlier iterations to define $\Dn^{(k)}$~\citep{RossGordonBagnell2011}.
Reusing the same dataset or changing the sampling distribution is not analyzed here, and in our analysis we assume that the dataset used for PolicyEval is independent of $\Dn^{(k)}$ and the same sampling distribution $\nu$ is used in all iterations.
}

\revised{
To extend the current loss function to problems with $\actionnum > 2$, one can define the action-gap function to be the difference between the value of the greedy action and the selected action, i.e., $\gap_{Q}(x,a) \eqdef \max_{a' \in \AA} Q(x,a') - Q(x,a)$. %(for $\actionnum = 2$, it would be the same as the current definition).
 The empirical loss function would be
%
%\begin{align*}
$
\hat{L}^{\pi_{k}}_n(\pi) \eqdef \int_{\XX} \gap_{\hat{Q}^{\pi_{k}}}(x,\pi(x)) \One{\pi(x) \neq \argmax_{a \in \AA} \hat{Q}^{\pi_{k}}(x,a) } \, \mathrm{d} \nu_n$
%\end{align*}
instead. Our theoretical analysis, however, does not cover this case.
}

%%%%%%%%%%%%%%%%%%%%%%%%%%%%%%%%%%%%%%%%%%%%%%%
%%%%%%%%%% IEEE TAC Version -- No detail on KNN-CAPI %%%%%%%%%%%%%
%%%%%%%%%%%%%%%%%%%%%%%%%%%%%%%%%%%%%%%%%%%%%%%
\revised{
The computational complexity of solving the minimization problem~\eqref{eq:CAPI-OptimizationProblem} depends on the choice of policy space $\Pi$.
The problem is similar (but not identical) to minimizing the $0/1$-loss function in binary classification \revised{(or multi-class classification for $\actionnum > 2$)}, and since the loss function is non-convex, minimizing it can be difficult in general.
One possible solution is to relax the non-convex loss function with a convex surrogate such as action-gap-weighted hinge or exponential loss.
One may also use local methods such as action-gap-weighted K-Nearest Neighbour or decision tree classification. For these policy spaces, the computational cost is cheap.

As an example of a local method that leads to computationally cheap solutions, suppose that we have a partition $\XX_1, \XX_2, \dotsc$ of the state space, i.e., 
$\bigcup_i \XX_i = \XX$ and $\XX_i \cap \XX_j = \emptyset$ for $i \neq
j$. Let $I: \XX \ra \{1,2, \dotsc \}$ be an index function that returns $i$ if
$x \in \XX_i$. Thus, $\Dkx \eqdef \Dn^{(k)} \cap \XX_{I(x)}$ is the set of
data points from $\Dn^{(k)}$ that are in the same partition as $x$ is.
The policy $\pi_{k+1}(x)$ is
\ifDoubleColumn
%%%%%%%%% DOUBLE COLUMN %%%%%%%%%
\begin{align*}
	 & \pi_{k+1}(x) \leftarrow 
	 \argmin_{a \in \AA}
	 \sum_{X_i \in \Dkx}	 
	 \gap_{\hat{Q}^{\pi_k }}(X_i) \One{a \neq \pihat(X_i;\hat{Q}^{\pi_{k}
})  } 
	 \\ & 
	 \equiv
	 \argmin_{a \in \AA} \sum_{X_i \in \Dkx}
	\hat{Q}^{\pi_{k}}(X_i,\pihat(X_i;\hat{Q}^{\pi_{k}} )  ) -
\hat{Q}^{\pi_{k}}(X_i,a)
	\\ & 
	\equiv
	\argmax_{a \in \AA} \sum_{X_i \in \Dkx} 
	\hat{Q}^{\pi_{k}}(X_i,a), 
\end{align*}
\else
%%%%%%%%% SINGLE COLUMN %%%%%%%%%
\begin{align*}
	 & \pi_{k+1}(x) \leftarrow 
	 \argmin_{a \in \AA}
	 \sum_{X_i \in \Dkx}	 
	 \gap_{\hat{Q}^{\pi_k }}(X_i) \One{a \neq \pihat(X_i;\hat{Q}^{\pi_{k}
})  } 
	 \\ & 
	 \equiv
	 \argmin_{a \in \AA} \sum_{X_i \in \Dkx}
	\hat{Q}^{\pi_{k}}(X_i,\pihat(X_i;\hat{Q}^{\pi_{k}} )  ) -
\hat{Q}^{\pi_{k}}(X_i,a)
	% \\ &
	\equiv
	\argmax_{a \in \AA} \sum_{X_i \in \Dkx} 
	\hat{Q}^{\pi_{k}}(X_i,a), 
\end{align*}
\fi
where we used the fact that 
$\hat{Q}^{\pi_{k}}(X_i,\pihat(X_i;\hat{Q}^{\pi_k}))$ is not a function of
$a$, so it does not influence the minimizer.
\revised{The derivation for $\actionnum > 2$ with the modified action-gap function leads to the same rule $\pi_{k+1}(x) \leftarrow \argmax_{a \in \AA} \sum_{X_i \in \Dkx} \hat{Q}^{\pi_{k}}(X_i,a)$.}

The result is a very simple rule: pick the action that maximizes the action-value
among all the data points in the same partition as $x$.
Note that this is different from choosing the majority over the greedy actions
in the partition, which would be the rule if we neglected the action-gap.
This partition-based policy is what we get for using a tree structure to represent the policy. Similar rules can be obtained for action-gap-weighted K-Nearest Neighbour-based and other local methods.
}
%%%%%%%%%%%%%%%%%%%%%%%%%%%%%%%%%%%%%%%%%%%%%%%%
%%%%%%%%%%%%%%%%%%%%%%%%%%%%%%%%%%%%%%%%%%%%%%%%
%%%%%%%%%%%%%%%%%%%%%%%%%%%%%%%%%%%%%%%%%%%%%%%%
\fi %%% The if in beginning switching between IEEE Technical Note and arXiv

%%%%%%%%%%%%%%%%%%%%%%%%%%%%%%%%%%%%%%%%%%%%%%%
%%%%%%%%%%%%%%%%%%%%%%%%%%%%%%%%%%%%%%%%%%%%%%%
%%%%%%%%%%%%%%%%%%%%%%%%%%%%%%%%%%%%%%%%%%%%%%%
\section{Theoretical Analysis}
\label{sec:CAPI-Theory}
%!TEX root = CAPI-TAC.tex

In this section we analyze the theoretical properties of CAPI-style algorithms and provide an upper bound on the \emph{performance loss} (or \emph{regret}) of the resulting policy $\pi_{K}$.
The performance loss of a policy $\pi$ is the expected difference between the value of the optimal policy $\piopt$ and the value of $\pi$ when the initial state distribution is $\rho \in \MM(\XX)$, i.e.,
\begin{align*}
%\label{eq:CAPI-PerformanceLoss}
	\mathrm{Loss}(\pi;\rho) \eqdef \int_{\XX} \left( \Vopt(x) - V^\pi (x) \right) \mathrm{d} \rho(x).
\end{align*}
The choice of $\rho$ enables the user to specify the relative importance of different states.

The analysis has two main steps. First, in Section~\ref{sec:CAPI-APIError} we study the behaviour of one iteration of the algorithm and provide an error bound on the expected loss 
$L^{\pi_{k}}(\pi_{k+1}) \eqdef \int_{\XX} g_{Q^{\pi_{k}}}(x) \allowbreak \One{\pi_{k+1}(x) \neq \argmax_{a \in \AA} Q^{\pi_{k}}(x,a) } \, \mathrm{d} \nu$,
as a function of the number of samples in $\Dn^{(k)}$, the quality of the estimate $\Qhat^{\pi_{k}}$, the complexity of $\Pi$, and the policy approximation error.
In Section~\ref{sec:CAPI-ErrorPropagation}, we analyze how the loss sequence $\left( L^{\pi_{k}}(\pi_{k+1}) \right)_{k=0}^{K-1}$ affects $\mathrm{Loss}(\pi_{K};\rho)$.

%%%%%%%%%%%%%%%%%%%%%%%%%%%%%%%%%%%%%%%%%%%%%%%
%%%%%%%%%%%%%%%%%%%%%%%%%%%%%%%%%%%%%%%%%%%%%%%
%%%%%%%%%%%%%%%%%%%%%%%%%%%%%%%%%%%%%%%%%%%%%%%
\subsection{Approximate Policy Improvement Error}
\label{sec:CAPI-APIError}
Policy $\pi_{k}$ depends on data used in earlier iterations, but is independent of $\Dn^{(k)}$, so we will work on the probability space conditioned on $\Dn^{(0)},\dotsc,\Dn^{(k-1)}$.
To avoid clutter, we omit the conditional probability symbol and the dependence of the loss function, policy, and dataset on the iteration number.
In the rest of this section, $\pi'$ refers to a $\sigma(\Dn^{(0)},\dotsc,\Dn^{(k-1)})$-measurable policy and is independent of $\Dn$, which denotes a set of $n$ independent and identically distributed (i.i.d.) samples from the distribution $\nu \in \MM(\XX)$. 
We also assume that we have a $\Dn$-independent approximation $\QpiPrimehat$ of the action-value function $\QpiPrime$.

For any $\pi \in \Pi$, we define two pointwise loss functions:
\revised{
\ifIEEETN
$
l^{\pi'}(\pi)(x) = \gapQPrime(x) \One{\pi(x) \neq \argmax_{a \in \AA} \QpiPrime(x,a) }$ and 
$
\hat{l}^{\pi'}(\pi)(x) = \gapQPrimehat(x) \One{\pi(x) \neq \argmax_{a \in \AA} \QpiPrimehat(x,a) }$.
\else
\begin{align*}
l^{\pi'}(\pi)(x) = \gapQPrime(x) \One{\pi(x) \neq \argmax_{a \in \AA} \QpiPrime(x,a) }, \\
\hat{l}^{\pi'}(\pi)(x) = \gapQPrimehat(x) \One{\pi(x) \neq \argmax_{a \in \AA} \QpiPrimehat(x,a) }.
\end{align*}
\fi}
Note that $l^{\pi'}(\pi)$ is defined as a function of $\QpiPrime$, which is not accessible to the algorithm. On the other hand, $\hat{l}^{\pi'}(\pi)$ is defined as a function of $\QpiPrimehat$, which is available to the algorithm. The latter pointwise loss is a distorted version of the former.
\revised{To simplify the notation, we may use $l(\pi)$ and $\hat{l}(\pi)$ to refer to $l^{\pi'}(\pi)$ and $\hat{l}^{\pi'}(\pi)$, respectively.}

For a function $l: \XX \to \Real$, let
\revised{$\Pr_n l = \frac{1}{n}\sum_{i=1}^n l(X_i)$ and $\Pr l = \EE{l(X)}$}, where $X, X_i \stackrel{\text{i.i.d.}}{\sim} \nu$
and $X_i$s are from $\Dn$.
Now we can define the expected loss $L(\pi) = \Pr l(\pi)$ and the empirical loss $L_n(\pi) = \Pr_n l(\pi)$ (both w.r.t. the true action-value function $\QpiPrime$) and the distorted empirical loss $\hat{L}_n(\pi) = \Pr_n \hat{l}(\pi)$ (w.r.t. the estimate $\QpiPrimehat$).
Given $\Dn$ and $\QpiPrimehat$, let
\begin{align}\label{eq:CAPI-Theory-pihatDefinition}
	\pihat_n \leftarrow \argmin_{\pi \in \Pi} \hat{L}_n(\pi),
\end{align}
(cf.~\eqref{eq:CAPI-OptimizationProblem}). Here and in the rest of the paper we make the standard assumption that the minimum in~\eqref{eq:CAPI-Theory-pihatDefinition} exists. \ifIEEETN\else If it does not, one can instead use a function whose empirical loss is arbitrary close to the infimum and carry out the analysis.\footnote{This is aside the numerical error in finding the minimizer or the computational hardness of optimization. In the case of having an optimization error, an extra term should be added to the upper bounds~\citep{BottouBousquet2008}.}\fi

%%%%%%%%%%%%%%%%%%%%%%%%%%%%%%%%%%%%%%%%%%%%%%%%
%%%%%%%%%%% Shortened version of IEEE Technical Note %%%%%%%%%%%%%%%
%%%%%%%%%%%%%%%%%%%%%%%%%%%%%%%%%%%%%%%%%%%%%%%%
\ifIEEETN
To study the behaviour of $L(\pihat_n$), we need to take care of two main issues.
First we should relate the empirical loss of the minimizer of the distorted empirical loss $\hat{L}_n$, that is $L_n(\pihat_n)$, to the (unavailable) minimum of the empirical loss, $\min_{\pi \in \Pi} L_n(\pi)$ (Lemma~\ref{lem:CAPI-DistortionLemma} in Appendix~\ref{sec:CAPI-Appendix-Proofs}).
We also should relate the expected loss $L(\pihat_n)$ to the empirical loss $L_n(\pihat_n)$. Making this relation requires define a notion of complexity \revised{(or capacity)} of policy space $\Pi$. Among \revised{common choices in the machine learning/statistics literature} (such as VC-dimension, metric entropy, etc., \revised{see e.g.,\ifIEEETN~\citep{Gyorfi02}\else~\citep{DevroyeGyorfiLugosi96,Gyorfi02,Boucheron05} \fi}), we use localized Rademacher complexity since it has favourable properties that often lead to tight upper bounds~\citep{BartlettBousquetMendelson05}.
The use of localized Rademacher complexity to analyze an RL/ADP algorithm is a novel aspect of this work.

%%%%%%%%%%%%%%%%%%%%%%%%%%%%%%%%%%%%%%%%%%%%%%%%
%%%%%%%%%%%%%%%%%% Longer arXiv version %%%%%%%%%%%%%%%%%%%
%%%%%%%%%%%%%%%%%%%%%%%%%%%%%%%%%%%%%%%%%%%%%%%%
\else
We are interested in studying the behaviour of $L(\pihat_n$).
We need to take care of two main issues.
First note that $\pihat_n$ is the minimizer of the distorted empirical loss $\hat{L}_n$ and not $L_n$. This difference causes some error. 
Lemma~\ref{lem:CAPI-DistortionLemma} in Appendix~\ref{sec:CAPI-Appendix-Proofs} relates the empirical loss of $\pihat_n$, that is $L_n(\pihat_n)$, to the (unavailable) minimum of the empirical loss, $\min_{\pi \in \Pi} L_n(\pi)$.

The other issue is to relate the expected loss $L(\pihat_n)$ to the empirical loss $L_n(\pihat_n)$. Making this relation requires defining a notion of complexity \revised{(or capacity)} of policy space $\Pi$. Among \revised{common choices in the machine learning/statistics literature} (such as VC-dimension, metric entropy, etc., \revised{see e.g.,\ifIEEETN~\citep{Gyorfi02} \else~\citep{DevroyeGyorfiLugosi96,Gyorfi02,Boucheron05} \fi for definitions}), we use localized Rademacher complexity~\citep{BartlettBousquetMendelson05} since it has favourable properties that often lead to tight upper bounds. Moreover, as opposed to VC-dimension, it can be used to describe the complexity of nonparametric (infinite dimensional) policy spaces. Another nice property of Rademacher complexity is that it can be estimated empirically, which can be quite useful for model selection. However, we do not discuss its empirical estimation here, as it goes beyond the scope of this paper.
The use of localized Rademacher complexity to analyze an RL/ADP algorithm is a novel aspect of this work.
\fi
%%%%%%%%%%%%%%%%%%%%%%%%%%%%%%%%%%%%%%%%%%%%%%%%

\ifIEEETN\else We briefly define Rademacher complexity and refer the reader to~\citet{BartlettBousquetMendelson05,BartlettMendelson02} for more information. \fi
Let $\sigma_1, \dotsc, \sigma_n$ be independent random variables with $\Prob{\sigma_i = 1} = \Prob{\sigma_i = - 1} = 1/2$. For a function space $\GG: \XX \to \Real$, define
$R_n \GG = \sup_{g \in \GG} \frac{1}{n} \sum_{i=1}^n \sigma_i g(X_i)$ \revised{with $X_i \sim \nu$}. The Rademacher complexity (or average) of $\GG$ is $\EE{R_n \GG}$, in which the expectation is w.r.t. both $\sigma$ and $X_i$\ifIEEETN~\citep{BartlettBousquetMendelson05}\fi.
One can interpret the Rademacher complexity as a measure that quantifies the extent that a function from $\GG$ can fit a noise sequence of length $n$\ifIEEETN\else~\citep{BartlettMendelson02}\fi.

In order to benefit from the localized version of Rademacher complexity, we need to define a sub-root function.
A non-negative and non-decreasing function $\Psi: [0,\infty) \to [0,\infty)$ is called sub-root if $r \mapsto \frac{\Psi(r)}{\sqrt{r}}$ is non-increasing for $r > 0$~\cite{BartlettBousquetMendelson05}.
The following theorem is the main result of this subsection.

%%%%%%%%%%%%%%%%%%%%%%%%%%%%%%%%%%%%%%%%%%%%%%%
%%%%%%%%%%%%% THEOREM - Loss at each iteration %%%%%%%%%%%%%%
%%%%%%%%%%%%%%%%%%%%%%%%%%%%%%%%%%%%%%%%%%%%%%%
\begin{theorem}\label{thm:CAPI-ErrorInEachIteration}
Fix a policy $\pi'$ and assume that $\Dn$ consists of $n$ i.i.d. samples drawn from distribution $\nu$ \revised{and $\QpiPrimehat$ is independent of $\Dn$}. Let $\pihat_n$ be defined by~\eqref{eq:CAPI-Theory-pihatDefinition}. Suppose that Assumption~\ref{ass:CAPI-ActionGap} holds with a particular value of $(\zeta,c_g)$.
Let $\Psi$ be a sub-root function with a fixed point of $r^*$ such that for $r \geq r^*$, 
\ifIEEETN {\small \fi
\begin{align}\label{eq:CAPI-ComplexityCondition}
%	\Psi(r) \geq 2 \Qmax \EE{R_n \cset{\pi \in \Pi}{L^{\pi'}(\pi) \leq \frac{r}{2 \Qmax} }}.
	\Psi(r) \geq 2 \Qmax \EE{R_n \cset{l^{\pi'}(\pi)}{\pi \in \Pi, \Pr [l^{\pi'}(\pi)]^2 \leq r  }}.
\end{align}
\ifIEEETN } \fi
Then there exist $c_1, c_2, c_3 > 0$, which are independent of $n$, $\smallnorm{\QpiPrimehat - \QpiPrime}_{\infty,\Dn}$, and $r^*$, so that for any $0 < \delta < 1$, 
\ifIEEETN
$	L(\pihat_n)
	\leq
	12 \inf_{\pi \in \Pi} L(\pi) + c_1 r^* + c_2 \smallnorm{\QpiPrimehat - \QpiPrime}_{\infty,\Dn}^{1 + \zeta}  + 
	c_3 \frac{\ln(1/\delta)}{n}$,
\else
%%%%%%%%% DOUBLE COLUMN %%%%%%%%%
\ifDoubleColumn
\begin{align*}
	L(\pihat_n)
	\leq
	12 \inf_{\pi \in \Pi} L(\pi) + c_1 r^* & + c_2 \norm{\QpiPrimehat - \QpiPrime}_{\infty,\Dn}^{1 + \zeta}  + \\ & 
	c_3 \frac{\ln(1/\delta)}{n},
\end{align*}
\else
%%%%%%%%% SINGLE COLUMN %%%%%%%%%
\begin{align*}
	L(\pihat_n)
	\leq
	12 \inf_{\pi \in \Pi} L(\pi) + c_1 r^* & + c_2 \norm{\QpiPrimehat - \QpiPrime}_{\infty,\Dn}^{1 + \zeta}  + %\\ & 
	c_3 \frac{\ln(1/\delta)}{n},
\end{align*}
\fi % Single/Double Column
\fi % XXX IEEE TN or Long Version
with probability at least $1 - \delta$.
\end{theorem}
\ifIEEETN \else The proof is in Appendix~\ref{sec:CAPI-Appendix-Proofs}.\fi
%
%%%%%%%%%%%%%%%%%%%%%%%
%%%%%%%%%%%%%%%%%%%%%%%%%%%%%%%%%%%%%%%%%%%%%%%%
%%%%%%%%%%% Shortened version of IEEE Technical Note %%%%%%%%%%%%%%%
%%%%%%%%%%%%%%%%%%%%%%%%%%%%%%%%%%%%%%%%%%%%%%%%
\ifIEEETN
The upper bound has three important terms. 
The $\inf_{\pi \in \Pi} L(\pi)$ term is the \emph{policy approximation error}. For a rich enough policy space, this term can be zero. 

The second important term is the \emph{estimation error} of the classifier, which is mainly determined by the behaviour of the fixed point $r^*$ of~\eqref{eq:CAPI-ComplexityCondition}.
Condition~\eqref{eq:CAPI-ComplexityCondition} implies that the estimation error is not determined by the global complexity of the function space, but by its complexity in the neighbourhood of the minimizer $\argmin_{\pi \in \Pi} L^{\pi'}(\pi)$.
If $\Pi$ is a space with VC-dimension $d$, one can show that $r^*$ behaves as $O(d \log(n)/n)$
(cf. proof of Corollary 3.7 of~\citep{BartlettBousquetMendelson05}).
This rate is considerably faster than the $O(\sqrt{d/n})$ behaviour of the estimation error term in the result of~\citep{LazaricGhavamzadehMunosDPI2010,GabLazGhaSch2011}.
Similar local Rademacher complexity results exist for nonparametric spaces.

The last important term is $\smallnorm{\QpiPrimehat - \QpiPrime}_{\infty,\Dn}^{1 + \zeta}$, whose size depends on 1) the quality of $\QpiPrimehat$ at the points in $\Dn$, and 2) the action-gap regularity of the problem, characterized by $\zeta$.
When $\zeta > 0$, the policy evaluation error $\smallnorm{\QpiPrimehat - \QpiPrime}_{\infty,\Dn}$ is dampened and the rate improves geometrically. 
The analysis of~\citep{LazaricGhavamzadehMunosDPI2010,GabLazGhaSch2011} does not benefit from this regularity.
\revised{Finally note that $\QpiPrimehat$ is often estimated using data, so $\smallnorm{\QpiPrimehat - \QpiPrime}_{\infty,\Dn}$ would be random. As we assumed that $\QpiPrimehat$ is independent of $\Dn$, 
the source of randomness of $\smallnorm{\QpiPrimehat - \QpiPrime}_{\infty,\Dn}$ in the upper bound is different from $\Dn$.
The high probability guarantee of the theorem is on the randomness due to $\Dn$.
}
%%%%%%%%%%%%%%%%%%%%%%%%%%%%%%%%%%%%%%%%%%%%%%%%
%%%%%%%%%%%%%%%%%% Longer arXiv version %%%%%%%%%%%%%%%%%%%
%%%%%%%%%%%%%%%%%%%%%%%%%%%%%%%%%%%%%%%%%%%%%%%%

\else
The upper bound has three important terms. 
The first term is $\inf_{\pi \in \Pi} L(\pi)$, which is the \emph{policy approximation error}. For a rich enough policy space (e.g., a nonparametric one), this term can be zero. The constant multiplier $12$ is by no means optimal and can be chosen arbitrarily close to $1$, at the price of increasing other constants.

The second important term is the \emph{estimation error} of the classifier, which is mainly determined by the behaviour of the fixed point $r^*$ of~\eqref{eq:CAPI-ComplexityCondition}.\footnote{\revised{The choice of a sub-root function $\Psi$ that satisfies~\eqref{eq:CAPI-ComplexityCondition} is discussed in Section~3.1.1 of~\citet{BartlettBousquetMendelson05}. For convex $\Pi$, one might choose $\Psi$ to be the right-hand side of~\eqref{eq:CAPI-ComplexityCondition} and it is guaranteed that we get a sub-root function.
Moreover, the existence and uniqueness of the fixed point of a sub-root function is proven in Lemma 3.2 of~\citep{BartlettBousquetMendelson05}.}}
A couple of interesting observations can be made about condition~\eqref{eq:CAPI-ComplexityCondition}.

First, this condition defines a local notion of complexity. Intuitively,~\eqref{eq:CAPI-ComplexityCondition} states that the estimation error is not determined by the global complexity of function space $\GG_\Pi = \{l^{\pi'}(\pi): \pi \in \Pi \}$, but by its complexity in the neighbourhood of the minimizer $\argmin_{\pi \in \Pi} L^{\pi'}(\pi)$, that is, the Rademacher complexity of $\{l^{\pi'}(\pi): \pi \in \Pi, \Pr [l^{\pi'}(\pi)]^2 \leq r  \}$.

Second, this complexity is related to the complexity of $\Pi$ through the loss function $l^{\pi'}(\pi)$, which is a function of the action-gap $\gapQPrime$.
Hence, it is possible to have a complex policy space but a simple $\GG_\Pi$, e.g., in the extreme case in which the value function is constant everywhere, $\GG_\Pi$ has only a single function.
The question of the interplay between the complexity of policy space $\Pi$, the action-gap function $\gapQPrime$, and the complexity of $\GG_\Pi$ is an interesting future research direction.
Disregarding this subtle aspect of the bound on the estimation error, we note that even a conservative analysis leads to fast rates:
If $\Pi$ is a space with VC-dimension $d$, one can show that $r^*$ behaves as $O(d \log(n)/n)$
\ifIEEETN 
(cf. proof of Corollary 3.7 of~\citet{BartlettBousquetMendelson05})\else
(cf. Proposition~\ref{prop:CAPI-VC2Rad} in Appendix~\ref{sec:CAPI-Appendix-rstar}; also proof of Corollary 3.7 of~\citet{BartlettBousquetMendelson05})\fi.
This rate is considerably faster than the $O(\sqrt{d/n})$ behaviour of the estimation error term in the result of~\citet{LazaricGhavamzadehMunosDPI2010,GabLazGhaSch2011}.

The last important term is $\smallnorm{\QpiPrimehat - \QpiPrime}_{\infty,\Dn}^{1 + \zeta}$, whose size depends on 1) the quality of $\QpiPrimehat$ at the points in $\Dn$, %, which in turn depends on whether the policy evaluation benefits from regularities of the action-value function (such as its smoothness) 
and 2) the action-gap regularity of the problem, characterized by $\zeta$.
When $\zeta = 0$ (i.e., the MDP is ``hard'' according to its action-gap regularity), the policy evaluation error $\smallnorm{\QpiPrimehat - \QpiPrime}_{\infty,\Dn}$ is not dampened, but when $\zeta > 0$, the rate improves geometrically.
The analysis of~\citet{LazaricGhavamzadehMunosDPI2010,GabLazGhaSch2011} does not benefit from this regularity.

\revised{The function $\QpiPrimehat$ is often estimated using data, so $\smallnorm{\QpiPrimehat - \QpiPrime}_{\infty,\Dn}$ would be a random quantity. Because we assumed that $\QpiPrimehat$ is independent of $\Dn$, the source of randomness of $\smallnorm{\QpiPrimehat - \QpiPrime}_{\infty,\Dn}$ is different from $\Dn$.
If one provides an $\eps$ that with probability at least $1 - \delta'$ satisfies 
$\eps \geq \smallnorm{\QpiPrimehat - \QpiPrime}_{\infty,\Dn}$, one can then get
$L(\pihat_n) \leq 12 \inf_{\pi \in \Pi} L(\pi) + c_1 r^* + c_2 \eps^{1 + \zeta} + c_3 \frac{\ln(1/\delta)}{n}$ with probability at least $1 - (\delta + \delta')$.}

Currently, our result is only stated when the quality of policy evaluation is quantified by the supremum norm. Extending it to other $L_p$-norms is an interesting research question.
\revised{Also the question of how to balance the policy approximation error, the estimation error, and the policy evaluation error by the choice of $\Pi$ and PolicyEval algorithm is an interesting question that requires developing proper model selection algorithms.}
\fi
%%%%%%%%%%%%%%%%%%%%%%%%%%%%%%%%%%%%%%%%%%%%%%%%

%%%%%%%%%%%%%%%%%%%%%%%

%%%%%%%%%%%%%%%%%%%%%%%%%%%%%%%%%%%%%%%%%%%%%%%
%%%%%%%%%%%%%%%%%%%%%%%%%%%%%%%%%%%%%%%%%%%%%%%
%%%%%%%%%%%%%%%%%%%%%%%%%%%%%%%%%%%%%%%%%%%%%%%
\subsection{Performance Loss of CAPI}
\label{sec:CAPI-ErrorPropagation}

\ifIEEETN
Here we state the main result of this paper, which upper bounds $\mathrm{Loss}(\pi_{K};\rho)$ as a function of $L^{\pi_{k}}(\pi_{k+1}$) at iterations $k=0,\dotsc,K-1$ and some other properties of the MDP and policy space $\Pi$.
\else
In this section, we state the main result of this paper, Theorem~\ref{thm:CAPI-PerformanceLoss}, which upper bounds the performance loss $\mathrm{Loss}(\pi_{K};\rho)$ as a function of the expected loss $L^{\pi_{k}}(\pi_{k+1}$) at iterations $k=0,1,\dotsc,K-1$ and some other properties of the MDP and policy space $\Pi$.
\fi
First we introduce two definitions.

\begin{definition}[\revised{Worst-Case} Greedy Policy Error]\label{def:CAPI-InherentApproximationError}
For a policy space $\Pi$, the worst-case greedy policy error is $d(\Pi) = \sup_{\pi' \in \Pi} \inf_{\pi \in \Pi} L^{\pi'}(\pi)$.
%The \emph{inherent greedy policy error} of a policy space $\Pi$ is $d(\Pi) = \sup_{\pi' \in \Pi} \inf_{\pi \in \Pi} L^{\pi'}(\pi)$.
\end{definition}
\ifIEEETN
\else
This definition can be understood as follows. Consider a policy $\pi'$ belonging to $\Pi$. It induces an action-value function $\QpiPrime$ and consequently, a greedy policy $\pihat(\cdot;\QpiPrime)$ w.r.t. $\QpiPrime$. This greedy policy may not belong to $\Pi$, so there would be a policy approximation error $\inf_{\pi \in \Pi} L^{\pi'}(\pi)$.
The quantity $d(\Pi)$ is the worst-case of this error over the choice of $\pi'$.
% The worst-case greedy policy error is the worst-case of this error over the choice of $\pi'$.
% The inherent greedy policy error is the supremum of this error over all possible $\pi' \in \Pi$.
\fi

\ifIEEETN
\else
Due to the dynamical nature of MDPs, the performance loss $\mathrm{Loss}(\pi_{K};\rho)$ depends not only on $\left( L^{\pi_{k}}(\pi_{k+1}) \right)_{k=0}^{K-1}$, but also on the difference between the sampling distribution $\nu$ and the future-state distributions in the form $\rho \PKernel^{\pi_1} \PKernel^{\pi_2} \cdots$.
The analysis relating these quantities is called \emph{error propagation} and has been studied in the context of Approximate Value/Policy Iteration algorithms~\ifIEEETN\citep{Munos07,FarahmandMunosSzepesvari10}\else\citep{Munos03,Munos07,FarahmandMunosSzepesvari10,ScherrerLesner2012}\fi.
It relies on quantities called concentrability coefficients, which we define now.
\fi

%
%
%%%%%%%%%%%%%%%%%%%%%%%%%%%%%%%%%%%%%%%%%%%%%
\begin{definition}[Concentrability Coefficient]
\label{def:CAPI-ConcentrabilityCoefficients}
Given $\rho, \nu \in \MM(\XX)$, a policy $\pi$, and two integers $m_1, m_2 \geq 0$, let
$\rho (\PKernel^*)^{m_1} (\PKernel^\pi)^{m_2}$ denote the future-state distribution obtained when the first state is drawn from $\rho$, then the optimal policy $\piopt$ is followed for $m_1$ steps and policy $\pi$ for $m_2$ steps.
Denote the supremum of the Radon-Nikodym derivative of the resulting distribution w.r.t. $\nu$ by
% XXX Short Inline Version XXX
\ifIEEETN
$	c_{\rho,\nu}(m_1;m_2;\pi) \eqdef
	\smallnorm
	{\frac
		{\mathrm{d} ( \rho (\PKernel^*)^{m_1} (\PKernel^\pi)^{m_2} )}
		{\mathrm{d} \nu}
	}_\infty$.
\else
% XXX Long Version XXX
\begin{align*}
	c_{\rho,\nu}(m_1;m_2;\pi) \eqdef
	\norm
	{\frac
		{\mathrm{d} ( \rho (\PKernel^*)^{m_1} (\PKernel^\pi)^{m_2} )}
		{\mathrm{d} \nu}
	}_\infty.
\end{align*}
\fi
If $\rho (\PKernel^*)^{m_1} (\PKernel^\pi)^{m_2}$ is not absolutely continuous w.r.t. $\nu$, then $c(m_1,m_2;\pi) = \infty$.
For an integer $K \geq 1$ and a real $s \in [0,1]$, define
%
% XXX Short Inline Version XXX
\ifIEEETN
$C_{\rho,\nu}(K,s) \eqdef
		\frac{1-\gamma}{2} 
			\sum_{k = 0}^{K-1} \gamma^{(1-s)k} \sum_{m \geq 0} \gamma^m \sup_{\pi' \in \Pi} c_{\rho,\nu}(k,m;\pi')$.
\else
% XXX Long Version XXX
\begin{align*}
	C_{\rho,\nu}(K,s) \eqdef
		\frac{1-\gamma}{2} 
			\sum_{k = 0}^{K-1} \gamma^{(1-s)k} \sum_{m \geq 0} \gamma^m \sup_{\pi' \in \Pi} c_{\rho,\nu}(k,m;\pi').
\end{align*}
\fi
\end{definition}
%%%%%%%%%%%%%%%%%%%%%%%%%%%%%%%%%%%%%%%%%%%%%

\ifIEEETN
The intuition behind Definition~\ref{def:CAPI-InherentApproximationError} is discussed by~\citep{FarahmandCAPIExtended2014}. For a discussion of Definition~\ref{def:CAPI-ConcentrabilityCoefficients} and similar concentrability coefficients, refer to~\citep{Munos07,FarahmandMunosSzepesvari10}.
We are now ready to state the main result.
\else
\noindent We are now ready to state the main result of this paper.
\fi

%%%%%%%%%%%%%%%%%%%%%%%%%%%%%%%%%%%%%%%%%%%%%%%
%%%%%%%%%%%%% THEOREM - Performance loss for CAPI %%%%%%%%%%%%
%%%%%%%%%%%%%%%%%%%%%%%%%%%%%%%%%%%%%%%%%%%%%%%
\begin{theorem}\label{thm:CAPI-PerformanceLoss}
Consider the sequence of independent datasets $(\Dn^{(k)})_{k=0}^{K-1}$, each with $n$ i.i.d. samples drawn from $\nu \in \MM(\XX)$.
Let $\pi_{0} \in \Pi$ be a fixed initial policy and $(\pi_{k})_{k=1}^K$ be a sequence of policies obtained by solving~\eqref{eq:CAPI-emp-loss}, using estimate $\hat{Q}^{\pi_{k}}$ of $Q^{\pi_{k}}$. % in the calculation of $\pi_{k+1}$. 
Suppose that $\hat{Q}^{\pi_{k}}$ is independent of $\Dn^{(k)}$ and Assumption~\ref{ass:CAPI-ActionGap} holds with a particular value of $(\zeta,c_g)$. Let $r^*$ be the fixed point of a sub-root function $\Psi$ such that for any $\pi' \in \Pi$ and $r \geq r^*$,
\ifIEEETN $ \else \[ \fi
	\Psi(r) \geq 2 \Qmax \EE{R_n \cset{l^{\pi'}(\pi)}{\pi \in \Pi, \Pr [{ l^{\pi'}}(\pi) ]^2 \leq r }}.
\ifIEEETN $ \else \] \fi
%
%$\Psi(r) \geq 2 \Qmax \EE{R_n \cset{\pi \in \Pi}{\Pr \left[({l^{\pi'}}^2(\pi) \right] \Pr(L^{\pi'}(\pi) \leq \frac{r}{2 \Qmax} }}$.
Then there exist constants $c_1, c_2, c_3 > 0$ such that for any $0 < \delta < 1$, for $\mathcal{E}(s)$ $(0 \leq s \leq 1)$ defined as
% XXX Short Inline Version XXX
\ifIEEETN
\ifDoubleColumn $
\else \[ \fi
	\mathcal{E}(s) \eqdef
	12 d(\Pi) + c_1 r^* +
				c_2 \max_{0 \leq k \leq K-1}
		\left[  \gamma^{(K-k-1) s} \smallnorm{\hat{Q}^{\pi_{k}} - Q^{\pi_{k}} }_{\infty,\Dn^{(k)}}^{1+\zeta} \right] +	c_3 \frac{\ln(K/\delta)}{n},
\ifDoubleColumn $
\else \] \fi
% XXX Long Version XXX
\else
%%%%%%%%% DOUBLE COLUMN %%%%%%%%%
\ifDoubleColumn
\begin{align*}
& \mathcal{E}(s) \eqdef
	12 d(\Pi) + c_1 r^* + {}
				\\
				&	
				c_2 \max_{0 \leq k \leq K-1}
		\left[ \gamma^{(K-k-1) s} \norm{\hat{Q}^{\pi_{k}} - Q^{\pi_{k}} }_{\infty,\Dn^{(k)}}^{1+\zeta} \right] + {}
%		\\ &	
		c_3 \frac{\ln(K/\delta)}{n},
\end{align*}
\else
%%%%%%%%% SINGLE COLUMN %%%%%%%%%
\begin{align*}
& \mathcal{E}(s) \eqdef
	12 d(\Pi) + c_1 r^* +
				c_2 \max_{0 \leq k \leq K-1}
		\left[ \gamma^{(K-k-1) s} \norm{\hat{Q}^{\pi_{k}} - Q^{\pi_{k}} }_{\infty,\Dn^{(k)}}^{1+\zeta} \right] + {}
	%	\\ &	
		c_3 \frac{\ln(K/\delta)}{n},
\end{align*}
\fi % Single/Double Column
\fi % XXX IEEE TN or Long Version
we have with probability at least $1 - \delta$,
%
%
% XXX Short Inline Version XXX
\ifIEEETN
$	\mathrm{Loss}(\pi_{K};\rho) \leq 
	\frac{2}{1 - \gamma}
	\left[
		\inf_{s \in [0,1]} C_{\rho,\nu}(K,s) \, \mathcal{E}(s) + \gamma^K \Rmax
	\right]
$.
% XXX Long Version XXX
\else
\begin{align*}
	\mathrm{Loss}(\pi_{K};\rho) \leq 
	\frac{2}{1 - \gamma}
	\Big[
		\inf_{s \in [0,1]} C_{\rho,\nu}(K,s) \, \mathcal{E}(s) + \gamma^K \Rmax
	\Big].
\end{align*}
\fi
\end{theorem}
\ifIEEETN \else The proof is in Appendix~\ref{sec:CAPI-Appendix-Proofs}.\fi
All discussions after Theorem~\ref{thm:CAPI-ErrorInEachIteration} regarding the policy approximation error, the estimation error, and the role of the action-gap regularity apply here too. Moreover, the new error propagation result used in the proof is an improvement over the previous results~\citep{LazaricGhavamzadehMunosDPI2010,GabLazGhaSch2011}. The result indicates that the error $\smallnorm{\hat{Q}^{\pi_{k}} - Q^{\pi_{k}} }_{\infty,\Dn}$ is weighted proportional to $\gamma^{(K-k-1) s}$, which means that errors at earlier iterations are geometrically discounted.

\ifIEEETN
\else
A practical implication is that if one has finite resources (samples or computation time), it is better to focus on obtaining better estimates of $Q^{\pi_{k}}$ at later iterations. The same advice holds for the classifier: using more samples at later iterations is beneficial (though this is not apparent from the bound, as we fixed $n$ throughout all iterations).
The error propagation result is based on the technique developed by~\citet{FarahmandMunosSzepesvari10}. That work, however, studies the error propagation of Approximate Value/Policy Iteration algorithms and is not 
tailored to CAPI and its loss function.
For a detailed discussion of this type of error propagation result, the reader is referred to~\citep{FarahmandMunosSzepesvari10}.
% The discount rate is $\gamma^s$, so larger $s$ implies more discounting. On the other hand, larger $s$ increases $C_{\rho,\nu}(K,s)$. The optimal tradeoff depends on the behaviour of both terms. The upper bound shows that is automatically achieved.

\revised{
As discussed earlier, $\hat{Q}^{\pi_k}$ is often a random quantity because of the randomness in PolicyEval. Moreover, $\pi_k$ is a function of $(\Dn^{(l)})_{l=1}^{k-1}$ as well as the datasets used in the estimation of $(\hat{Q}^{\pi_l} )_{l=1}^{k-1}$, so it is random too.
Therefore, the upper bound of this theorem is random. Nonetheless, one might provide a high probability upper bound on $\smallnorm{\hat{Q}^{\pi_{k}} - Q^{\pi_{k}} }_{\infty,\Dn^{(k)}}$. Some of these bounds even hold uniformly over all policies, so the randomness of $\pi_k$ would not be an issue, e.g.,~\citet{AntosSzepesvariML08, FarahmandNIPS08, LazaricGhavamzadehMunosLSPI2012} for some $L_p$-norm results.
}
\fi

%%%%%%%%%%%%%%%%%%%%%%%%%%%%%%%%%%%%%%%%%%%%%%%
%%%%%%%%%%%%%%%%%%%%%%%%%%%%%%%%%%%%%%%%%%%%%%%
%%%%%%%%%%%%%%%%%%%%%%%%%%%%%%%%%%%%%%%%%%%%%%%
\ifIEEETN
\else
\section{Experiments}
\label{sec:CAPI-Experiments}
%!TEX root = CAPI-TAC.tex

%%%%%%%%%%%%%%%%%%%%%%%%%%%%%%%%%%%%%%%%%%%%%%%
%%%%%%%%%%%%%%%%%% IEEE TAC Version %%%%%%%%%%%%%%%%%%
%%%%%%%%%%%%%%%%%%%%%%%%%%%%%%%%%%%%%%%%%%%%%%%
%\ifIEEE

We first present a simple experiment to show that using the action-gap-weighted loss can lead to significantly better performance compared to 1) pure value-based approaches and 2) classification-based API with the $0/1$-loss. We also study the effect of the policy approximation error on the results.
Afterwards, we compare an instantiation of CAPI with a state-of-the-art pure value-based approach on the problem of designing adaptive treatment strategies for HIV-infected patients~\citep{ErnstStanGongalvesWehenkel2006}.
This problem is high-dimensional (the state space is $\Real^6$) and is considered a difficult task.
\revised{
We have also conducted several other experiments showing that CAPI is a flexible framework and that the algorithms derived from it (by choosing the policy evaluation procedure and policy space $\Pi$) can be quite competitive. 
The results are reported in Appendix~\ref{sec:CAPI-Appendix-Experiments(Additional)}.}

%%%%%%%%%%%%%%%%%%%%%%%%%%%%%%%%%%%%%%%%%%%%%%%
%%%%%%%%%%%%%%%%%%%%%%%%%%%%%%%%%%%%%%%%%%%%%%%
%%%%%%%%%%%%%%%%%%%%%%%%%%%%%%%%%%%%%%%%%%%%%%%
\subsection{1D Chain Walk}
\label{sec:CAPI-Experiments-1DChainWalk}

\begin{figure}[t]
\vskip 0.2in
\begin{center}
\centerline{\includegraphics[width = 1 \columnwidth]{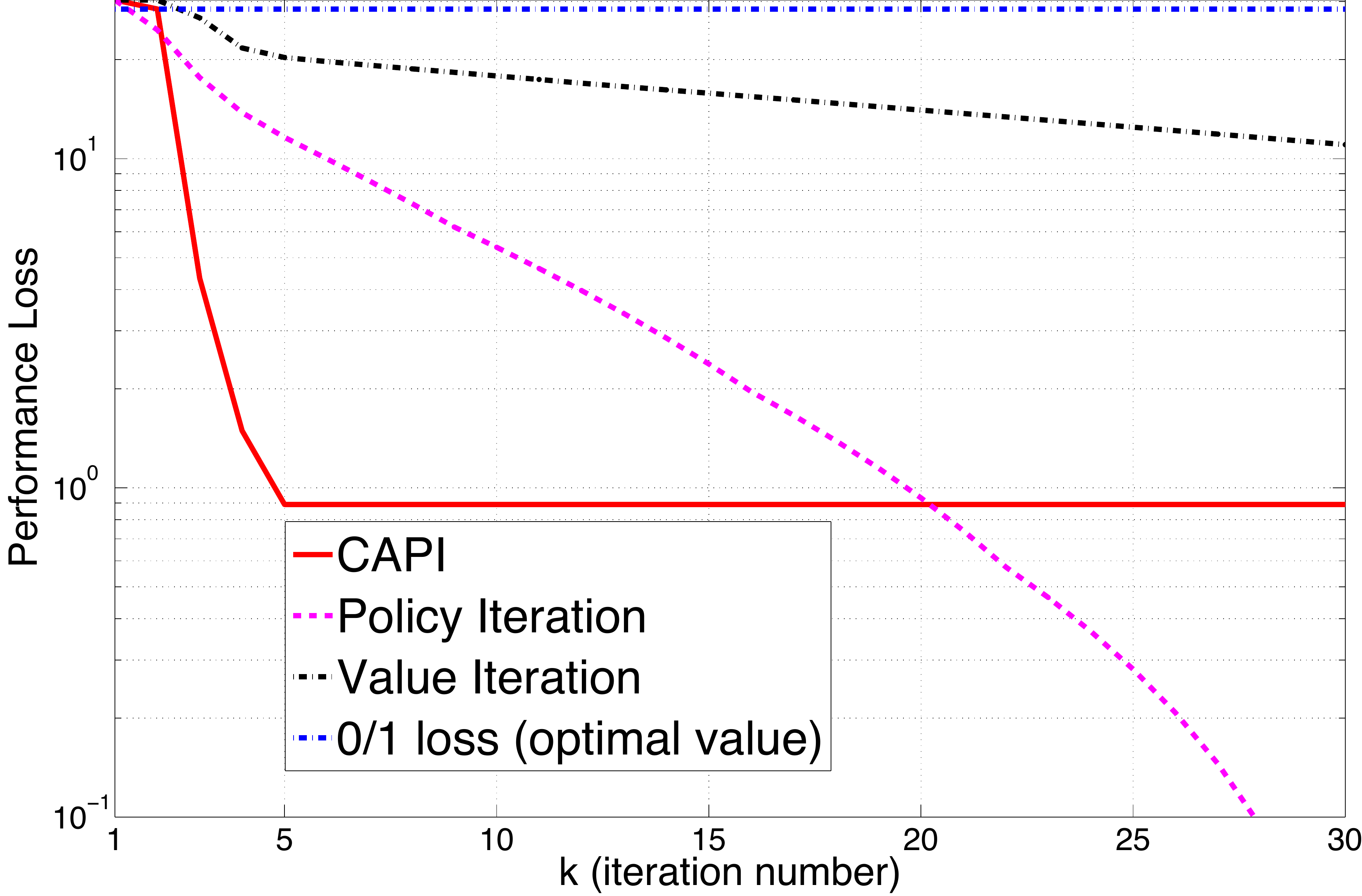}}
\caption{(1D Chain Walk with Policy Approximation Error) Performance loss for CAPI, Value Iteration, Policy Iteration, and the $0/1$-loss given $\Qopt$. 
%The problem is a 1D random walk with 200 states and $\gamma = 0.99$.
\vspace*{-0.7cm}
}
\label{fig:CAPI-Loss-VI-PI-CAPI}
\end{center}
%\vskip -0.2in
\end{figure}

We compare CAPI with Value Iteration (VI), Policy Iteration (PI), and a modified CAPI that uses the $0/1$-loss on a simple 1D chain walk problem (based on the example in Section~9.1 of~\citet{LagoudakisParr03}). The problem has $200$ states, the reward function is zero everywhere except at states $10-15$ (where it is $+1$ for both actions) and $180-190$ (where it is $+0.1$ for both actions), and $\gamma = 0.99$.

Note that the model is known. 
CAPI is run with the distribution $\nu_n$ in the loss function~\eqref{eq:CAPI-emp-loss} as the uniform distribution over states. %, so all states are equally counted. 
The value of $\hat{Q}^{\pi_{k}}$ at iteration $k$ of CAPI is obtained by running just one iteration of VI-based policy evaluation, i.e., $\hat{Q}^{\pi_{k}} = T^{\pi_{k}} \hat{Q}^{\pi_{k-1}}$, in which $T^{\pi_{k}}$ is the Bellman operator for policy $\pi_{k}$.
This makes the number of times CAPI queries the model similar to that of VI. %, which is defined as $\hat{Q}_k^\text{VI} = \Topt \hat{Q}_k^\text{VI}$, in which $\Topt$ is the Bellman optimality operator.
The policy space $\Pi$  is defined as the space of indicator functions of the set of all half-spaces, i.e., 
the set of  policies that choose action $1$ (or $2$) on  $\{1,\dotsc, p \}$ and action $2$ (or $1$) on $\{p+1, \dotsc, 200\}$ for $1 \leq p \leq 200$.
This is a very small subset of all possible policies.
We \emph{intentionally} designed the reward function such that the optimal policy is \emph{not} in $\Pi$, so CAPI will be subjected to policy approximation error. % (i.e, systematic bias).
% We give disadvantage to CAPI by designing the reward function such that the optimal policy does not belong to $\Pi$, so there shall be a policy approximation error.

Figure~\ref{fig:CAPI-Loss-VI-PI-CAPI} shows that the performance loss of CAPI converges to this policy approximation error, which is the best solution achievable given $\Pi$. The convergence rate is considerably faster than that of VI and PI. 
This speedup is due to the fact that CAPI searches in a much smaller policy space compared to VI or PI.
The comparison of CAPI and VI is especially striking, since both of them use the same number of queries to the model \revised{and are computationally comparable (CAPI is a bit more expensive than VI due to the optimization involved in finding the best policy in the policy space, but since the policy space is rather small in this problem, the difference is not considerable).
The computation time of PI is considerably higher as it evaluates a policy at each iteration.}

We also report the performance loss of a modified CAPI that uses the $0/1$-loss and the \emph{exact} $\Qopt$ (so there will be no estimation error).
The result is quite poor.
To understand this behaviour, note that the minimizer of the $0/1$-loss is a policy that approximates the greedy policy (in this case, the optimal policy) without paying attention to the action-gap function. 
Here, the minimizer of the $0/1$-loss policy is such that it fits the optimal policy, which does not belong to the policy space, in a large region of the state space where the action-gap is small and differs from the optimal policy in a smaller region where the action-gap is large. This selection ignores the relative importance of choosing the wrong action in different regions of the state space and results in poor behaviour (cf. Theorems 3--4 of~\citep{LiBulitkoGreiner2007}).
% Here, the $0/1$-loss minimizer policy is such that it fits the greedy policy in a large region of the state space where the action-gap is small and differs from the greedy policy in a smaller region where the action-gap is large. This selection ignores the relative importance of choosing the wrong action in different regions of the state space, resulting in poor behaviour.
%
% Only when the optimal policy belongs to the policy space $\Pi$ and an accurate estimate of the optimal action-value function is available, the minimizer of the $0/1$ loss would pick the optimal action.

We now study the case in which the optimal policy belongs to policy space $\Pi$.
We use the same 1D chain walk, with the difference that the reward function is zero everywhere except at states $10-15$ (where it is $+1$ for both actions). In the previous experiment, the reward function was also nonzero between $180-190$.
Other than this, the experiment is performed as before.
This change ensures that the optimal policy belongs to policy space $\Pi$.

The result is shown in Figure~\ref{fig:CAPI-Loss-VI-PI-CAPI-WithoutPolicyApprox}.
The performance loss of both CAPI and PI goes to zero very \revised{quickly}.
However, the computation time of CAPI is comparable to VI, and is much cheaper than that of PI.
The performance loss of the $0/1$-loss CAPI with the optimal action-value function is zero in this case, so we do not report it.

%%%%%%%%%%%%%%%%%%%%%%%%%%%%%%%%%%%%%%%%%%%%%%%
\begin{figure}[t]
\vskip 0.2in
\begin{center}
\centerline{\includegraphics[width = 1 \columnwidth]{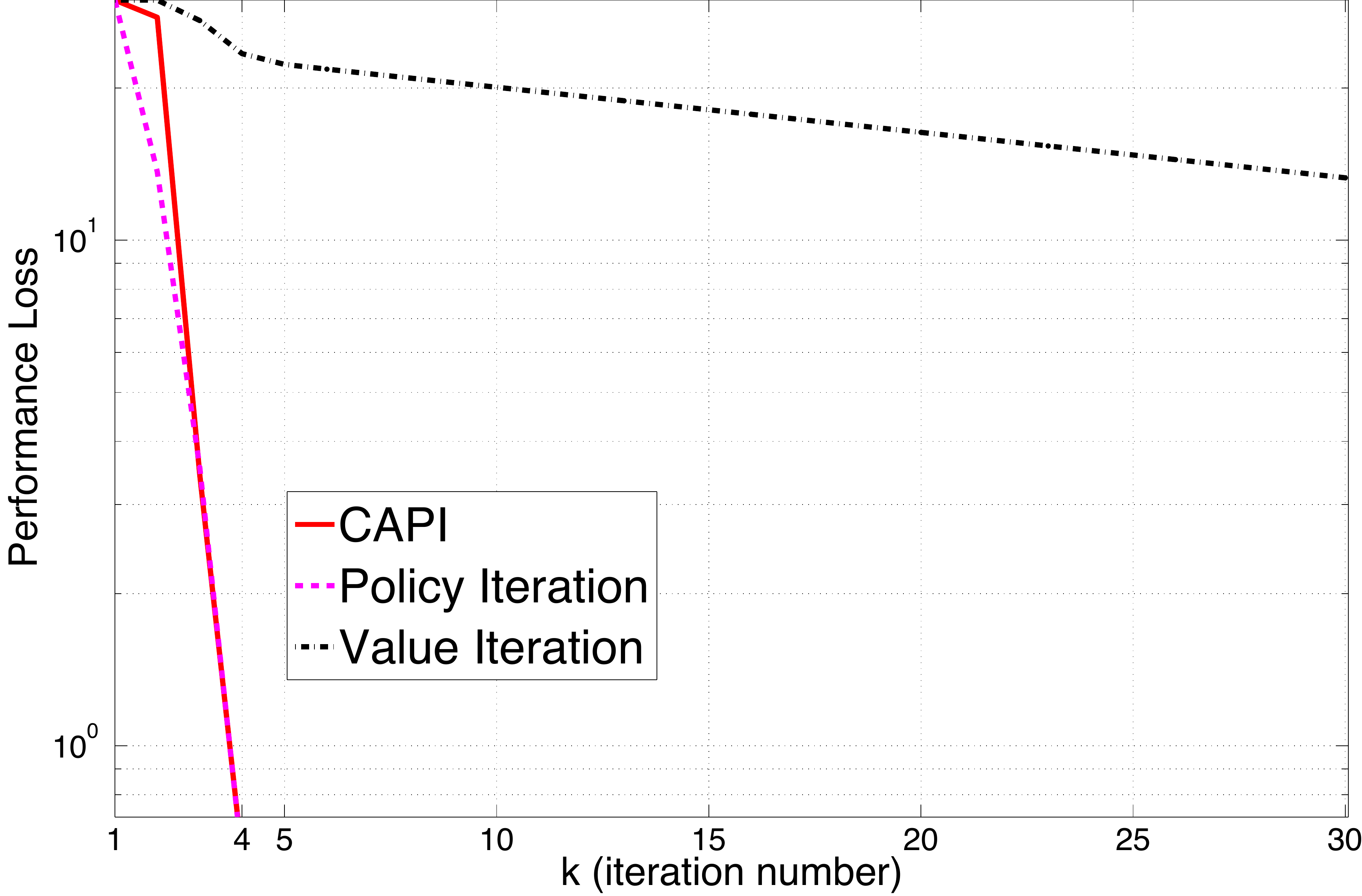}}
\caption{(1D Chain Walk without Policy Approximation Error) Performance loss for CAPI, Value Iteration, and Policy Iteration.
%. The problem is a 1D random walk with 200 states and $\gamma = 0.99$. 
After the $4^\text{th}$ iteration, the performance loss of PI and CAPI is zero.}
\label{fig:CAPI-Loss-VI-PI-CAPI-WithoutPolicyApprox}
\end{center}
%\vskip -0.2in
\end{figure} 
%%%%%%%%%%%%%%%%%%%%%%%%%%%%%%%%%%%%%%%%%%%%%%%

%%%%%%%%%%%%%%%%%%%%%%%%%%%%%%%%%%%%%%%%%%%%%%%
%%%%%%%%%%%%%%%%%%%%%%%%%%%%%%%%%%%%%%%%%%%%%%%
%%%%%%%%%%%%%%%%%%%%%%%%%%%%%%%%%%%%%%%%%%%%%%%
\subsection{HIV drug schedule}
\label{sec:CAPI-Experiments-HIV}

Most of the anti-HIV drugs currently available fall into one of two categories:
reverse transcriptase inhibitors (RTI) and protease inhibitors (PI). RTI and PI
drugs act differently
on the organism, and typical HIV treatments use drug cocktails containing both
types of medication. 
Despite the success of drug cocktails in maintaining
low viral loads, there are several complications associated with their
long-term use. This has attracted the interest of the scientific community
to the problem of optimizing drug-scheduling strategies. 
Among them, a strategy that has
received a lot of attention recently is structured treatment interruption
(STI), in which patients undergo alternate cycles with and without the drugs.

The scheduling of STI treatments can be seen as a sequential decision problem
in which the actions correspond to the types of cocktail that should
be administered to a patient~\citep{ErnstStanGongalvesWehenkel2006}. 
To simplify the problem formulation, 
it is assumed that RTI and PI drugs are administered at fixed amounts, 
reducing the available actions to the four possible combinations
of drugs. The goal is to minimize the HIV viral load
using a drug amount as small as possible. 

We studied the problem of optimizing STI treatments using a model of the
interaction between the immune system and HIV developed
by~\citet{AdamsBanksKwonTran2004} based on real clinical data. 
All the parameters of the model were set as 
suggested by~\citet{ErnstStanGongalvesWehenkel2006}.
The methodology adopted in the computational experiments, such as the
evaluation of the decision policies and the collection of sample
transitions, also followed the same protocol.

To illustrate the potential benefits of controlling the complexity of the
decision policies, we compare the pure value-based algorithm adopted 
by~\citet{ErnstStanGongalvesWehenkel2006}, Fitted Q-Iteration, with a modified
version in which the space of policies is restricted.
Following the original experiments, we approximated the value function using an
ensemble of $30$ decision trees generated by
\citeauthor{GeurtsErnstWehenkel2006}'s~\citeyear{GeurtsErnstWehenkel2006}
extra-trees algorithm (we refer
to this instantiation of Fitted Q-Iteration as Tree-FQI).
Tree-CAPI works exactly as Tree-FQI, except that instead of using the greedy
policies induced by the current value function approximation, 
it uses a second ensemble of $30$ trees to represent the decision
policies.

Note that in order to build the trees representing decision policies,
the extra-trees algorithm has to be slightly modified to incorporate the estimated action-gap as its loss function. Consider a Tree-CAPI policy based on a single tree. The tree defines a partition $\XX_1, \XX_2, \dotsc$ of the state space. This partitioning of the state space defines a policy as described in Section~\ref{sec:CAPI-Algorithm}.
When we have several trees, as in extra-trees, their policies are combined by voting to obtain the outcome policy.

The complexity of the models built by the extra-trees algorithm can be
controlled by the minimum number of points required to split a node 
during the construction of the trees, \revised{$\eta$}~\cite{GeurtsErnstWehenkel2006}. In general,
the larger this number is, the simpler the resulting models are. Here, we fixed this
parameter for the trees representing the value function at $\mnec = 50$, while
the corresponding parameter for the decision-policy trees, called here \mnea,
was varied  in the set $\{2, 10, 20, 50, 100, 200, 300, 500, 1000, 2000\}$.

\begin{figure}[tb] %  figure placement: here, top, bottom, or page
\centering
   \includegraphics[width= 1 \columnwidth]{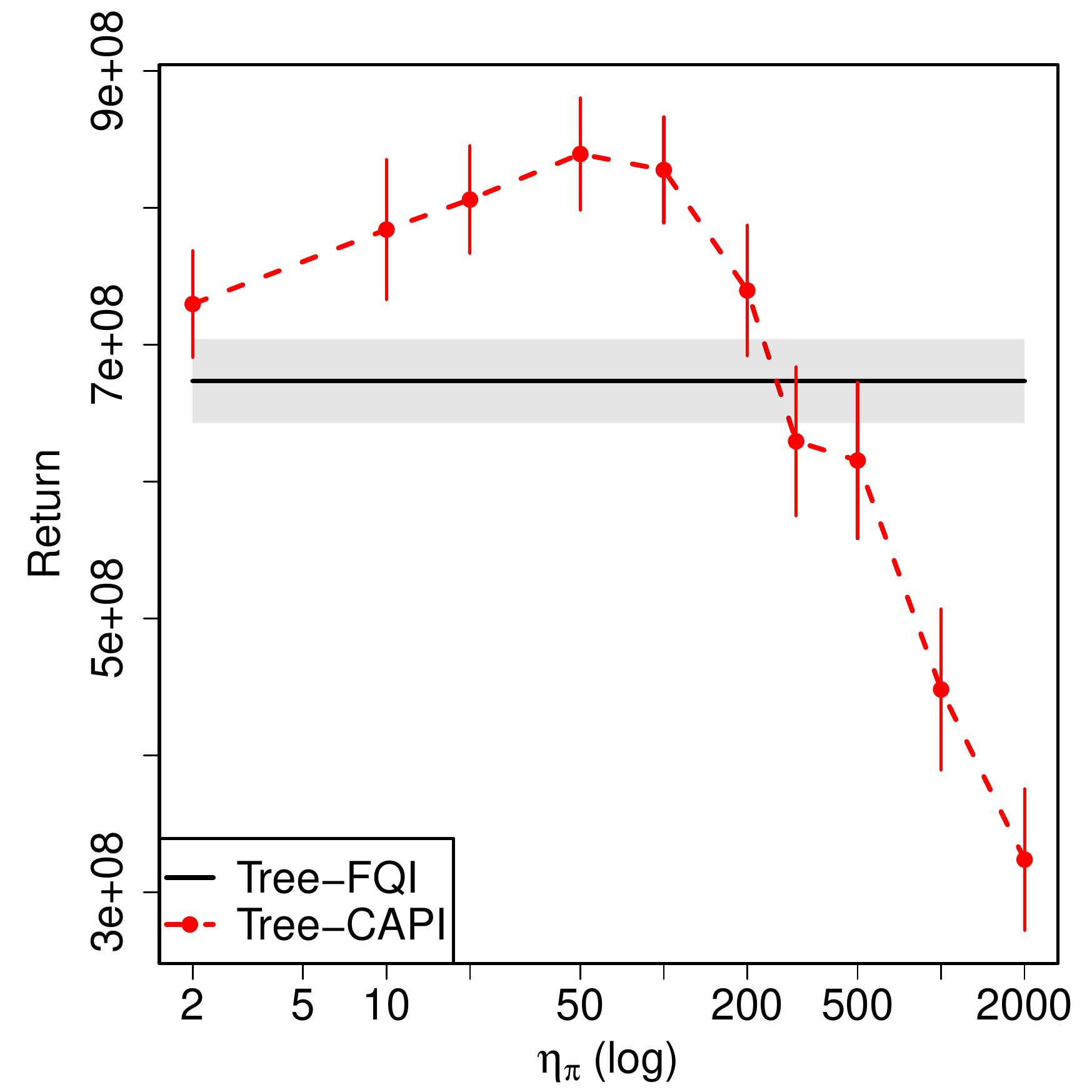}
 \caption{(HIV) Comparing the expected return for Tree-based CAPI vs. Tree-based
Fitted Q-Iteration as a function of parameter $\eta_{\pi}$ describing the complexity of policy space (inversely proportional).
The decision policies (in this case, STI treatments) were evaluated for 
$5,000$ days starting from an ``unhealthy'' state with a high viral
load (see~\cite{ErnstStanGongalvesWehenkel2006}). The error bars and shadowed region
show one standard error over 50 runs.
}
\label{fig:CAPI-HIV-TreeCAPIvsTreeFQI}
\end{figure}

Figure~\ref{fig:CAPI-HIV-TreeCAPIvsTreeFQI} shows the result of this
an experiment. 
As can be seen, restricting the policy
space can have a dramatic impact on the performance of the resulting 
decision policies. When $2 \le \mnea \le 100$, the
policies computed by Tree-CAPI perform better than those computed by Tree-FQI,
leading to increases in the empirical return as high as $24\%$.
On the other hand, an overly restricted policy space precludes the
representation of the intricacies of an efficient STI treatment, resulting in
poor performance.
We expect these results to be representative of a more general trend, in which 
the ``right'' level of complexity of the decision policies is influenced by
factors such as the difficulty of the  problem and
the number of sample transitions available. With CAPI, it is possible to
adjust the complexity of the policy space to a specific context, and the use of
a greedy policy is simply the particular case in which no restrictions
are imposed.

\fi

%%%%%%%%%%%%%%%%%%%%%%%%%%%%%%%%%%%%%%%%%%%%%%%
%%%%%%%%%%%%%%%%%%%%%%%%%%%%%%%%%%%%%%%%%%%%%%%
%%%%%%%%%%%%%%%%%%%%%%%%%%%%%%%%%%%%%%%%%%%%%%%
\ifIEEETN
\else
\section{\revised{Comparison with Other Work}}
\label{sec:CAPI-RelatedWork}
%!TEX root = CAPI(TAC).tex

%It is worthwhile to briefly compare CAPI with some other similar approaches.
As mentioned already, the most similar approach to CAPI is DPI~\cite{LazaricGhavamzadehMunosDPI2010}, which uses rollouts.
%The main difference  is that DPI is specifically designed for rollout-based policy evaluation. Hence, CAPI can exploit both policy and value function regularities, while DPI  targets only regularities of the policy space (note that 
%
DPI-Critic~\cite{GabLazGhaSch2011}, which is also a special case of CAPI, can benefit from value function regularities, but only in the correction term of the rollout estimates, for which DPI-Critic uses a general value function estimator to reduce the truncation bias.
Both of these algorithms are special cases of CAPI with a particular choice for the policy evaluation step.
%Additionally, rollout-based approaches cannot handle batch data easily, whereas there are already many policy evaluation algorithms well-suited for batch data.
The theoretical analysis for CAPI compared to that of DPI/DPI-Critic shows a tighter upper bound for the estimation error (e.g., $O(d \log(n) /n)$ for CAPI vs. $O(\sqrt{d/n})$ for DPI/DPI-Critic, for policy spaces of VC-dimension $d$), handles nonparametric policy spaces, and shows that solving MDPs with a favourable action-gap property is easier.
\ifIEEETN 
In some other experiments reported by~\citep{FarahmandCAPIExtended2014}, CAPI used orders of magnitude fewer samples than Tree-DPI to achieve good performance.
\else
In some other experiments reported in Appendix~\ref{sec:CAPI-Appendix-Experiments(Additional)}, CAPI used orders of magnitude fewer samples than Tree-DPI to achieve good performance. 
\fi
This shows the power of generalizing with a good value function estimator.

Another similar approach is the Conservative Policy Iteration (CPI) algorithm introduced by~\citet{KakadeLangfordCPI2002}.
CPI was designed to address the well-known problem that API may not converge, and thus the performance of the resulting policy may oscillate. CPI guarantees that the performance loss of the generated sequence of policies is monotonically decreasing before the algorithm stops (with high probability). DPI and CAPI do not have such a monotonicity guarantee (but note that we still have upper bounds on the performance loss).
Recently~\citet{GhavamzadehLazaric2012} have closely compared CPI and DPI, and concluded 
 that these algorithms are actually quite similar (and as a result, CPI has some similarities with CAPI as well). One difference is in the probability distribution used in the minimization problem~\eqref{eq:CAPI-Theory-pihatDefinition} for DPI.  Also, CPI updates policies conservatively, but DPI does not.
The effect of this conservatism is that CPI requires exponentially more iterations (and samples) than DPI (and, by extension, CAPI) to achieve the same accuracy. Also, the number of required iterations for DPI is a function of $1/(1-\gamma)$, which can be large for $\gamma$ close to $1$.
Note that CPI might converge to a suboptimal policy whose performance is not better than the one achieved by DPI.
For more details, we refer the reader to~\citet{GhavamzadehLazaric2012}.

There is an interesting connection between CAPI and the actor-critic (AC) family of algorithms~\citep{KondaTsitsiklis01,PetersVijayakumarSchaal03,BhatnagarSuttonGhavamzadehLee2009}.
The critic is essentially a policy evaluation algorithm, which is a component of CAPI too. 
On the other hand, in most implementations of AC algorithms, the actor uses a gradient-based approach to update the policy.
If the actor was generated to minimize the loss function~\eqref{eq:CAPI-emp-loss} in a policy space $\Pi$, one would obtain a CAPI-style algorithm.

An example of this relaxed definition of AC is a Fitted Q-Iteration algorithm developed by~\citet{Antos08NIPS} that is tailored to continuous action spaces. The authors mentioned that the resulting method might be called a \emph{fitted actor-critic} algorithm.
That algorithm shares the general form of the actor with CAPI, but is specialized to a particular class of parametric Fitted Q-Iteration-based critics.
Since these two algorithms consider different problems (continuous vs. finite action spaces), a direct comparison of their theoretical guarantees is not possible. 
Acknowledging this major difference, we remark that our analysis allows nonparametric policy spaces, while their analysis is specialized to parametric spaces.
Moreover, their rate appears to be slower than what could be achieved. The reason is that they use global measures of complexity (pseudo-dimension, which is closely related to VC-dimension) and control the supremum of empirical processes as opposed to the more advanced techniques of this paper, i.e., localized Rademacher-based analysis (which is based on the modulus of continuity of the empirical process). On the other hand, they assume that the input data is a $\beta$-mixing process, which is more relaxed than our i.i.d. assumption.

%%%%%%%%%%%%%%%%%%%%%%%%%%%%%%%%%%%%%%%%%%%%%%%

The action-gap function and a similar notion of regularity have also been used by~\citet{DimitrakakisLagoudakis2008} in the context of rollout allocation. Their goal is to decide how to efficiently assign rollouts to states, so that informative samples can be provided to the classifier with as few rollouts as possible. The main idea is that it is easy to quickly choose the greedy action with high confidence when the action-gap function is large and vice versa. However,  they do not consider the fine balance between large action-gap/small probability of error and small action-gap/small regret (they only consider the first aspect).
Also, their suggested algorithm is based on the $0/1$-loss, which as argued before, can lead to bad policies.

\fi

%%%%%%%%%%%%%%%%%%%%%%%%%%%%%%%%%%%%%%%%%%%%%%%
%%%%%%%%%%%%%%%%%%%%%%%%%%%%%%%%%%%%%%%%%%%%%%%
%%%%%%%%%%%%%%%%%%%%%%%%%%%%%%%%%%%%%%%%%%%%%%%
\section{Conclusion and Future Work}
\label{sec:CAPI-Conclusion}
%!TEX root = CAPI-TAC.tex

We proposed CAPI, a general family of algorithms that exploits regularities of both the value function and the policy. 
CAPI uses any policy evaluation method, defines an action-gap-weighted loss function, and finds the policy minimizing this loss from a desired policy space.
We provided an error upper bound that is tighter than existing results and applies to general policy evaluation algorithms and nonparametric policy spaces.
\ifIEEETN
\revised{
The experiments reported in~\citep{FarahmandCAPIExtended2014} show that CAPI using a powerful PolicyEval outperforms a rollout-based classification-based algorithm as well as a state-of-the-art purely value-based approach.
}
\else

Our experiments showed that when a powerful policy evaluation method such as Fitted Q-Iteration is used in CAPI, the resulting policy outperforms a purely value-based approach in terms of the quality of the solution and the sample efficiency.
Moreover, our experiments in the appendix showed that CAPI outperforms a rollout-based classification-based RL algorithm.

CAPI might be computationally more expensive than a pure value-based approach. Its computational cost mainly depends on how the optimization problem~\eqref{eq:CAPI-emp-loss} is solved.
For example, the computational cost of learning a policy by Tree-CAPI is almost twice the cost of Tree-FQI, while the cost of finding the action is almost the same.
We envision CAPI as being especially useful in the batch setting and for problems in which acquiring samples is expensive (such as in medical treatment design or mining optimization applications), so sample efficiency is more important than the computational complexity.
\fi

%%%%%%%%%%%%%%%%%%%%%%%%%%%%%%%%%%%%%%%%%%%%%%%
%
\ifIEEETN
Analyzing CAPI with a convex surrogate loss is an interesting question, as is extending CAPI to continuous action spaces.
The sampling distribution $\nu$ can have a big effect on the performance; how to choose it is an open question. 
\else

We showed how to efficiently solve the optimization problem~\eqref{eq:CAPI-emp-loss} for policy spaces induced by local methods such as Decision Trees or KNN. Extending this to other reasonably general classes of policy spaces (e.g., policies that are defined by the sign of linear combination of basis functions) is an open question.
The use of surrogate losses is likely to be a reasonable answer, but the theoretical properties of that approach should be investigated, \revised{possibly similar to what \citet{BartlettJordanMcAuliffe2006} do in the context of classification.}
Additionally, an interesting research direction is to extend and analyze CAPI for continuous action spaces.

The sampling distribution $\nu$ can have a big effect on the performance; how to choose it well is an open question. One could even change the sampling distribution at each iteration, to actively obtain more informative samples.

Finally, CAPI-style algorithms could be very useful in problems where the policy space needs to be restricted, by imposing constraints (e.g., torques not exceeding maximum values). Studying such applications would be interesting.
\fi

\ifIEEETN
	\appendix[Proofs]
	%!TEX root = CAPI-TAC.tex

%%%%%%%%%%%%%%%%%%%%%%%%%%%%%%%%%%%%%%%%%%%%%%%
%%%%%%%%%%%%%%% IEEE TAC Technical Note Version %%%%%%%%%%%%%%
%%%%%%%%%% Things are a bit compressed here and some parts are omitted. %%%%%%
%%%%%%%%%%%%%%%%%%%%%%%%%%%%%%%%%%%%%%%%%%%%%%%

\ifIEEETN

\label{sec:CAPI-Appendix-Proofs}

%%%%%%%%%%%%%%%%%%%%%%%%%%%%%%%%%%%%%%%%%%%%%%%
%%%%%%%%%%%%%%%% LEMMA - Distortion Lemma %%%%%%%%%%%%%%%%
%%%%%%%%%%%%%%%%%%%%%%%%%%%%%%%%%%%%%%%%%%%%%%%
\begin{lemma}[Loss Distortion Lemma]\label{lem:CAPI-DistortionLemma}
Fix a policy $\pi'$. Suppose that $\QpiPrimehat$ is an approximation of the action-value function $\QpiPrime$. Given the dataset $\Dn$, let $\pihat_n$ be defined as~\eqref{eq:CAPI-Theory-pihatDefinition} and define $\pi^*_n \leftarrow \argmin_{\pi \in \Pi} L_n(\pi)$.
Let Assumption~\ref{ass:CAPI-ActionGap} hold. There exist finite $c_1, c_2 > 0$, which depend only on $\zeta$, $c_g$, and $\Qmax$, such that for any $z > 0$, we have
$L_n(\hat{\pi}_n)
	\leq
	3 L_n(\pi^*_n) + c_1 \smallnorm{\QpiPrimehat - \QpiPrime}_{\infty,\Dn}^{1+\zeta} + c_2 \frac{z}{n}$,
with probability at least $1 - e^{-z}$.
\end{lemma}

In the proofs, $c_1, c_2, \dotsc$ are constants whose values may change from line to line -- unless specified otherwise.
%%%%%%%%%%%%%%%%%%%%%%%%%%%%%%%%%%%%%%%%%%%%%%%
%%%%%%%%%%%%%% LEMMA - Distortion Lemma - Proof %%%%%%%%%%%%%%
%%%%%%%%%%%%%%%%%%%%%%%%%%%%%%%%%%%%%%%%%%%%%%%
\begin{proof}[Proof of Lemma~\ref{lem:CAPI-DistortionLemma}]
Let $\eps = \smallnorm{\QpiPrimehat - \QpiPrime}_{\infty,\Dn}$ and define the  set
$A_\eps = \{x:0 < \gapQPrime(x) \leq 4 \eps \}$.
Denote $p = \ProbWRT{\nu}{X \in A_\eps}$.
For any $z > 0$, Bernstein inequality Theorem 6.12 of~\citep{SteinwartChritmann2008}) shows that
$\ProbWRT{\nu_n}{X \in A_\eps} - \ProbWRT{\nu}{X \in A_\eps} \leq 
\sqrt{\frac{2 p(1-p) z}{n} } + \frac{2z}{3n}$ with probability at least $1 - e^{-z}$.
By the arithmetic mean--geometric mean inequality
\revised{$\sqrt{[p(1-p)] \frac{2z}{n} } \leq \frac{p(1-p)}{2} + \frac{2z}{2n} \leq \frac{p}{2} + \frac{z}{n}$, so we get}
\begin{align}\label{eq:CAPI-DistortionLemma-Proof-Pnun2Pnu}
	\ProbWRT{\nu_n}{X \in A_\eps} \leq \frac{3}{2} \ProbWRT{\nu}{X \in A_\eps} + \frac{5z}{3n}
\end{align}
with probability at least $1 - e^{-z}$. From now on, we focus on the event that this inequality holds.

Define the new auxiliary loss $\tilde{L}_n(\pi) = \int_{\XX} \gapQPrime(x) \One{\pi(x) \neq \argmax_{a \in \AA} \revised{\QpiPrimehat(x,a)}} \mathrm{d} \nu_n$. Notice that unlike $\hat{L}_n(\pi)$, it uses the weighting function $\gapQPrime$ (instead of $\gapQPrimehat$).
%
% XXX For any $\pi$, $L_n(\pi) - \hat{L}_n(\pi) \leq | L_n(\pi) - \tilde{L}_n(\pi) | + |\tilde{L}_n(\pi) - \hat{L}_n(\pi)|$.
In the following, for any  $\pi$, we first relate $L_n(\pi)$ to $\tilde{L}_n(\pi)$, and then relate $\tilde{L}_n(\pi)$ to $\hat{L}_n(\pi)$.

\noindent
\textbf{Upper bounding $|L_n(\pi) - \tilde{L}_n(\pi)|$.}
For any $\pi$, 
$ | L_n(\pi) - \tilde{L}_n(\pi)  | =
	 |
		\int_\XX
			\gapQPrime(x)	 	[
								\One{\pi(x) \neq \argmax_{a \in \AA} \QpiPrime(x,a)} -
								\One{\pi(x) \neq \argmax_{a \in \AA} \QpiPrimehat(x,a)}
							] \mathrm{d} \nu_n  |
	\leq \int_\XX
		\gapQPrime(x) \cdot   \One{\argmax_{a \in \AA} \QpiPrime(x,a) \neq \argmax_{a \in \AA} \QpiPrimehat(x,a)}
		\mathrm{d} \nu_n
	 = \int_{A_\eps} \gapQPrime(x) \One{\argmax_{a \in \AA} \QpiPrime(x,a) \neq \argmax_{a \in \AA} \QpiPrimehat(x,a)} \cdot
		\mathrm{d} \nu_n  + 
	\int_{A^c_\eps} \gapQPrime(x) \times \\ \One{\argmax_{a \in \AA} \QpiPrime(x,a) \neq \argmax_{a \in \AA} \QpiPrimehat(x,a)} 
		\mathrm{d} \nu_n$.

Whenever $|\QpiPrimehat(x,a) - \QpiPrime(x,a)| < \frac{1}{2} \gapQPrime(x)$ (for $x \in \Dn$ and $a \in \{1,2\}$), the maximizer action is the same. So on the set $A_\eps^c$, where $\gapQPrime(x) > 4 \eps \geq 4 |\QpiPrimehat(x,a) - \QpiPrime(x,a)|$,
the value of $\One{\argmax_{a \in \AA} \QpiPrime(x,a) \neq \argmax_{a \in \AA} \QpiPrimehat(x,a)}$ is always zero. Thus for any $z > 0$, we have
\ifDoubleColumn
%%%%%%%%% DOUBLE COLUMN %%%%%%%%%
\begin{align}\label{eq:CAPI-DistortionLemma-Proof-L2Ltilde1}
\nonumber
	& \left| L_n(\pi) - \tilde{L}_n(\pi) \right | \leq
%	\\ &
%	\int_{A_\eps} g_\QpiPrime(x) \One{\argmax_{a \in \AA} \QpiPrime(x,a) \neq \argmax_{a \in \AA} \QpiPrimehat(x,a)} 		\mathrm{d} \nu_n
%	\leq
%	\\ &
	(4 \eps) \ProbWRT{\nu_n}{X \in A_\eps} \leq
	\\ &
	\nonumber
	4 \eps 	\left[
				\frac{3}{2} \ProbWRT{\nu}{X \in A_\eps} + \frac{5z}{3n}
			\right]
	\leq
	\\
	&
	\nonumber
	6 \times 2^{2\zeta} \norm{\QpiPrimehat - \QpiPrime}_{\infty,\Dn}^{1+\zeta} + 
	\frac{20}{3} \norm{\QpiPrimehat - \QpiPrime}_{\infty,\Dn} \frac{z}{n} \leq
	\\ &
	c_1(\zeta) \norm{\QpiPrimehat - \QpiPrime}_{\infty,\Dn}^{1+\zeta} + c_2(\Qmax) \frac{z}{n}.
\end{align}
\else
%%%%%%%%% SINGLE COLUMN %%%%%%%%%
{\small
\begin{align}\label{eq:CAPI-DistortionLemma-Proof-L2Ltilde1}
\nonumber
	& \left| L_n(\pi) - \tilde{L}_n(\pi) \right | \leq
%	\\ &
%	\int_{A_\eps} g_\QpiPrime(x) \One{\argmax_{a \in \AA} \QpiPrime(x,a) \neq \argmax_{a \in \AA} \QpiPrimehat(x,a)} 		\mathrm{d} \nu_n
%	\leq
%	\\ &
	(4 \eps) \ProbWRT{\nu_n}{X \in A_\eps} \leq
	4 \eps 	\left[
				\frac{3}{2} \ProbWRT{\nu}{X \in A_\eps} + \frac{5z}{3n}
			\right]
	\leq
	\\
	&
	6 \times 2^{2\zeta} \norm{\QpiPrimehat - \QpiPrime}_{\infty,\Dn}^{1+\zeta} + 
	\frac{20}{3} \norm{\QpiPrimehat - \QpiPrime}_{\infty,\Dn} \frac{z}{n} \leq
	c_1(\zeta) \norm{\QpiPrimehat - \QpiPrime}_{\infty,\Dn}^{1+\zeta} + c_2(\Qmax) \frac{z}{n}.
\end{align}
}
\fi
Here we used~\eqref{eq:CAPI-DistortionLemma-Proof-Pnun2Pnu} in the second inequality, Assumption~\ref{ass:CAPI-ActionGap} in the third inequality, and $\smallnorm{\QpiPrimehat - \QpiPrime}_{\infty,\Dn} \leq 2 \Qmax$ in the last one.

\noindent
\textbf{Relation of $\hat{L}_n(\pi)$ to $\tilde{L}_n(\pi)$.}
First note that $|\gapQPrimehat(x) - \gapQPrime(x)| \leq 2 \eps$ (for all $x \in \Dn$).
We also have
$	\max_{x \in A_\eps^c \cap \Dn} \frac{\gapQPrimehat(x) - \gapQPrime(x) }{\gapQPrime(x)} \leq \frac{2\eps}{4 \eps} = \frac{1}{2}$ and
$\max_{x \in A_\eps^c \cap \Dn} \frac{\gapQPrime(x) - \gapQPrimehat(x)}{\gapQPrimehat(x)} \leq \frac{2\eps}{2 \eps} = 1$.
Thus,
\ifDoubleColumn
%%%%%%%%% DOUBLE COLUMN %%%%%%%%%
{\small
\begin{align*}
	& 	\hat{L}_n(\pi) - \tilde{L}_n(\pi) = \\
	& 
	\int_{A_\eps} 
			( \gapQPrimehat(x) - \gapQPrime(x) )
			\One{\pi(x) \neq \argmax_{a \in \AA} \QpiPrimehat(x,a)} \mathrm{d} \nu_n + {}
	\\ &
	\int_{A_\eps^c} 
			\frac{\gapQPrimehat(x) - \gapQPrime(x)}{\gapQPrime(x)}
			\gapQPrime(x) 
			\One{\pi(x) \neq \argmax_{a \in \AA} \QpiPrimehat(x,a)} \mathrm{d} \nu_n
	\\ 
	& {} \leq  {}
	(2 \eps) \ProbWRT{\nu_n}{X \in A_\eps} + \frac{1}{2} \tilde{L}_n(\pi).
\end{align*}
}
\else
%%%%%%%%% SINGLE COLUMN %%%%%%%%%
\begin{align*}
	 \hat{L}_n(\pi) - \tilde{L}_n(\pi) & = 
	 \int_{A_\eps} 
			( \gapQPrimehat(x) - \gapQPrime(x) )
			\One{\pi(x) \neq \argmax_{a \in \AA} \QpiPrimehat(x,a)} \mathrm{d} \nu_n + {}
	\\ &
	\quad \int_{A_\eps^c} 
			\frac{\gapQPrimehat(x) - \gapQPrime(x)}{\gapQPrime(x)}
			\gapQPrime(x) 
			\One{\pi(x) \neq \argmax_{a \in \AA} \QpiPrimehat(x,a)} \mathrm{d} \nu_n
	\\ 
	 & \leq
	(2 \eps) \ProbWRT{\nu_n}{X \in A_\eps} + \frac{1}{2} \tilde{L}_n(\pi).
\end{align*}
\fi
After re-arranging, we get
\begin{align}\label{eq:CAPI-DistortionLemma-Proof-Lhat2Ltilde1}
	\hat{L}_n(\pi) \leq \frac{3}{2} \tilde{L}_n(\pi) + 2 \eps \, \ProbWRT{\nu_n}{X \in A_\eps}.
\end{align}
Likewise, by writing $\gapQPrime(x) - \gapQPrimehat(x)$ as $\frac{\gapQPrime(x) - \gapQPrimehat(x) }{\gapQPrimehat(x)} \gapQPrimehat(x)$ and doing a similar decomposition of the state space into $A_\eps$ and $A_\eps^c$, we get
\begin{align}\label{eq:CAPI-DistortionLemma-Proof-Lhat2Ltilde2}
	\hat{L}_n(\pi) \geq \frac{1}{2} \tilde{L}_n(\pi) - \eps \ProbWRT{\nu_n}{X \in A_\eps}.
\end{align}

We use the optimizer property of $\hat{\pi}_n$ (which implies that $\hat{L}_n(\hat{\pi}_n) \leq \hat{L}_n(\pi^*_n)$), apply~\eqref{eq:CAPI-DistortionLemma-Proof-Lhat2Ltilde1}, and finally use inequalities~\eqref{eq:CAPI-DistortionLemma-Proof-L2Ltilde1} and~\eqref{eq:CAPI-DistortionLemma-Proof-Pnun2Pnu} to get
$
	\hat{L}_n(\hat{\pi}_n) \leq 
	\hat{L}_n(\pi^*_n) \leq
	\frac{3}{2} \tilde{L}_n(\pi^*_n) + (2 \eps) \ProbWRT{\nu_n}{X \in A_\eps} 
	 \leq 
	\frac{3}{2} [ L_n(\pi^*_n) + c_1 \smallnorm{\QpiPrimehat - \QpiPrime}_{\infty,\Dn}^{1+\zeta} + c_2 \frac{z}{n} ] +	(2 \eps) [ 
				\frac{3}{2} \ProbWRT{\nu}{X \in A_\eps} + \frac{5}{3} \frac{z}{n}
			]$.
From~\eqref{eq:CAPI-DistortionLemma-Proof-Lhat2Ltilde2} and by applying~\eqref{eq:CAPI-DistortionLemma-Proof-L2Ltilde1}, we also have
$	\hat{L}_n(\hat{\pi}_n) \geq \frac{1}{2} \tilde{L}_n(\hat{\pi}_n) - \eps \ProbWRT{\nu_n}{X \in A_\eps}
	 \geq
	\frac{1}{2}	[
				L_n(\pihat_n) - c_1 \smallnorm{\QpiPrimehat - \QpiPrime}_{\infty,\Dn}^{1+\zeta} - c_2 \frac{z}{n} ] -
	\eps [  \frac{3}{2} \ProbWRT{\nu}{X \in A_\eps} + \frac{5z}{3n} ]
$.
These two inequalities imply that
$L_n(\pihat_n) \leq 3 L_n(\pi^*_n) + c_1 \smallnorm{\QpiPrimehat - \QpiPrime}_{\infty,\Dn}^{1+\zeta} + c_2 \frac{z}{n}$ in the event that~\eqref{eq:CAPI-DistortionLemma-Proof-Pnun2Pnu} holds, which has probability at least $1 - e^{-z}$.
\end{proof}

%%%%%%%%%%%%%%%%%%%%%%%%%%%%%%%%%%%%%%%%%%%%%%%

%%%%%%%%%%%%%%%%%%%%%%%%%%%%%%%%%%%%%%%%%%%%%%%
%%%%%%%%%%%% THEOREM -  Loss at each iteration - Proof %%%%%%%%%%%%
%%%%%%%%%%%%%%%%%%%%%%%%%%%%%%%%%%%%%%%%%%%%%%%
\begin{proof}[Proof of Theorem~\ref{thm:CAPI-ErrorInEachIteration}] 
We use Theorem 3.3 by~\citep{BartlettBousquetMendelson05}.
For function $l(\pi)(x) = \gapQPrime(x) \One{\pi(x) \neq \argmax_{a \in \AA} \QpiPrime(x,a)}$, we have
$
	\Var{l(\pi)(X)} \leq \EE{|\gapQPrime(X) \One{\pi(X) \neq \argmax_{a \in \AA} \QpiPrime(X,a)}|^2 }
	\leq 
2\Qmax \EE{l(\pi)(X)}$,
so the variance condition of that theorem is satisfied. If we have a function $\Psi$ as defined  in~\eqref{eq:CAPI-ComplexityCondition}, the theorem states that there exist $c_1, c_2> 0$ such that for any $z > 0$ and any $\pi \in \Pi$ (including $\pihat_n \in \Pi$),
\begin{align}\label{eq:CAPI-ErrorInEachIteration-Proof-GeneralConcentrationBound}
	L(\pi) = \Pr l(\pi) \leq 2P_n l(\pi) + c_1 r^* + c_2 \frac{z}{n},
\end{align}
with probability at least $1 - e^{-z}$ ($c_1$ can be chosen as \revised{$704/\Qmax$} and $c_2$ can be chosen as $126 \, \Qmax$). 
% In particular, this inequality holds for $\pihat_n \in \Pi$.

Let $\pioptInPi  \leftarrow \argmin_{\pi \in \Pi} L(\pi)$ be the minimizer of the expected loss in policy space $\Pi$.
Consider~\eqref{eq:CAPI-ErrorInEachIteration-Proof-GeneralConcentrationBound} with the choice of $\pi = \pihat_n$, and add and subtract $6P_n l(\pioptInPi)$ and $6 \Pr l(\pioptInPi)$ and then use Lemma~\ref{lem:CAPI-DistortionLemma}. With probability at least $1 - 2 e^{-z}$, we get
$
	L(\pihat_n) \leq 
				2 \Pr_n l(\pihat_n) - 6 \left[ \Pr_n l(\pioptInPi) - \Pr_n l(\pioptInPi) \right]
				 - 6 \left[ \Pr l(\pioptInPi) - \Pr l(\pioptInPi) \right] + c_1 r^* + c_2 \frac{z}{n}
			\leq
			6 \left[ \Pr_n l(\pi^*_n) -  \Pr_n l(\pioptInPi) \right] + 6 \left[ \Pr_n l(\pioptInPi) - \Pr l(\pioptInPi) \right]
			+ 6 \Pr l(\pioptInPi)  + c_1 r^* + c_2 \smallnorm{\QpiPrimehat - \QpiPrime}_{\infty,\Dn}^{1+\zeta} +  c_3 \frac{z}{n} 
			\leq
			6 \left[ \Pr_n l(\pioptInPi) - \Pr l(\pioptInPi) \right] + 6 \Pr l(\pioptInPi) + c_1 r^* 
			+ c_2 \smallnorm{\QpiPrimehat - \QpiPrime}_{\infty,\Dn}^{1+\zeta} + c_3 \frac{z}{n}
			$,
where in the last inequality we used the minimizing property of $\pi^*_n$, i.e., $\Pr_n l(\pi^*_n) - \Pr_n l(\pioptInPi) \leq 0$.
\revised{Here $c_2$ can be chosen as $36 \times 2^{2\zeta}$.}

To upper bound $\Pr_n l(\pioptInPi) - \Pr l(\pioptInPi)$, we apply Bernstein inequality 
(Theorem 6.12 of~\citep{SteinwartChritmann2008}) to get that for any $ z > 0$,
$\Pr_n l(\pioptInPi) - \Pr l(\pioptInPi) \leq \sqrt{ \frac{2 \Var{l(\pioptInPi)} z}{ n} } + \frac{4 \Qmax z}{3 n}$, with probability at least $1 - e^{-z}$.
Since $\Var{l(\pioptInPi)} \leq 2 \Qmax \Pr l(\pioptInPi)$ (as shown above), by the application of arithmetic mean--geometric mean inequality we obtain 
$\Pr_n l(\pioptInPi) - \Pr l(\pioptInPi) \leq \Pr l(\pioptInPi) + \frac{7 \Qmax z}{3 n}$ with the same probability.
This and the inequality in the previous paragraph result in
$L(\pihat_n) \leq 12 \Pr l(\pioptInPi) + c_1 r^* + c_2 \smallnorm{\QpiPrimehat - \QpiPrime}_{\infty,\Dn}^{1+\zeta} + c_3 \frac{z}{n}$,
with probability at least $1 - 3 e^{-z}$ as desired.
\end{proof}

%%%%%%%%%%%%%%%%%%%%%%%%%%%%%%%%%%%%%%%%%%%%%%%
%%%%%%%%%%%%%%%%%%%%%%%%%%%%%%%%%%%%%%%%%%%%%%%
%%%%%%%%%%%%%%%%%%%%%%%%%%%%%%%%%%%%%%%%%%%%%%%

%%%%%%%%%%%%%%%%%%%%%%%%%%%%%%%%%%%%%%%%%%%%%%%
%%%%%%%%%% THEOREM -  Performance Loss for CAPI - Proof %%%%%%%%%%%
%%%%%%%%%%%%%%%%%%%%%%%%%%%%%%%%%%%%%%%%%%%%%%%
\begin{proof}[Proof of Theorem~\ref{thm:CAPI-PerformanceLoss}]
It is shown by~\citep{LazaricGhavamzadehMunosDPI2010} that
$
\Vopt - V^{\pi_{K}} \leq 
	\sum_{k=0}^{K-1} \gamma^{K-k-1} (\PKernel^\piopt)^{K-k-1}
		 \sum_{m \geq 0} \allowbreak \gamma^m (\PKernel^{\pi_{k+1}})^m l^{\pi_{k}}(\pi_{k+1})
	+ (\gamma \PKernel^\piopt)^K (\Vopt - V^{\pi_{0}}).
$
We apply $\rho$ to both sides and use the definition of $c_{\rho,\nu}(m_1;m_2;\pi)$ to get  
$\rho (\Vopt - V^{\pi_{K}}) \leq 
\sum_{k=0}^{K-1} \gamma^{K-k-1} \sum_{m \geq 0} \gamma^m c_{\rho,\nu}(K-k-1,m;\pi_{k+1}) \, \nu l^{\pi_{k}}(\pi_{k+1}) 
	+ \gamma^K (2 \Qmax)$.
Recall that $\nu l^{\pi_{k}}(\pi_{k+1}) = L^{\pi_{k}}(\pi_{k+1})$. We decompose $\gamma$ to $\gamma^s \gamma^{(1-s)}$ (for $0 \leq s \leq 1$) and separate terms involving the concentrability coefficients and those related to $L^{\pi_{k}}(\pi_{k+1})$.
We then have for any $0 \leq s \leq 1$,
$\rho (\Vopt - V^{\pi_{K}}) \leq 
\max_{0 \leq k \leq K-1} \left \{ \gamma^{s(K-k-1)} L^{\pi_{k}}(\pi_{k+1}) \right \} \times \sum_{k'=0}^{K-1} \gamma^{(1-s)k'} \sum_{m \geq 0} \gamma^m \sup_{\pi' \in \Pi} c_{\rho,\nu} (k',m;\pi') + \gamma^K (2 \Qmax)$.
Taking the infimum w.r.t. $s$ and using the definition of $C_{\rho,\nu}(K)$, we get that
$\mathrm{Loss}(\pi_{K};\rho) = \rho (\Vopt - V^{\pi_{K}}) \leq
	\frac{2}{1-\gamma} 
		[
		\inf_{s \in [0,1]} C_{\rho,\nu}(K,s) 
		\max_{0 \leq k \leq K-1}	[\gamma^{s(K-k-1)} L^{\pi_{k}}(\pi_{k+1}) ] 
			 + \gamma^K \Rmax
		].
$

Fix $ 0 < \delta < 1$.
For each iteration $k = 0,\dotsc, K-1$, by invoking Theorem~\ref{thm:CAPI-ErrorInEachIteration} with the confidence parameter $\delta / K$, we get
$L^{\pi_{k}}(\pi_{k+1})
	\leq
	12 \inf_{\pi \in \Pi} L^{\pi_{k}} (\pi) + c_1 r^* + c_2 \smallnorm{\hat{Q}^{\pi_{k}} - Q^{\pi_{k}} }_{\infty,\Dn^{k}}^{1 + \zeta}
	 +  c_3 \frac{\ln(K/\delta)}{n}$, which holds with probability at least $1 - \delta/K$.
Since $\inf_{\pi \in \Pi} L^{\pi_{k}} (\pi) \leq d(\Pi)$, the previous set of inequalities alongside the upper bound on $\mathrm{Loss}(\pi_{K};\rho)$ imply the desired result.
\end{proof}

\else %%% 
%%%%%%%%%%%%%%%%%%%%%%%%%%%%%%%%%%%%%%%%%%%%%%%
%%%%%%%%%%%%%%%%%%%%%%%%%%%%%%%%%%%%%%%%%%%%%%%
%%%%%%%%%%%%%%%%%%%%%%%%%%%%%%%%%%%%%%%%%%%%%%%

%%%%%%%%%%%%%%%%%%%%%%%%%%%%%%%%%%%%%%%%%%%%%%%
%%%%%%%%%%%%%%%%%%%%%%%%%%%%%%%%%%%%%%%%%%%%%%%
%%%%%%%%%%%%%%%%%%%%%%%%%%%%%%%%%%%%%%%%%%%%%%%

%%%%%%%%%%%%%%%%%%%%%%%%%%%%%%%%%%%%%%%%%%%%%%%
%%%%%%%%%%%%%%%%%%%% Long Version  %%%%%%%%%%%%%%%%%%%
%%%%%%%%%%%%%%%%%%%%%%%%%%%%%%%%%%%%%%%%%%%%%%%

%%%%%%%%%%%%%%%%%%%%%%%%%%%%%%%%%%%%%%%%%%%%%%%
%%%%%%%%%%%%%%%%%%%%%%%%%%%%%%%%%%%%%%%%%%%%%%%
%%%%%%%%%%%%%%%%%%%%%%%%%%%%%%%%%%%%%%%%%%%%%%%

%%%%%%%%%%%%%%%%%%%%%%%%%%%%%%%%%%%%%%%%%%%%%%%
%%%%%%%%%%%%%%%%%%%%%%%%%%%%%%%%%%%%%%%%%%%%%%%
%%%%%%%%%%%%%%%%%%%%%%%%%%%%%%%%%%%%%%%%%%%%%%%

\section{Concentration Inequalities}
\label{sec:CAPI-Appendix-ConcentrationInequalities}
For the convenience of the reader, we quote Bernstein inequality and Theorem 3.3 of~\citet{BartlettBousquetMendelson05}.

%%%%%%%%%%%%%%%%%%%%%%%%%%%%%%%%%%%%%%%%%%%%%%%
%%%%%%%%%%%%%%%%%% Bernstein Inequality %%%%%%%%%%%%%%%%%%
%%%%%%%%%%%%%%%%%%%%%%%%%%%%%%%%%%%%%%%%%%%%%%%
\begin{lemma}[Bernstein inequality -- Theorem 6.12 of~\citet{SteinwartChritmann2008}]\label{lem:CAPI-BernsteinInequality}
Let $(\Omega,\mathcal{A},P)$ be a probability space, $B > 0$, and $\sigma > 0$ be real numbers, and $n \geq 1$ be an integer. Furthermore, let $X_1, \dotsc, X_n: \Omega \to \Real$ be independent random variables satisfying $\EE{X_i} = 0$, $\norm{X_i}_\infty \leq B$, and $\EE{X_i^2} \leq \sigma^2$ for all $i=1,2, \dotsc, n$. Then we have
$\Prob{\frac{1}{n} \sum_{i=1}^n X_i \geq \sqrt{\frac{2 \sigma^2 z}{n}} + \frac{2 B z}{3 n} }
	\leq e^{-z}$ $(z > 0)$.
%	\begin{align*}
%		\Prob{\frac{1}{n} \sum_{i=1}^n X_i \geq \sqrt{\frac{2 \sigma^2 z}{n}} + \frac{2 B z}{3 n} }
%		\leq e^{-z}.
%				\qquad	(z > 0)
%	\end{align*}
\end{lemma}

%%%%%%%%%%%%%%%%%%%%%%%%%%%%%%%%%%%%%%%%%%%%%%%
%%%% Concentration inequality with Local Rademacher Complexity Inequality  %%%%%
%%%%%%%%%%%%%%%%%%%%%%%%%%%%%%%%%%%%%%%%%%%%%%%
\begin{theorem}[Theorem 3.3 of~\citet{BartlettBousquetMendelson05} -- First Part]\label{thm:ConcentrationWithLocalRademacher}
Let $\FF$ be a class of functions with ranges in $[a,b]$ and assume that there are some functional $T: \FF \to \Real^+$ and some constant $B$ such that for every $f \in \FF$, $\Var{f} \leq T(f) \leq B \Pr f$. Let $\Psi$ be a sub-root function and let $r^*$ be the fixed point of $\Psi$.
Assume that for any $r \geq r^*$,  $\Psi$ satisfies
$\Psi(r) \geq B \EE{ R_n \{ f \in \FF: T(f) \leq r \} }$.
Then with $c_1 = 704$ and $c_2 = 26$, for any $K > 1$ and every $x > 0$, with probability at least $1 - e^{-x}$, for any $f \in \FF$, we have
\[
	\Pr f \leq \frac{K}{K-1} \Pr_n f + \frac{c_1 K}{B} r^* + \frac{ x(11(b-a) + c_2 B K)}{n}.
\]
\end{theorem}

%%%%%%%%%%%%%%%%%%%%%%%%%%%%%%%%%%%%%%%%%%%%%%%
%%%%%%%%%%%%%%%%%%%%%%%%%%%%%%%%%%%%%%%%%%%%%%%
%%%%%%%%%%%%%%%%%%%%%%%%%%%%%%%%%%%%%%%%%%%%%%%
\section{Proof of the Main Result}
\label{sec:CAPI-Appendix-Proofs}

Recall that the goal is to provide an upper bound for the loss $L(\pihat_n)$, where $\pihat_n$ is the minimizer of the distorted empirical loss function $\hat{L}_n$ obtained at each iteration (cf.~\eqref{eq:CAPI-Theory-pihatDefinition}).
The loss function $\hat{L}_n$, however, is defined based on the estimate $\QpiPrimehat$ instead of the true action-value function $\QpiPrime$. This causes some errors. 
So in Lemma~\ref{lem:CAPI-DistortionLemma}, we quantify how well $\pihat_n$ minimizes the loss function $L_n$ (defined based on $\QpiPrime$ and unavailable to the algorithm).
Having this result, we prove Theorem~\ref{thm:CAPI-ErrorInEachIteration}, which upper bounds the true loss $L(\pihat_n)$.

%%%%%%%%%%%%%%%%%%%%%%%%%%%%%%%%%%%%%%%%%%%%%%%
%%%%%%%%%%%%%%%% LEMMA - Distortion Lemma %%%%%%%%%%%%%%%%
%%%%%%%%%%%%%%%%%%%%%%%%%%%%%%%%%%%%%%%%%%%%%%%
\begin{lemma}[Loss Distortion Lemma]\label{lem:CAPI-DistortionLemma}
Fix a policy $\pi'$. Suppose that $\QpiPrimehat$ is an approximation of the action-value function $\QpiPrime$. Given the dataset $\Dn$, let $\pihat_n$ be defined as~\eqref{eq:CAPI-Theory-pihatDefinition} and define $\pi^*_n \leftarrow \argmin_{\pi \in \Pi} L_n(\pi)$.
Let Assumption~\ref{ass:CAPI-ActionGap} hold. There exist finite $c_1, c_2 > 0$, which depend only on $\zeta$, $c_g$, and $\Qmax$, such that for any $z > 0$, we have
\begin{align*}
	L_n(\hat{\pi}_n)
	\leq
	3 L_n(\pi^*_n) + c_1 \norm{\QpiPrimehat - \QpiPrime}_{\infty,\Dn}^{1+\zeta} + c_2 \frac{z}{n},
\end{align*}
with probability at least $1 - e^{-z}$.
\end{lemma}

%To understand this lemma, note that $\pihat_n$ is the minimizer of $\hat{L}_n(\pi)$, which is the empirical loss that $\QpiPrimehat$ instead of the true action-value function $\QpiPrime$.

This lemma states that the empirical loss $L_n$ of policy $\pihat_n$, which is the minimizer of the distorted empirical loss $\hat{L}_n$, is still close to the true minimizer of $L_n$, which uses the true, but unavailable, action-value value function $\QpiPrime$.
The difference mainly depends on the error of estimating $\QpiPrime$ by $\QpiPrimehat$. When the problem has a favourable action-gap regularity (i.e., $\zeta > 0$), the error is exponentially dampened by a factor of $1 + \zeta$.
This lemma currently uses the supremum norm of the policy evaluation error, but one might also extend it to an $L_p$-norm result similar to what was done by~\citet{FarahmandNIPS2011}. Note that even though the proof technique here has some similarities with the proofs by \citep{FarahmandNIPS2011}, they are considerably different.

In the proofs, $c_1, c_2, \dotsc$ are constants whose values may change from line to line -- unless specified otherwise.
%%%%%%%%%%%%%%%%%%%%%%%%%%%%%%%%%%%%%%%%%%%%%%%
%%%%%%%%%%%%%% LEMMA - Distortion Lemma - Proof %%%%%%%%%%%%%%
%%%%%%%%%%%%%%%%%%%%%%%%%%%%%%%%%%%%%%%%%%%%%%%
\begin{proof}[Proof of Lemma~\ref{lem:CAPI-DistortionLemma}]
Let $\eps = \smallnorm{\QpiPrimehat - \QpiPrime}_{\infty,\Dn}$ and define the  set
$A_\eps = \{x:0 < \gapQPrime(x) \leq 4 \eps \}$.
Denote $p = \ProbWRT{\nu}{X \in A_\eps}$.
For any $z > 0$, Bernstein inequality (Lemma~\ref{lem:CAPI-BernsteinInequality} in Appendix~\ref{sec:CAPI-Appendix-ConcentrationInequalities}) shows that
$\ProbWRT{\nu_n}{X \in A_\eps} - \ProbWRT{\nu}{X \in A_\eps} \leq 
\sqrt{\frac{2 p(1-p) z}{n} } + \frac{2z}{3n}$ with probability at least $1 - e^{-z}$.
By the arithmetic mean--geometric mean inequality
\revised{$\sqrt{[p(1-p)] \frac{2z}{n} } \leq \frac{p(1-p)}{2} + \frac{2z}{2n} \leq \frac{p}{2} + \frac{z}{n}$, so we get}
\begin{align}\label{eq:CAPI-DistortionLemma-Proof-Pnun2Pnu}
	\ProbWRT{\nu_n}{X \in A_\eps} \leq \frac{3}{2} \ProbWRT{\nu}{X \in A_\eps} + \frac{5z}{3n}
\end{align}
with probability at least $1 - e^{-z}$. From now on, we focus on the event that this inequality holds.

Define the new auxiliary loss $\tilde{L}_n(\pi) = \int_{\XX} \gapQPrime(x) \One{\pi(x) \neq \argmax_{a \in \AA} \revised{\QpiPrimehat(x,a)}} \mathrm{d} \nu_n$. Notice that unlike $\hat{L}_n(\pi)$, it uses the weighting function $\gapQPrime$ (instead of $\gapQPrimehat$).
%
% XXX For any $\pi$, $L_n(\pi) - \hat{L}_n(\pi) \leq | L_n(\pi) - \tilde{L}_n(\pi) | + |\tilde{L}_n(\pi) - \hat{L}_n(\pi)|$.
In the following, for any  $\pi$, we first relate $L_n(\pi)$ to $\tilde{L}_n(\pi)$, and then relate $\tilde{L}_n(\pi)$ to $\hat{L}_n(\pi)$.

\noindent
\textbf{Upper bounding $|L_n(\pi) - \tilde{L}_n(\pi)|$.}
For any $\pi$, 
\ifDoubleColumn
%%%%%%%%% DOUBLE COLUMN %%%%%%%%%
\begin{align*}
	& \left| L_n(\pi) - \tilde{L}_n(\pi) \right | = \\
	& 
	\Bigg |
		\int_\XX
			\gapQPrime(x)	 	\Big[
								\One{\pi(x) \neq \argmax_{a \in \AA} \QpiPrime(x,a)} - {} \\
							& \qquad \qquad \qquad
								\One{\pi(x) \neq \argmax_{a \in \AA} \QpiPrimehat(x,a)}
							\Big] \mathrm{d} \nu_n \Bigg |  \\
	&
	\leq \int_\XX
		\gapQPrime(x)  \One{\argmax_{a \in \AA} \QpiPrime(x,a) \neq \argmax_{a \in \AA} \QpiPrimehat(x,a)}
		\mathrm{d} \nu_n \\
	&
	 = \int_{A_\eps} \gapQPrime(x) \One{\argmax_{a \in \AA} \QpiPrime(x,a) \neq \argmax_{a \in \AA} \QpiPrimehat(x,a)}
		\mathrm{d} \nu_n % + {}
	\\
	& {} + 
	\int_{A^c_\eps} \gapQPrime(x)  \One{\argmax_{a \in \AA} \QpiPrime(x,a) \neq \argmax_{a \in \AA} \QpiPrimehat(x,a)}
		\mathrm{d} \nu_n.	
\end{align*}
\else
%%%%%%%%% SINGLE COLUMN %%%%%%%%%
\begin{align*}
	& \left| L_n(\pi) - \tilde{L}_n(\pi) \right | = \\
	& 
	\Bigg |
		\int_\XX
			\gapQPrime(x)	 	\Big[
								\One{\pi(x) \neq \argmax_{a \in \AA} \QpiPrime(x,a)} - %\\
							%& \qquad \qquad \qquad
								\One{\pi(x) \neq \argmax_{a \in \AA} \QpiPrimehat(x,a)}
							\Big] \mathrm{d} \nu_n \Bigg |  \leq \\
	&
	\int_\XX
		\gapQPrime(x)  \One{\argmax_{a \in \AA} \QpiPrime(x,a) \neq \argmax_{a \in \AA} \QpiPrimehat(x,a)}
		\mathrm{d} \nu_n = \\
	&
	 \int_{A_\eps} \gapQPrime(x) \One{\argmax_{a \in \AA} \QpiPrime(x,a) \neq \argmax_{a \in \AA} \QpiPrimehat(x,a)}
		\mathrm{d} \nu_n + {}
	\\
	& 
	\int_{A^c_\eps} \gapQPrime(x)  \One{\argmax_{a \in \AA} \QpiPrime(x,a) \neq \argmax_{a \in \AA} \QpiPrimehat(x,a)}
		\mathrm{d} \nu_n.	
\end{align*}
\fi % Single/Two Column

Whenever $|\QpiPrimehat(x,a) - \QpiPrime(x,a)| < \frac{1}{2} \gapQPrime(x)$ (for $x \in \Dn$ and $a \in \{1,2\}$), the maximizer action is the same. So on the set $A_\eps^c$, where $\gapQPrime(x) > 4 \eps \geq 4 |\QpiPrimehat(x,a) - \QpiPrime(x,a)|$,
the value of $\One{\argmax_{a \in \AA} \QpiPrime(x,a) \neq \argmax_{a \in \AA} \QpiPrimehat(x,a)}$ is always zero. Thus for any $z > 0$, we have
\ifDoubleColumn
%%%%%%%%% DOUBLE COLUMN %%%%%%%%%
\begin{align}\label{eq:CAPI-DistortionLemma-Proof-L2Ltilde1}
\nonumber
	& \left| L_n(\pi) - \tilde{L}_n(\pi) \right | \leq
%	\\ &
%	\int_{A_\eps} g_\QpiPrime(x) \One{\argmax_{a \in \AA} \QpiPrime(x,a) \neq \argmax_{a \in \AA} \QpiPrimehat(x,a)} 		\mathrm{d} \nu_n
%	\leq
%	\\ &
	(4 \eps) \ProbWRT{\nu_n}{X \in A_\eps} \leq
	\\ &
	\nonumber
	4 \eps 	\left[
				\frac{3}{2} \ProbWRT{\nu}{X \in A_\eps} + \frac{5z}{3n}
			\right]
	\leq
	\\
	&
	\nonumber
	6 \times 2^{2\zeta} \norm{\QpiPrimehat - \QpiPrime}_{\infty,\Dn}^{1+\zeta} + 
	\frac{20}{3} \norm{\QpiPrimehat - \QpiPrime}_{\infty,\Dn} \frac{z}{n}
	\\ &
	 \leq
	c_1(\zeta) \norm{\QpiPrimehat - \QpiPrime}_{\infty,\Dn}^{1+\zeta} + c_2(\Qmax) \frac{z}{n}.
\end{align}
\else
%%%%%%%%% SINGLE COLUMN %%%%%%%%%
\begin{align}\label{eq:CAPI-DistortionLemma-Proof-L2Ltilde1}
\nonumber
	\left| L_n(\pi) - \tilde{L}_n(\pi) \right | & \leq
	(4 \eps) \ProbWRT{\nu_n}{X \in A_\eps} \leq
	4 \eps 	\left[
				\frac{3}{2} \ProbWRT{\nu}{X \in A_\eps} + \frac{5z}{3n}
			\right]
	\\
	& \leq
	\nonumber
	6 \times 2^{2\zeta} \norm{\QpiPrimehat - \QpiPrime}_{\infty,\Dn}^{1+\zeta} + 
	\frac{20}{3} \norm{\QpiPrimehat - \QpiPrime}_{\infty,\Dn} \frac{z}{n}
	\\ &
	 \leq
	c_1(\zeta) \norm{\QpiPrimehat - \QpiPrime}_{\infty,\Dn}^{1+\zeta} + c_2(\Qmax) \frac{z}{n}.
\end{align}
\fi
Here we used~\eqref{eq:CAPI-DistortionLemma-Proof-Pnun2Pnu} in the second inequality, Assumption~\ref{ass:CAPI-ActionGap} in the third inequality, and $\smallnorm{\QpiPrimehat - \QpiPrime}_{\infty,\Dn} \leq 2 \Qmax$ in the last one.

\noindent
\textbf{Relation of $\hat{L}_n(\pi)$ to $\tilde{L}_n(\pi)$.}
First note that $|\gapQPrimehat(x) - \gapQPrime(x)| \leq 2 \eps$ (for all $x \in \Dn$).
We also have
\begin{align*}
	\max_{x \in A_\eps^c \cap \Dn} \frac{\gapQPrimehat(x) - \gapQPrime(x) }{\gapQPrime(x)} \leq \frac{2\eps}{4 \eps} = \frac{1}{2}, \\
	\max_{x \in A_\eps^c \cap \Dn} \frac{\gapQPrime(x) - \gapQPrimehat(x)}{\gapQPrimehat(x)} \leq \frac{2\eps}{2 \eps} = 1.
\end{align*}
Thus,
\ifDoubleColumn
%%%%%%%%% DOUBLE COLUMN %%%%%%%%%
{\small
\begin{align*}
	& 	\hat{L}_n(\pi) - \tilde{L}_n(\pi) = \\
	&
	\int_{A_\eps} 
			( \gapQPrimehat(x) - \gapQPrime(x) )
			\One{\pi(x) \neq \argmax_{a \in \AA} \QpiPrimehat(x,a)} \mathrm{d} \nu_n + {}
	\\ &
	\int_{A_\eps^c} 
			\frac{\gapQPrimehat(x) - \gapQPrime(x)}{\gapQPrime(x)}
			\gapQPrime(x) 
			\One{\pi(x) \neq \argmax_{a \in \AA} \QpiPrimehat(x,a)} \mathrm{d} \nu_n
	\\ 
	& {} \leq  {}
	(2 \eps) \ProbWRT{\nu_n}{X \in A_\eps} + \frac{1}{2} \tilde{L}_n(\pi).
\end{align*}
}
\else
%%%%%%%%% SINGLE COLUMN %%%%%%%%%
\begin{align*}
	 \hat{L}_n(\pi) - \tilde{L}_n(\pi) & = 
	 \int_{A_\eps} 
			( \gapQPrimehat(x) - \gapQPrime(x) )
			\One{\pi(x) \neq \argmax_{a \in \AA} \QpiPrimehat(x,a)} \mathrm{d} \nu_n + {}
	\\ &
	\quad \int_{A_\eps^c} 
			\frac{\gapQPrimehat(x) - \gapQPrime(x)}{\gapQPrime(x)}
			\gapQPrime(x) 
			\One{\pi(x) \neq \argmax_{a \in \AA} \QpiPrimehat(x,a)} \mathrm{d} \nu_n
	\\ 
	 & \leq
	(2 \eps) \ProbWRT{\nu_n}{X \in A_\eps} + \frac{1}{2} \tilde{L}_n(\pi).
\end{align*}
\fi
After re-arranging, we get
\begin{align}\label{eq:CAPI-DistortionLemma-Proof-Lhat2Ltilde1}
	\hat{L}_n(\pi) \leq \frac{3}{2} \tilde{L}_n(\pi) + 2 \eps \, \ProbWRT{\nu_n}{X \in A_\eps}.
\end{align}
Likewise, by writing $\gapQPrime(x) - \gapQPrimehat(x)$ as $\frac{\gapQPrime(x) - \gapQPrimehat(x) }{\gapQPrimehat(x)} \gapQPrimehat(x)$ and doing a similar decomposition of the state space into $A_\eps$ and $A_\eps^c$, we get
\begin{align}\label{eq:CAPI-DistortionLemma-Proof-Lhat2Ltilde2}
	\hat{L}_n(\pi) \geq \frac{1}{2} \tilde{L}_n(\pi) - \eps \ProbWRT{\nu_n}{X \in A_\eps}.
\end{align}

We use the optimizer property of $\hat{\pi}_n$ (which implies that $\hat{L}_n(\hat{\pi}_n) \leq \hat{L}_n(\pi^*_n)$), apply~\eqref{eq:CAPI-DistortionLemma-Proof-Lhat2Ltilde1}, and finally use inequalities~\eqref{eq:CAPI-DistortionLemma-Proof-L2Ltilde1} and~\eqref{eq:CAPI-DistortionLemma-Proof-Pnun2Pnu} to get
\ifDoubleColumn
%%%%%%%%% DOUBLE COLUMN %%%%%%%%%
\begin{align*}
	\hat{L}_n(\hat{\pi}_n) & \leq 
	\hat{L}_n(\pi^*_n) \leq
	\frac{3}{2} \tilde{L}_n(\pi^*_n) + (2 \eps) \ProbWRT{\nu_n}{X \in A_\eps} 
	\\ & \leq {}
	\frac{3}{2} \left[ L_n(\pi^*_n) + c_1 \norm{\QpiPrimehat - \QpiPrime}_{\infty,\Dn}^{1+\zeta} + c_2 \frac{z}{n} \right] + {}
	\\ &
	\quad \,
	(2 \eps) \left[ 
				\frac{3}{2} \ProbWRT{\nu}{X \in A_\eps} + \frac{5}{3} \frac{z}{n}
			\right].
\end{align*}
\else
%%%%%%%%% SINGLE COLUMN %%%%%%%%%
\begin{align*}
	\hat{L}_n(\hat{\pi}_n) & \leq 
	\hat{L}_n(\pi^*_n) \leq
	\frac{3}{2} \tilde{L}_n(\pi^*_n) + (2 \eps) \ProbWRT{\nu_n}{X \in A_\eps} 
	\\ & \leq {}
	\frac{3}{2} \left[ L_n(\pi^*_n) + c_1 \norm{\QpiPrimehat - \QpiPrime}_{\infty,\Dn}^{1+\zeta} + c_2 \frac{z}{n} \right] + {}
	%\\ &
	%\quad \,
	(2 \eps) \left[ 
				\frac{3}{2} \ProbWRT{\nu}{X \in A_\eps} + \frac{5}{3} \frac{z}{n}
			\right].
\end{align*}
\fi
From~\eqref{eq:CAPI-DistortionLemma-Proof-Lhat2Ltilde2} and by applying~\eqref{eq:CAPI-DistortionLemma-Proof-L2Ltilde1}, we also have
\ifDoubleColumn
%%%%%%%%% DOUBLE COLUMN %%%%%%%%%
\begin{align*}
	\hat{L}_n(\hat{\pi}_n) & \geq \frac{1}{2} \tilde{L}_n(\hat{\pi}_n) - \eps \, \ProbWRT{\nu_n}{X \in A_\eps}\\
	& \geq
	\frac{1}{2}	\left[
				L_n(\pihat_n) - c_1 \norm{\QpiPrimehat - \QpiPrime}_{\infty,\Dn}^{1+\zeta} - c_2 \frac{z}{n} 			\right] - {}
	\\ &
	\quad \; \;
	\eps \left[  \frac{3}{2} \ProbWRT{\nu}{X \in A_\eps} + \frac{5z}{3n} \right].
\end{align*}
\else
%%%%%%%%% SINGLE COLUMN %%%%%%%%%
\begin{align*}
	\hat{L}_n(\hat{\pi}_n) & \geq \frac{1}{2} \tilde{L}_n(\hat{\pi}_n) - \eps \, \ProbWRT{\nu_n}{X \in A_\eps}\\
	& \geq
	\frac{1}{2}	\left[
				L_n(\pihat_n) - c_1 \norm{\QpiPrimehat - \QpiPrime}_{\infty,\Dn}^{1+\zeta} - c_2 \frac{z}{n} 			\right] - {}
	%\\ &
	%\quad \; \;
	\eps \left[  \frac{3}{2} \ProbWRT{\nu}{X \in A_\eps} + \frac{5z}{3n} \right].
\end{align*}
\fi % Single/Two Column
These two inequalities imply that
$L_n(\pihat_n) \leq 3 L_n(\pi^*_n) + c_1 \smallnorm{\QpiPrimehat - \QpiPrime}_{\infty,\Dn}^{1+\zeta} + c_2 \frac{z}{n}$ in the event that~\eqref{eq:CAPI-DistortionLemma-Proof-Pnun2Pnu} holds, which has probability at least $1 - e^{-z}$.
\end{proof}

%%%%%%%%%%%%%%%%%%%%%%%%%%%%%%%%%%%%%%%%%%%%%%%

%%%%%%%%%%%%%%%%%%%%%%%%%%%%%%%%%%%%%%%%%%%%%%%
%%%%%%%%%%%% THEOREM -  Loss at each iteration - Proof %%%%%%%%%%%%
%%%%%%%%%%%%%%%%%%%%%%%%%%%%%%%%%%%%%%%%%%%%%%%
\begin{proof}[Proof of Theorem~\ref{thm:CAPI-ErrorInEachIteration}] 
We use Theorem 3.3 by~\citet{BartlettBousquetMendelson05} (quoted as Theorem~\ref{thm:ConcentrationWithLocalRademacher} in Appendix~\ref{sec:CAPI-Appendix-ConcentrationInequalities}).
For function $l(\pi)(x) = \gapQPrime(x)  \One{\pi(x) \neq \argmax_{a \in \AA} \QpiPrime(x,a)}$, we have
\ifDoubleColumn
%%%%%%%%% DOUBLE COLUMN %%%%%%%%%
\begin{align*}
	\Var{l(\pi)(X)} & \leq \EE{\left| \gapQPrime(X) \One{\pi(X) \neq \argmax_{a \in \AA} \QpiPrime(X,a)} \right|^2 }
	\\
	& \leq 
2\Qmax \EE{l(\pi)(X)},
\end{align*}
\else
%%%%%%%%% SINGLE COLUMN %%%%%%%%%
\begin{align*}
	\Var{l(\pi)(X)} \leq \EE{\left| \gapQPrime(X) \One{\pi(X) \neq \argmax_{a \in \AA} \QpiPrime(X,a)} \right|^2 }
\leq 
2\Qmax \EE{l(\pi)(X)},
\end{align*}
\fi
so the variance condition of that theorem is satisfied. If we have a function $\Psi$ as defined  in~\eqref{eq:CAPI-ComplexityCondition}, the theorem states that there exist $c_1, c_2> 0$ such that for any $z > 0$ and any $\pi \in \Pi$ (including $\pihat_n \in \Pi$),
\begin{align}\label{eq:CAPI-ErrorInEachIteration-Proof-GeneralConcentrationBound}
	L(\pi) = \Pr l(\pi) \leq 2P_n l(\pi) + c_1 r^* + c_2 \frac{z}{n},
\end{align}
with probability at least $1 - e^{-z}$ ($c_1$ can be chosen as \revised{$704/\Qmax$} and $c_2$ can be chosen as $126 \, \Qmax$). 
% In particular, this inequality holds for $\pihat_n \in \Pi$.

Let $\pioptInPi  \leftarrow \argmin_{\pi \in \Pi} L(\pi)$ be the minimizer of the expected loss in policy space $\Pi$.
Consider~\eqref{eq:CAPI-ErrorInEachIteration-Proof-GeneralConcentrationBound} with the choice of $\pi = \pihat_n$, and add and subtract $6P_n l(\pioptInPi)$ and $6 \Pr l(\pioptInPi)$ and then use Lemma~\ref{lem:CAPI-DistortionLemma}. With probability at least $1 - 2 e^{-z}$, we get
\ifDoubleColumn
%%%%%%%%% DOUBLE COLUMN %%%%%%%%%
\begin{align}
\nonumber
	L(\pihat_n) \leq {} &
				2 \Pr_n l(\pihat_n) - 6 \left[ \Pr_n l(\pioptInPi) - \Pr_n l(\pioptInPi) \right]
				\\ \nonumber &
				{} - 6 \left[ \Pr l(\pioptInPi) - \Pr l(\pioptInPi) \right] + c_1 r^* + c_2 \frac{z}{n}
	\\ \nonumber
			\leq {} &
			6 \left[ \Pr_n l(\pi^*_n) -  \Pr_n l(\pioptInPi) \right] + 6 \left[ \Pr_n l(\pioptInPi) - \Pr l(\pioptInPi) \right]
	\\ \nonumber &
			+ 6 \Pr l(\pioptInPi)  + c_1 r^* + c_2 \norm{\QpiPrimehat - \QpiPrime}_{\infty,\Dn}^{1+\zeta} +  c_3 \frac{z}{n} 
			\\ 
			\nonumber
			\leq {} &
			6 \left[ \Pr_n l(\pioptInPi) - \Pr l(\pioptInPi) \right] + 6 \Pr l(\pioptInPi) + c_1 r^* 
			 \\ &
			{} + c_2 \norm{\QpiPrimehat - \QpiPrime}_{\infty,\Dn}^{1+\zeta} + c_3 \frac{z}{n},\label{eq:CAPI-ErrorInEachIteration-Proof-FirstBoundOnLpihatn}
\end{align}
\else
%%%%%%%%% SINGLE COLUMN %%%%%%%%%
\begin{align}
\nonumber
	& L(\pihat_n) \leq
				2 \Pr_n l(\pihat_n) - 6 \left[ \Pr_n l(\pioptInPi) - \Pr_n l(\pioptInPi) \right]
				%\\ \nonumber &
				{} - 6 \left[ \Pr l(\pioptInPi) - \Pr l(\pioptInPi) \right] + c_1 r^* + c_2 \frac{z}{n}
				\leq
	\\ \nonumber
			&
			6 \left[ \Pr_n l(\pi^*_n) -  \Pr_n l(\pioptInPi) \right] + 6 \left[ \Pr_n l(\pioptInPi) - \Pr l(\pioptInPi) \right]
	%\\ \nonumber &
			+ 6 \Pr l(\pioptInPi)  + c_1 r^* + c_2 \norm{\QpiPrimehat - \QpiPrime}_{\infty,\Dn}^{1+\zeta} +  c_3 \frac{z}{n} \leq
			\\  & %\leq {} &
			6 \left[ \Pr_n l(\pioptInPi) - \Pr l(\pioptInPi) \right] + 6 \Pr l(\pioptInPi) + c_1 r^* 
			% \\ &
			{} + c_2 \norm{\QpiPrimehat - \QpiPrime}_{\infty,\Dn}^{1+\zeta} + c_3 \frac{z}{n},\label{eq:CAPI-ErrorInEachIteration-Proof-FirstBoundOnLpihatn}
\end{align}
\fi % Single/Two Column
where in the last inequality we used the minimizing property of $\pi^*_n$, i.e., $\Pr_n l(\pi^*_n) - \Pr_n l(\pioptInPi) \leq 0$.
\revised{Here $c_2$ can be chosen as $36 \times 2^{2\zeta}$.}

To upper bound $\Pr_n l(\pioptInPi) - \Pr l(\pioptInPi)$, we apply Bernstein inequality (Lemma~\ref{lem:CAPI-BernsteinInequality} in Appendix~\ref{sec:CAPI-Appendix-ConcentrationInequalities}) to get that for any $ z > 0$,
$\Pr_n l(\pioptInPi) - \Pr l(\pioptInPi) \leq \sqrt{ \frac{2 \Var{l(\pioptInPi)} z}{ n} } + \frac{4 \Qmax z}{3 n}$, with probability at least $1 - e^{-z}$.
Since $\Var{l(\pioptInPi)} \leq 2 \Qmax \Pr l(\pioptInPi)$ (as shown above), by the application of arithmetic mean--geometric mean inequality we obtain 
$\Pr_n l(\pioptInPi) - \Pr l(\pioptInPi) \leq \Pr l(\pioptInPi) + \frac{7 \Qmax z}{3 n}$ with the same probability.
This and~\eqref{eq:CAPI-ErrorInEachIteration-Proof-FirstBoundOnLpihatn} result in
\begin{align*}
	L(\pihat_n) \leq 12 \Pr l(\pioptInPi) + c_1 r^* + c_2 \norm{\QpiPrimehat - \QpiPrime}_{\infty,\Dn}^{1+\zeta} + c_3 \frac{z}{n},
\end{align*}
with probability at least $1 - 3 e^{-z}$ as desired.
\end{proof}

%%%%%%%%%%%%%%%%%%%%%%%%%%%%%%%%%%%%%%%%%%%%%%%
%%%%%%%%%%%%%%%%%%%%%%%%%%%%%%%%%%%%%%%%%%%%%%%
%%%%%%%%%%%%%%%%%%%%%%%%%%%%%%%%%%%%%%%%%%%%%%%

%%%%%%%%%%%%%%%%%%%%%%%%%%%%%%%%%%%%%%%%%%%%%%%
%%%%%%%%%% THEOREM -  Performance Loss for CAPI - Proof %%%%%%%%%%%
%%%%%%%%%%%%%%%%%%%%%%%%%%%%%%%%%%%%%%%%%%%%%%%
\begin{proof}[Proof of Theorem~\ref{thm:CAPI-PerformanceLoss}]
It is shown by~\citep{LazaricGhavamzadehMunosDPI2010} that
\ifDoubleColumn
%%%%%%%%% DOUBLE COLUMN %%%%%%%%%
\begin{align*}
	& \Vopt - V^{\pi_{K}} \leq 
	\\ &
	\sum_{k=0}^{K-1} \gamma^{K-k-1} (\PKernel^\piopt)^{K-k-1}
		\sum_{m \geq 0} \gamma^m (\PKernel^{\pi_{k+1}})^m l^{\pi_{k}}(\pi_{k+1})
		\\ &
	+ (\gamma \PKernel^\piopt)^K (\Vopt - V^{\pi_{0}}).
\end{align*}
\else
%%%%%%%%% SINGLE COLUMN %%%%%%%%%
\begin{align*}
	& \Vopt - V^{\pi_{K}} \leq 
	%\\ &
	\sum_{k=0}^{K-1} \gamma^{K-k-1} (\PKernel^\piopt)^{K-k-1}
		\sum_{m \geq 0} \gamma^m (\PKernel^{\pi_{k+1}})^m l^{\pi_{k}}(\pi_{k+1})
	%	\\ &
	+ (\gamma \PKernel^\piopt)^K (\Vopt - V^{\pi_{0}}).
\end{align*}
\fi % Single/Two Column
We apply $\rho$ to both sides and use the definition of $c_{\rho,\nu}(m_1;m_2;\pi)$ to get  
\ifDoubleColumn
%%%%%%%%% DOUBLE COLUMN %%%%%%%%%
\begin{align*}
	& \rho (\Vopt - V^{\pi_{K}}) \leq \\
& 
\sum_{k=0}^{K-1} \gamma^{K-k-1} \sum_{m \geq 0} \gamma^m c_{\rho,\nu}(K-k-1,m;\pi_{k+1}) \, \nu l^{\pi_{k}}(\pi_{k+1}) \\
	&
	  {} + \gamma^K (2 \Qmax).
\end{align*}
\else
%%%%%%%%% SINGLE COLUMN %%%%%%%%%
\begin{align*}
	& \rho (\Vopt - V^{\pi_{K}}) \leq %\\
%& 
\sum_{k=0}^{K-1} \gamma^{K-k-1} \sum_{m \geq 0} \gamma^m c_{\rho,\nu}(K-k-1,m;\pi_{k+1}) \, \nu l^{\pi_{k}}(\pi_{k+1})% \\
	%&
	  {} + \gamma^K (2 \Qmax).
\end{align*}
\fi % Single/Two Column
Recall that $\nu l^{\pi_{k}}(\pi_{k+1}) = L^{\pi_{k}}(\pi_{k+1})$. We decompose $\gamma$ to $\gamma^s \gamma^{(1-s)}$ (for $0 \leq s \leq 1$) and separate terms involving the concentrability coefficients and those related to $L^{\pi_{k}}(\pi_{k+1})$.
We then have for any $0 \leq s \leq 1$,
%
%%%%%%%%% SINGLE COLUMN %%%%%%%%%
\ifDoubleColumn
%%%%%%%%% DOUBLE COLUMN %%%%%%%%%
\begin{align*}
& \rho (\Vopt - V^{\pi_{K}}) \leq 
\\ &
\max_{0 \leq k \leq K-1} \left [ \gamma^{s(K-k-1)} L^{\pi_{k}}(\pi_{k+1}) \right] \times
\\ & 
\, \sum_{k'=0}^{K-1} \gamma^{(1-s)k'} \sum_{m \geq 0} \gamma^m \sup_{\pi' \in \Pi} c_{\rho,\nu} (k',m;\pi') + \gamma^K (2 \Qmax).
\end{align*}
\else
%%%%%%%%% SINGLE COLUMN %%%%%%%%%
\begin{align*}
& \rho (\Vopt - V^{\pi_{K}}) \leq 
%\\ &
\max_{0 \leq k \leq K-1} \left [ \gamma^{s(K-k-1)} L^{\pi_{k}}(\pi_{k+1}) \right] 
%\\ & 
\, \sum_{k'=0}^{K-1} \gamma^{(1-s)k'} \sum_{m \geq 0} \gamma^m \sup_{\pi' \in \Pi} c_{\rho,\nu} (k',m;\pi') + \gamma^K (2 \Qmax).
\end{align*}
\fi % Single/Two Column

Taking the infimum w.r.t. $s$ and using the definition of $C_{\rho,\nu}(K)$, we get that for $\mathrm{Loss}(\pi_{K};\rho) = \rho (\Vopt - V^{\pi_{K}})$,
\ifDoubleColumn
%%%%%%%%% DOUBLE COLUMN %%%%%%%%%
\begin{align}\label{thm:CAPI-PerformanceLoss-Proof-ErrorPropagation}
\nonumber
	& \mathrm{Loss}(\pi_{K};\rho) \leq 	
	\\ & \nonumber
	\frac{2}{1-\gamma} 
		\Big[
		\inf_{s \in [0,1]} C_{\rho,\nu}(K,s)
		\max_{0 \leq k \leq K-1}	[\gamma^{s(K-k-1)} L^{\pi_{k}}(\pi_{k+1}) ] 
		\\ & \qquad \quad
			 {} + \gamma^K \Rmax
		\Big].
\end{align}
\else
%%%%%%%%% SINGLE COLUMN %%%%%%%%%
\begin{align}\label{thm:CAPI-PerformanceLoss-Proof-ErrorPropagation}
%\nonumber
	& \mathrm{Loss}(\pi_{K};\rho) \leq 	
	%\\ & \nonumber
	\frac{2}{1-\gamma} 
		\Bigg[
		\inf_{s \in [0,1]} C_{\rho,\nu}(K,s)
		\max_{0 \leq k \leq K-1}	[\gamma^{s(K-k-1)} L^{\pi_{k}}(\pi_{k+1}) ] 
		%\\ & \qquad \quad
			 {} + \gamma^K \Rmax
		\Bigg].
\end{align}
\fi % Single/Two Column

Fix $ 0 < \delta < 1$.
For each iteration $k = 0,\dotsc, K-1$, by invoking Theorem~\ref{thm:CAPI-ErrorInEachIteration} with the confidence parameter $\delta / K$, we get
$L^{\pi_{k}}(\pi_{k+1})
	\leq
	12 \inf_{\pi \in \Pi} L^{\pi_{k}} (\pi) + c_1 r^* + c_2 \smallnorm{\hat{Q}^{\pi_{k}} - Q^{\pi_{k}} }_{\infty,\Dn^{k}}^{1 + \zeta}
	 +  c_3 \frac{\ln(K/\delta)}{n}$, which holds with probability at least $1 - \delta/K$.
Since $\inf_{\pi \in \Pi} L^{\pi_{k}} (\pi) \leq d(\Pi)$, the previous set of inequalities alongside~\eqref{thm:CAPI-PerformanceLoss-Proof-ErrorPropagation} imply the desired result.
\end{proof}

\revised{
We would like to remark that to extend the analysis to $\actionnum > 2$, Lemma~\ref{lem:CAPI-DistortionLemma} is the main result that should be modified. The proofs of Theorems~\ref{thm:CAPI-ErrorInEachIteration} and~\ref{thm:CAPI-PerformanceLoss} remain intact.
}

\section{Upper bound on $r^*$}
\label{sec:CAPI-Appendix-rstar}
%%%%%%%%%%%%%%%%%%%%%%%%%%%%%%%%%%%%%%%%%%%%%%%
%%%%%%%%%%%%%%%%%%%%%%%%%%%%%%%%%%%%%%%%%%%%%%%
%%%%%%%%%%%%%%%%%%%%%%%%%%%%%%%%%%%%%%%%%%%%%%%
The following proposition provides a distribution-free upper bound for $r^*$ \eqref{eq:CAPI-ComplexityCondition} for policy space $\Pi$ with a finite VC-dimension.
We closely follow the proof of Corollary 3.7 of~\citet{BartlettBousquetMendelson05}.
\begin{proposition}\label{prop:CAPI-VC2Rad}
Suppose that policy space $\Pi$ has a finite VC-dimension $d$. There exists constant $c > 0$, which is independent of $n$ and $d$, such that for $n \geq d$, complexity condition~\eqref{eq:CAPI-ComplexityCondition} is satisfied with $r^* \leq \frac{c d \log(n/d)}{n}$.
\end{proposition}
%%%%%%%%%%%%%%%%%%%%%%%
\begin{proof}
The proof has three main steps: 1) the Rademacher complexity of  $\cset{l(\pi)}{\pi \in \Pi, \Pr l^2(\pi) \leq r}$ (the RHS of~\eqref{eq:CAPI-ComplexityCondition}) is upper bounded by the integral of the metric entropy (i.e., logarithm of the covering number),  2) the metric entropy of that set is related to the metric entropy of $\Pi$, and  3) the metric entropy of $\Pi$ is related to its VC-dimension.

\noindent 1) Define the sub-root function
\ifDoubleColumn $ \else \[ \fi
	\Psi(r) = 20 \Qmax \EE{ R_n \cset{l(\pi)}{ \pi \in \Pi, \Pr l^2(\pi) \leq r}  }
+ 44 \Qmax^2 \frac{\log n}{n}.
\ifDoubleColumn $ \else \] \fi
If $r \geq \Psi(r)$, Corollary 2.2 of~\citet{BartlettBousquetMendelson05} indicates that with probability at least $1 - 1/n$, we have
$\cset{l(\pi)}{\pi \in \Pi, \Pr l^2(\pi) \leq r} \subseteq
\cset{l(\pi)}{\pi \in \Pi, \Pr_n l^2(\pi) \leq 2r}$. 
Thus, we can upper bound the Rademacher complexity of $\cset{l(\pi)}{\pi \in \Pi, \Pr l^2(\pi) \leq r}$ as 
\ifDoubleColumn $ \else \[ \fi
\EE{ R_n \cset{l(\pi)}{\pi \in \Pi, \Pr l^2(\pi) \leq r} }
\leq
\EE{R_n \cset{l(\pi)}{\pi \in \Pi, \Pr_n l^2(\pi) \leq 2r }} + \frac{2\Qmax}{n}.
\ifDoubleColumn $ \else \] \fi
The fixed-point equation $r^* = \Psi(r^*)$ satisfies
\ifDoubleColumn
%%%%%%%%% DOUBLE COLUMN %%%%%%%%%
\begin{align}\label{eq:CAPI-VC2Rad-FixedPointCondition}
\nonumber
r^* \leq & 20 \Qmax \EE{R_n \cset{l(\pi)}{\pi \in \Pi, \Pr_n l^2(\pi) \leq 2r^* }} +
	\\ &
	 \frac{2 \Qmax + 44 \Qmax^2 \log n}{n}.
\end{align}
\else
\begin{align}\label{eq:CAPI-VC2Rad-FixedPointCondition}
%\nonumber
r^* \leq & 20 \Qmax \EE{R_n \cset{l(\pi)}{\pi \in \Pi, \Pr_n l^2(\pi) \leq 2r^* }} +
%	\\ &
	 \frac{2 \Qmax + 44 \Qmax^2 \log n}{n}.
\end{align}
\fi
Using the metric entropy integral upper bound for the Rademacher complexity (Theorem A.7 of \citet{BartlettBousquetMendelson05}; originally from~\citet{Dudley1999}), we get that there exists a constant $C > 0$ such that
\ifDoubleColumn
%%%%%%%%% DOUBLE COLUMN %%%%%%%%%
\begin{align*}
& \EE{R_n \cset{l(\pi)}{\pi \in \Pi, \Pr_n l^2(\pi) \leq 2r^* }}
\leq
\\
& 
\frac{C}{\sqrt{n} }
\EE{	\int_{0}^{\sqrt{2 r^*} } 
	\sqrt{ \revised{\log} \cN(\eps, \cset{l(\pi)}{\pi \in \Pi }, L_2(\Pr_n) ) } \; \mathrm{d} \eps
	}.
\end{align*}
\else
%%%%%%%%% SINGLE COLUMN %%%%%%%%%
\begin{align*}
& \EE{R_n \cset{l(\pi)}{\pi \in \Pi, \Pr_n l^2(\pi) \leq 2r^* }}
\leq
%\\
%& 
\frac{C}{\sqrt{n} }
\EE{	\int_{0}^{\sqrt{2 r^*} } 
	\sqrt{ \revised{\log} \cN(\eps, \cset{l(\pi)}{\pi \in \Pi }, L_2(\Pr_n) ) } \; \mathrm{d} \eps
	}.
\end{align*}
\fi
Here $\cN(\eps, \cset{l(\pi)}{\pi \in \Pi }, L_2(\Pr_n) )$ is an $\eps$-covering number of the set $\cset{l(\pi)}{\pi \in \Pi }$ w.r.t. the $L_2(\Pr_n)$-norm (cf. Chapter 9 of~\citet{Gyorfi02}).

\noindent 2) Since for any $\pi_1$ and $\pi_2$, we have $| l(x;\pi_1) - l(x;\pi_2) |^2 \leq (2 \Qmax)^2 | \One{\pi_1(x) \neq \pi_2(x)}|^2$, a $u$-covering for $\Pi$ w.r.t. $L_2(\Pr_n)$ induces a $2\Qmax u$-covering for $\cset{l(\pi)}{\pi \in \Pi }$. Therefore,
$\cN(\eps, \cset{l(\pi)}{\pi \in \Pi }, L_2(\Pr_n) ) \leq 
\cN(\frac{\eps}{2 \Qmax}, \Pi, L_2(\Pr_n) )$.

\noindent 3) For a class $\Pi$ with finite VC-dimension $d$, its metric entropy can be upper bounded by
$\log \cN(\eps,\Pi,L_2(\Pr_n) ) \leq c d \log(1/\eps)$ for some constant $c > 0$.
Thus, 
\ifDoubleColumn
%%%%%%%%% DOUBLE COLUMN %%%%%%%%%
\begin{align*}
& \EE{R_n \cset{l(\pi)}{\pi \in \Pi, \Pr_n l^2(\pi) \leq 2r^* }} \leq 
\\ &
\sqrt{\frac{c d r^* \log(1/r^*)}{n} }
\leq
\sqrt{c \left( \frac{d^2}{n^2} + \frac{d r^* \log(n/ed)}{n} \right) }.
\end{align*}
\else
%%%%%%%%% SINGLE COLUMN %%%%%%%%%
\begin{align*}
& \EE{R_n \cset{l(\pi)}{\pi \in \Pi, \Pr_n l^2(\pi) \leq 2r^* }} \leq 
%\\ &
\sqrt{\frac{c d r^* \log(1/r^*)}{n} }
\leq
\sqrt{c \left( \frac{d^2}{n^2} + \frac{d r^* \log(n/ed)}{n} \right) }.
\end{align*}
\fi
Using this upper bound and solving for $r^*$ in~\eqref{eq:CAPI-VC2Rad-FixedPointCondition} one get that for $n \geq d$, we have $r^* \leq \frac{c d \log(n/d)}{n}$.
\end{proof}
%Note that even though this leads to the fast $r^* = O(\frac{d \log(n)}{n})$ rate for the estimation error, it might still be a conservative upper bound for $r^*$ and the behavior of the Rademacher complexity and as a result $r^*$ might be much more favourable.
\fi

\else
	\appendices
	
	%!TEX root = CAPI-TAC.tex

%%%%%%%%%%%%%%%%%%%%%%%%%%%%%%%%%%%%%%%%%%%%%%%
%%%%%%%%%%%%%%%%%%%%%%%%%%%%%%%%%%%%%%%%%%%%%%%
%%%%%%%%%%%%%%%%%%%%%%%%%%%%%%%%%%%%%%%%%%%%%%%
\section{Additional Experiments}
\label{sec:CAPI-Appendix-Experiments(Additional)}

We present some additional experiments to show that CAPI is a quite flexible framework and that the algorithms derived from it (by choosing PolicyEval and policy space $\Pi$) can be quite competitive.

The benchmark domains that we consider are Mountain-Car (2-dimensional state space) and Pole Balancing (4-dimensional state space) -- both standard benchmarks in the reinforcement learning community.
As a showcase of CAPI's flexibility in the choice of the policy evaluation method and policy space $\Pi$, we use different algorithms for each domain.\footnote{The source codes of our experiments will be available online.}

%%%%%%%%%%%%%%%%%%%%%%%%%%%%%%%%%%%%%%%%%%%%%%%
%%%%%%%%%%%%%%%%%%%%%%%%%%%%%%%%%%%%%%%%%%%%%%%
%%%%%%%%%%%%%%%%%%%%%%%%%%%%%%%%%%%%%%%%%%%%%%%
\subsection{Mountain-Car}
\label{sec:CAPI-Experiments-MountainCar}
%%%%%%%%%%%%%%%%%%%%%%%%%%%%%%%%%%%%%%%%%%%%%%%
%%%%%%%%%%%%%%%%%%%%%%%%%%%%%%%%%%%%%%%%%%%%%%%
%%%%%%%%%%%%%%%%%%%%%%%%%%%%%%%%%%%%%%%%%%%%%%%
\begin{figure*}[tb]
\centering
%  \subfigure[All]{ 
%   \includegraphics[scale=\scl]{Figures/capi_pole_episodes}
%    }
  \subfigure[Steps to goal]{ 
   \includegraphics[width=0.45 \linewidth]{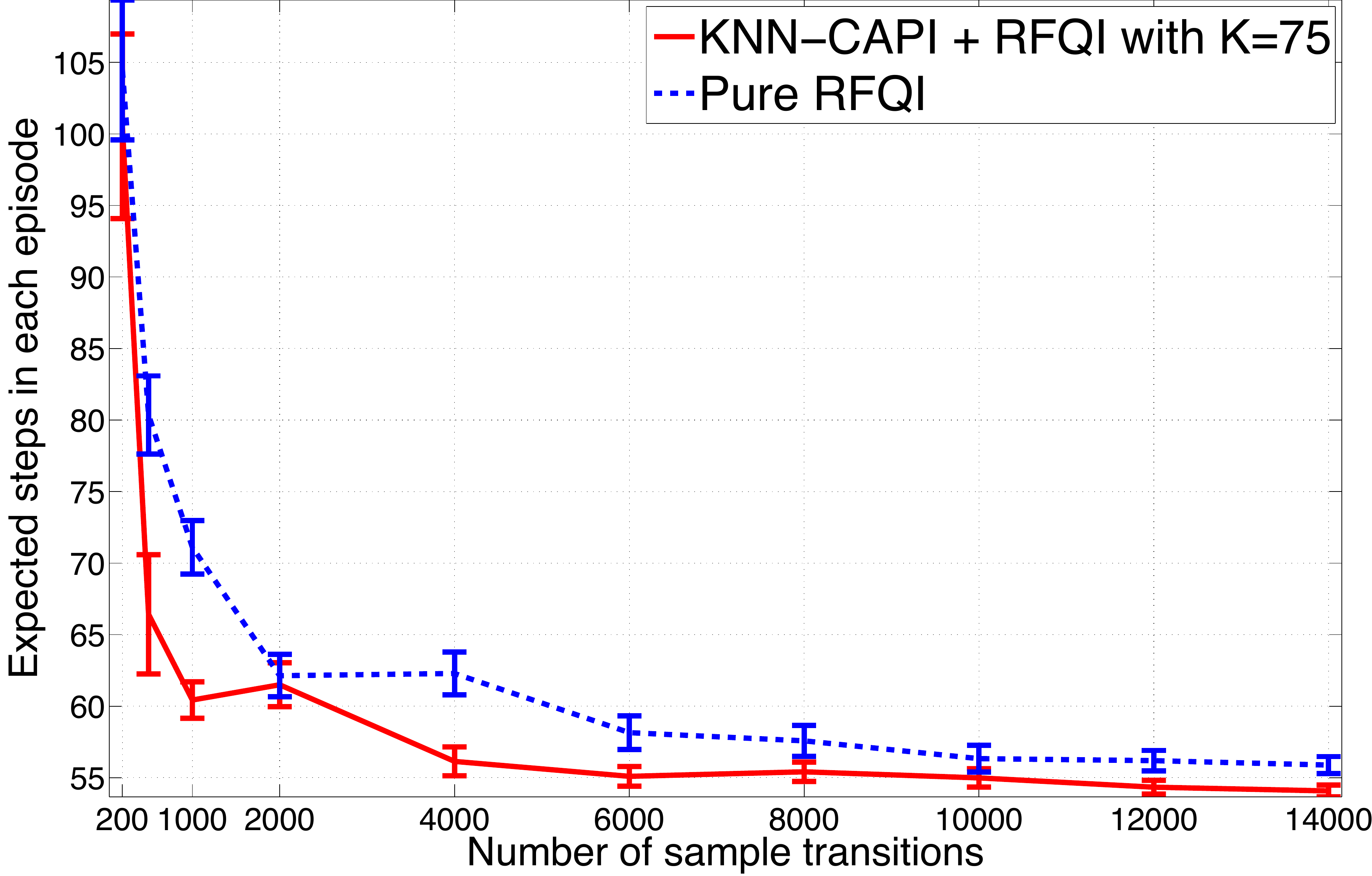}
    }
  \subfigure[Return]{ 
   \includegraphics[width=0.45 \linewidth]{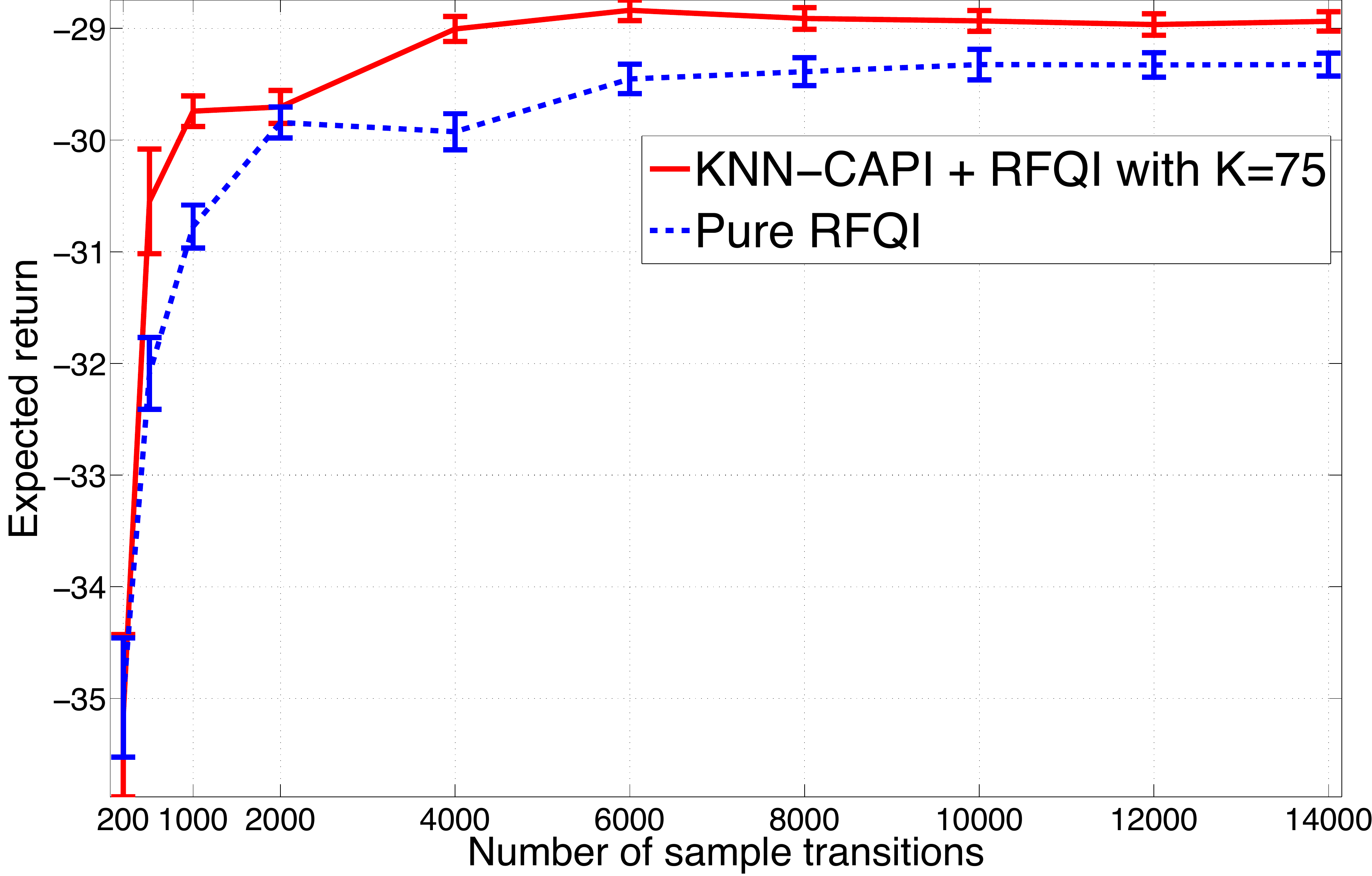}
    }
 \caption{(Mountain-Car) Comparing the expected (a) number of steps to goal and (b) return in each episode for KNN-CAPI vs. Regularized Fitted Q-Iteration as a function of sample size. The error bars show one standard error over 50 runs.}
 \label{fig:CAPI-MountainCar-KNNCAPIvsRFQI-Main}
\end{figure*}
%%%%%%%%%%%%%%%%%%%%%%%%%%%%%%%%%%%%%%%%%%%%%%%
%%%%%%%%%%%%%%%%%%%%%%%%%%%%%%%%%%%%%%%%%%%%%%%
%%%%%%%%%%%%%%%%%%%%%%%%%%%%%%%%%%%%%%%%%%%%%%%

We compare an instantiation of CAPI with a pure value-based approach on the Mountain-Car task~\citep{Sutton98}, which has a 2-dimensional state space.
The choice of value-based approach is the nonparametric kernelized Regularized Fitted Q-Iteration (RFQI) algorithm~\citep{FarahmandACC09}.
For CAPI, we have to choose the policy evaluation method PolicyEval and policy space $\Pi$ (cf. Algorithm~\ref{alg:CAPI}).
For a given policy $\pi_{k}$, we perform one iteration of RFQI (used for policy evaluation only, that is, there is no policy improvement) to estimate $\hat{Q}^{\pi_{k}}$. 
This is similar to how CAPI was implemented in the 1D Chain Walk example in Section~\ref{sec:CAPI-Experiments-1DChainWalk}, except that here we deal with a continuous state space, so we use an approximate value iteration algorithm RFQI instead of the exact value iteration.

To derive $\pi_{k+1}$ from $\hat{Q}^{\pi_{k}}$, we use action-gap-weighted KNN-CAPI formulation as follows.
At any state $x$, KNN-CAPI chooses an action that minimizes the CAPI's empirical loss~\eqref{eq:CAPI-emp-loss} in the $\kappa$-neighbourhood of $x$.
More concretely, suppose that the state space $\XX$ is endowed with a norm $\norm{\cdot}$, e.g., the $l_2$-norm for $\XX \subset \Real^d$.
For some positive integer $\kappa$, let $\mathrm{N}_\kappa(x)\subseteq \Dn^{(k)}$ be the set of $\kappa$-closest points to $x$ in $\Dn^{(k)}$  (w.r.t. the norm of $\XX$), breaking ties deterministically.
The KNN-based CAPI policy $\pi_{k+1}(x)$ is then defined as:
\begin{align*}
	 & \pi_{k+1}(x) \leftarrow
	 \argmin_{a \in \AA}
	 \sum_{X_i \in \mathrm{N}_\kappa(x)}	 
	 \gap_{\hat{Q}^{\pi_k }}(X_i) \One{a \neq \pihat(X_i;\hat{Q}^{\pi_{k} })  } 
	 \\ & 
	 \equiv
	 \argmin_{a \in \AA} \sum_{X_i \in \mathrm{N}_\kappa(x)}
	\hat{Q}^{\pi_{k}}(X_i,\pihat(X_i;\hat{Q}^{\pi_{k}} )  ) - \hat{Q}^{\pi_{k}}(X_i,a)
	\\ &
	\equiv
	\argmax_{a \in \AA} \sum_{X_i \in \mathrm{N}_\kappa(x)} 
	\hat{Q}^{\pi_{k}}(X_i,a),
\end{align*}
Similar to Tree-CAPI, this is a very simple rule: pick the action that maximizes the action-value at the data points in the $\kappa$-NN of the query point $x$.
% Note that this is different from choosing the majority over the greedy actions in the neighbourhood, which would be the resulting policy if we neglected the action-gap function.
%
%
One could also assign different weights to each point in $\mathrm{N}_\kappa(x)$ as a function of distance to $x$, or more generally, use any local averaging estimator \cite{DevroyeGyorfiLugosi96,Gyorfi02} without much change in the formulation.
The derivation of KNN-CAPI  for $\actionnum > 2$ leads to the same rule.

We implemented the dynamics and the reward function of Mountain-Car task the same as Example 8.2 of~\citet{Sutton98}, and we set the discount factor to $\gamma = 0.98$.
The initial state was chosen uniformly random and the data was collected by a uniformly random policy. 
For data collection, we used trajectories with the length of at most 100 steps (or if the episode is terminated by reaching the goal).
To evaluate the policy, we let the agent go for at most 200 steps. If it did not reach the goal by then, 200 was reported as the length of the episode.

In our implementation of RFQI~\citep{FarahmandACC09}, we used a Gaussian kernel $\kfun(x_1,x_2) = \exp \left(-\frac{\norm{x_1 - x_2}^2}{2 \sigma^2} \right)$ with $\sigma^2 = 10^{-2}$.
The regularization coefficient was chosen as $\lambda = \frac{0.01}{n}$, in which $n$ is the number of samples.
We also used the sparsification method of~\citet{EngelMannorMeir2004} to reduce the computational complexity.
The parameters were selected after some trial-and-error, but were not systematically optimized. We use the same parameters of RFQI when it is used as the PolicyEval of KNN-CAPI.
In all experiments, the number of runs was set to 50. The shown error bars are one standard error around the sample average.

Figure~\ref{fig:CAPI-MountainCar-KNNCAPIvsRFQI-Main} shows the expected number of steps to reach the goal as well as the expected return per episode for both RFQI and KNN-CAPI with $K = 75$. % (averaged over $50$ runs).
The value of $K$ is selected based on another experiment that we will shortly present.
%
% In these experiments, the state's initial distribution is chosen uniformly random.
%
Even though both approaches are quite sample efficient as they learn a reasonable policy (i.e., a one that takes less than 60 steps to reach the goal) in a matter of a few thousand of samples or even less, KNN-CAPI outperforms RFQI, especially in the small-sample regime.
This is because CAPI benefits from the regularities of the policy space while RFQI, or any other purely value-based approach, is oblivious to it.

\subsubsection{Effect of $K$ in KNN-CAPI}
The value of $K$ in KNN-CAPI implicitly determines the underlying policy space $\Pi$. Thus it is interesting to see the effect of $K$ on the performance of the algorithm.
Figure~\ref{fig:MC-KNNCAPI-VaryingK-Steps} depicts the expected number of steps in each episode as well as the expected return as a function of $K$.
It shows the effect of $K$ at various sample size regimes. We see that when the number of samples is small, the choice of $K$ makes a big difference, but even in large-sample regime it has a noticeable effect.
The existence of an optimum shows that in order to benefit from the regularity of policy, the policy space should be chosen properly and problem-dependently.
This is a model selection problem, which is beyond the scope of this paper (cf. \citet{FarahmandSzepesvariMLJ11} for the discussion of the model selection for sequential decision-making problems).

% On the other hand, when there are many samples available, the quality of $\hat{Q}$ is already good enough that the added benefit of using policy regularity would not be very significant.

\begin{figure*}[tb]
\centering
%  \subfigure[All]{ 
%   \includegraphics[scale=\scl]{Figures/capi_pole_episodes}
%    }
  \subfigure[Steps to goal]{ 
   \includegraphics[width=0.45 \linewidth]{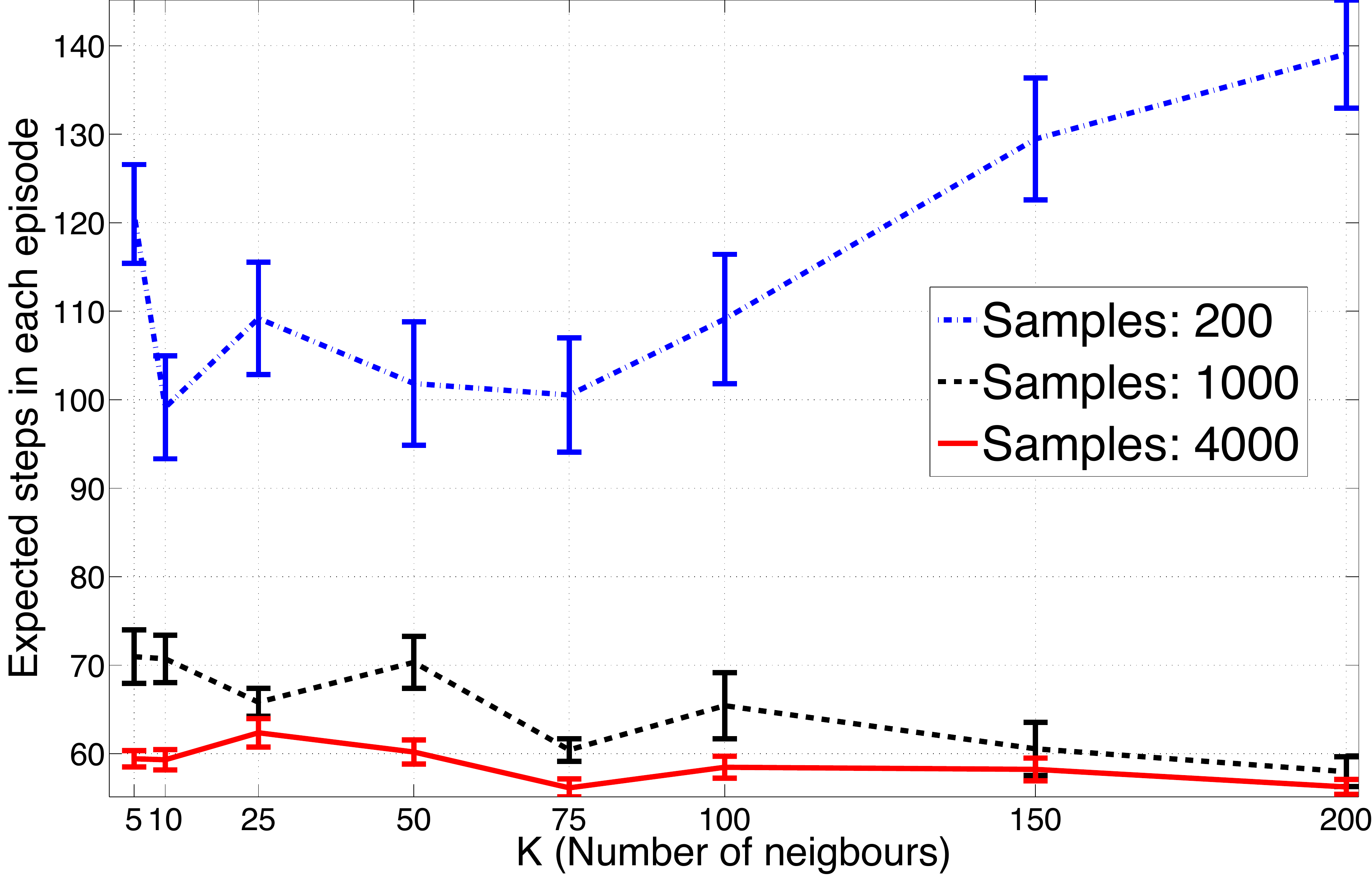}
    }
  \subfigure[Return]{ 
   \includegraphics[width=0.45 \linewidth]{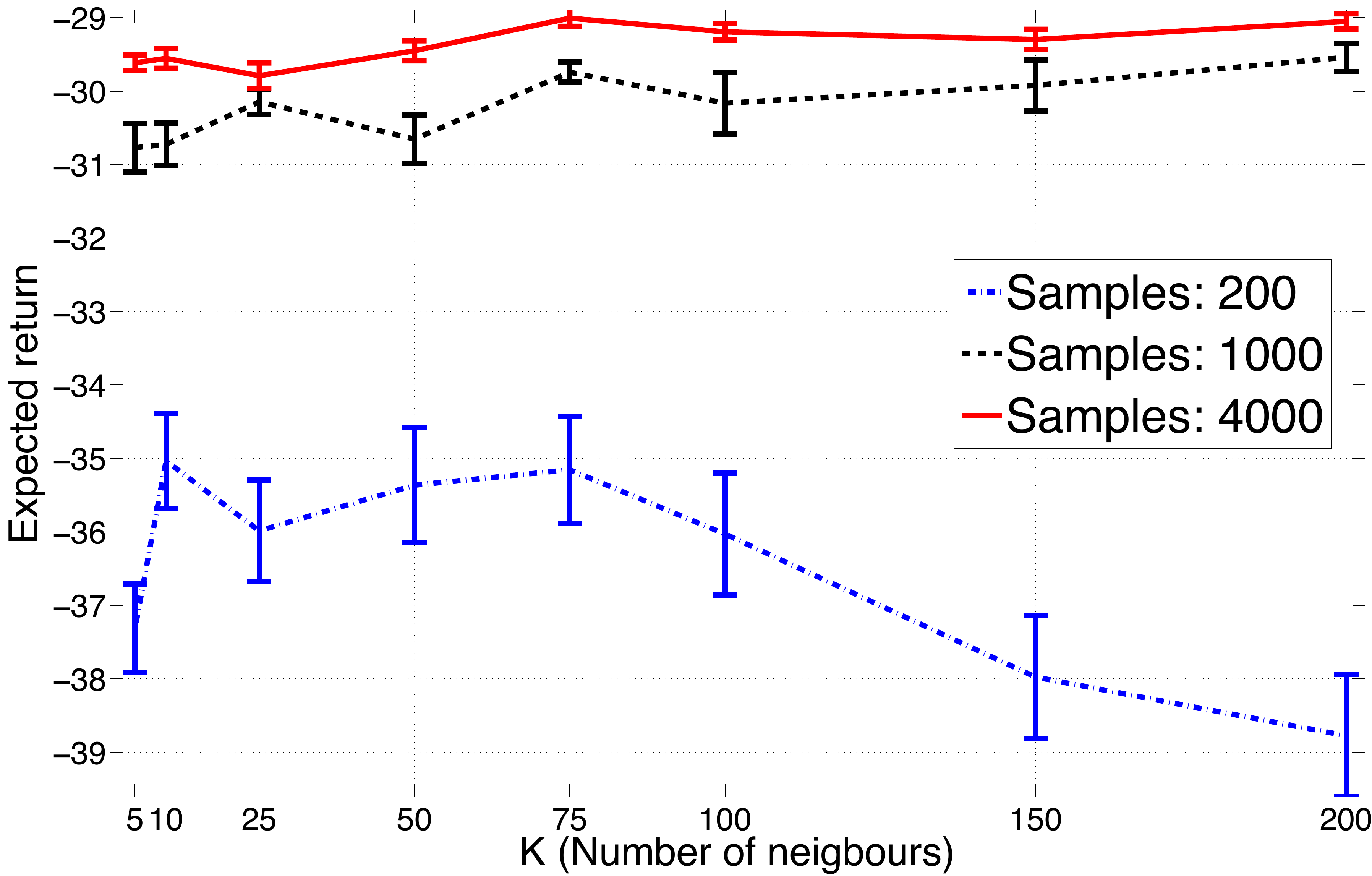}
    }
 \caption{(Mountain-Car) The expected (a) number of steps to goal and (b) return in each episode as a function of $K$ in KNN-CAPI. The error bars show one standard error over 50 runs.}
 \label{fig:MC-KNNCAPI-VaryingK-Steps}
\end{figure*}

%%%%%%%%%%%%%%%%%%%%%%%%%%%%%%%%%%%%%%%%%%%%%%%
%%%%%%%%%%%%%%%%%%%%%%%%%%%%%%%%%%%%%%%%%%%%%%%
%%%%%%%%%%%%%%%%%%%%%%%%%%%%%%%%%%%%%%%%%%%%%%%
\subsection{Pole Balancing}

\begin{figure*}[tb] %  figure placement: here, top, bottom, or page
\centering
\subfigure[Number of steps]{ %[$\mnea = 20$]{ 
\centering
   \includegraphics[width= 0.3\linewidth]{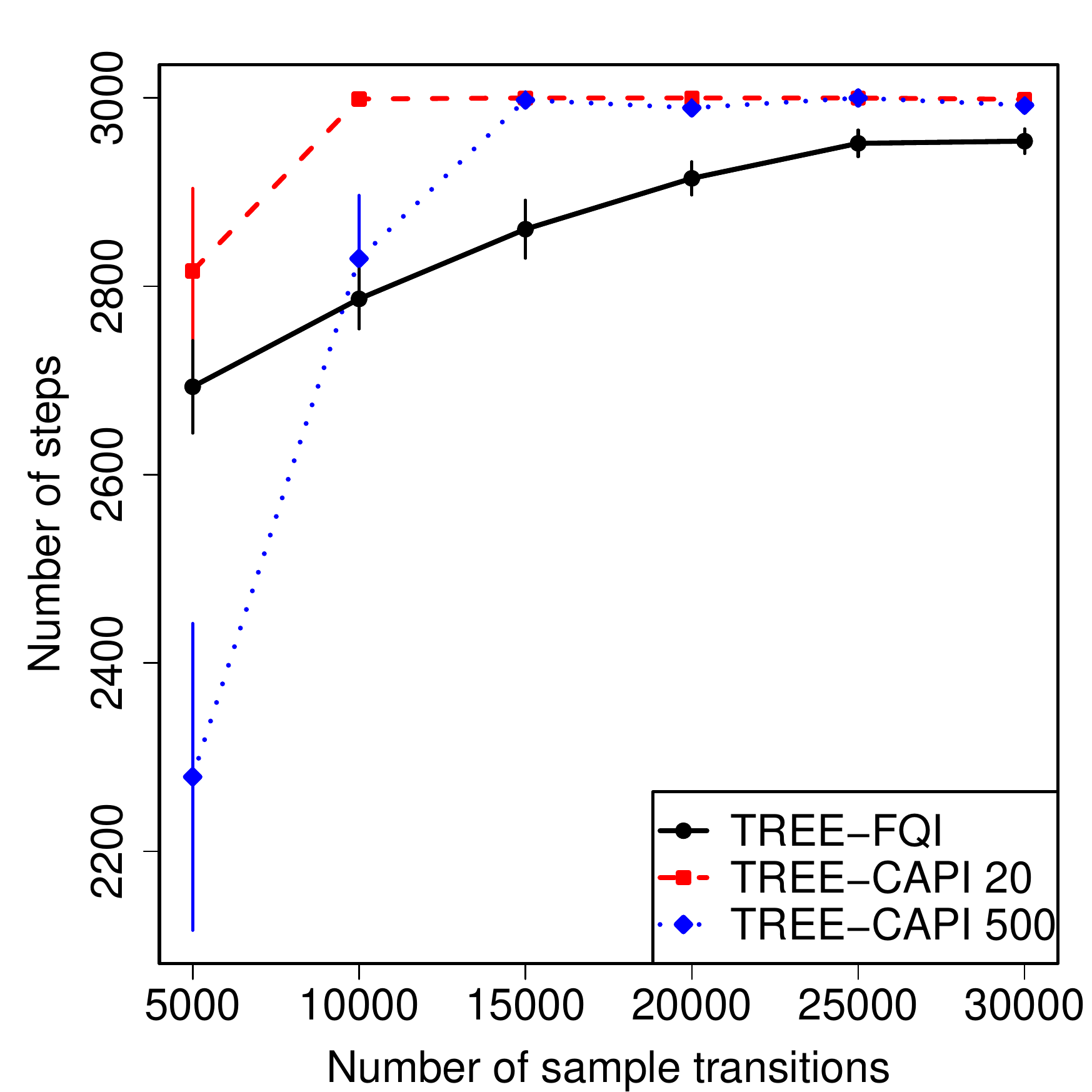}
    }
\subfigure[Return]
   {
   \centering
%   \label{fig:ConcentrabilityComparison}
   \includegraphics[width= 0.3\linewidth]{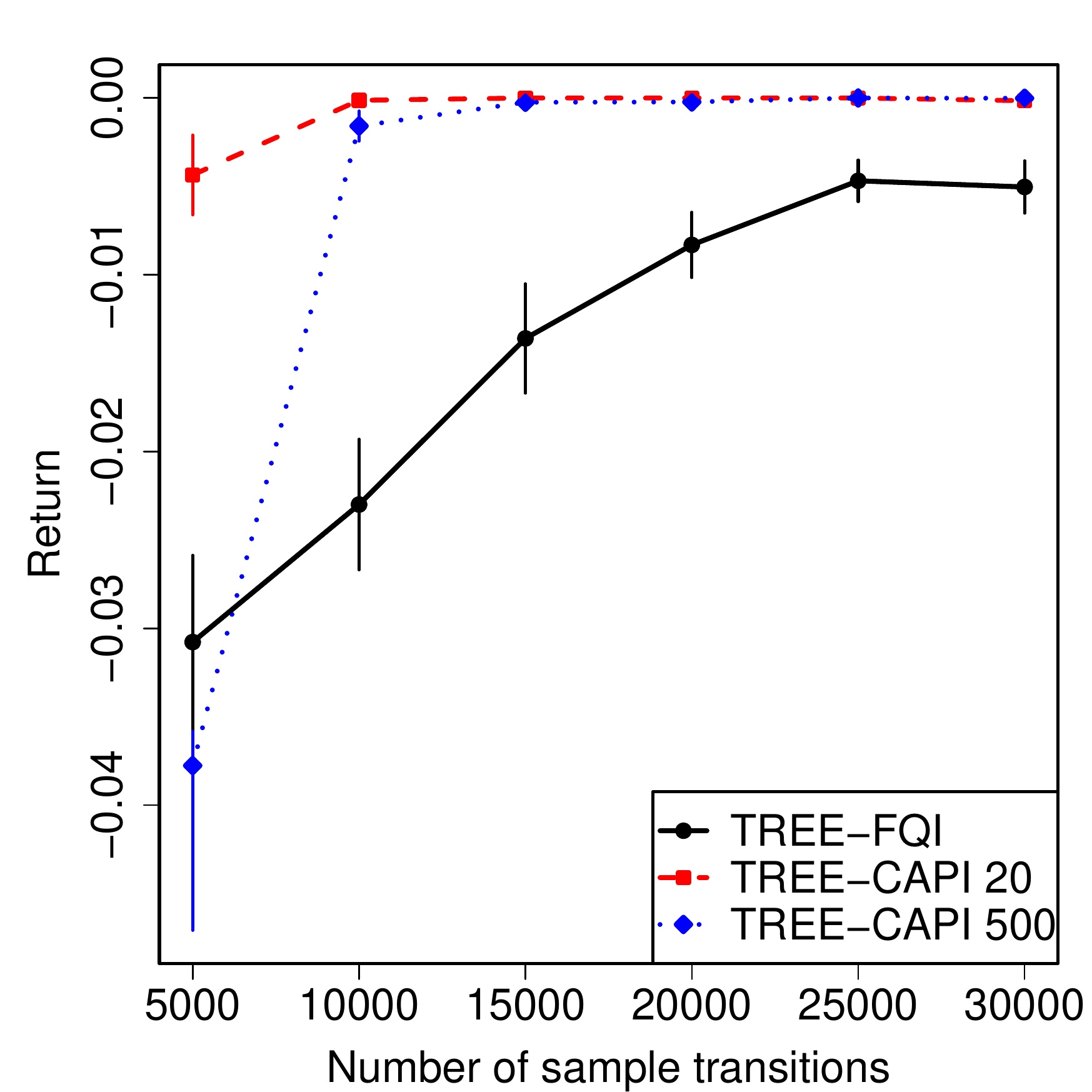}
   }	
   \subfigure[Time]
   {
   \centering
%   \label{fig:ConcentrabilityComparison}
   \includegraphics[width= 0.3\linewidth]{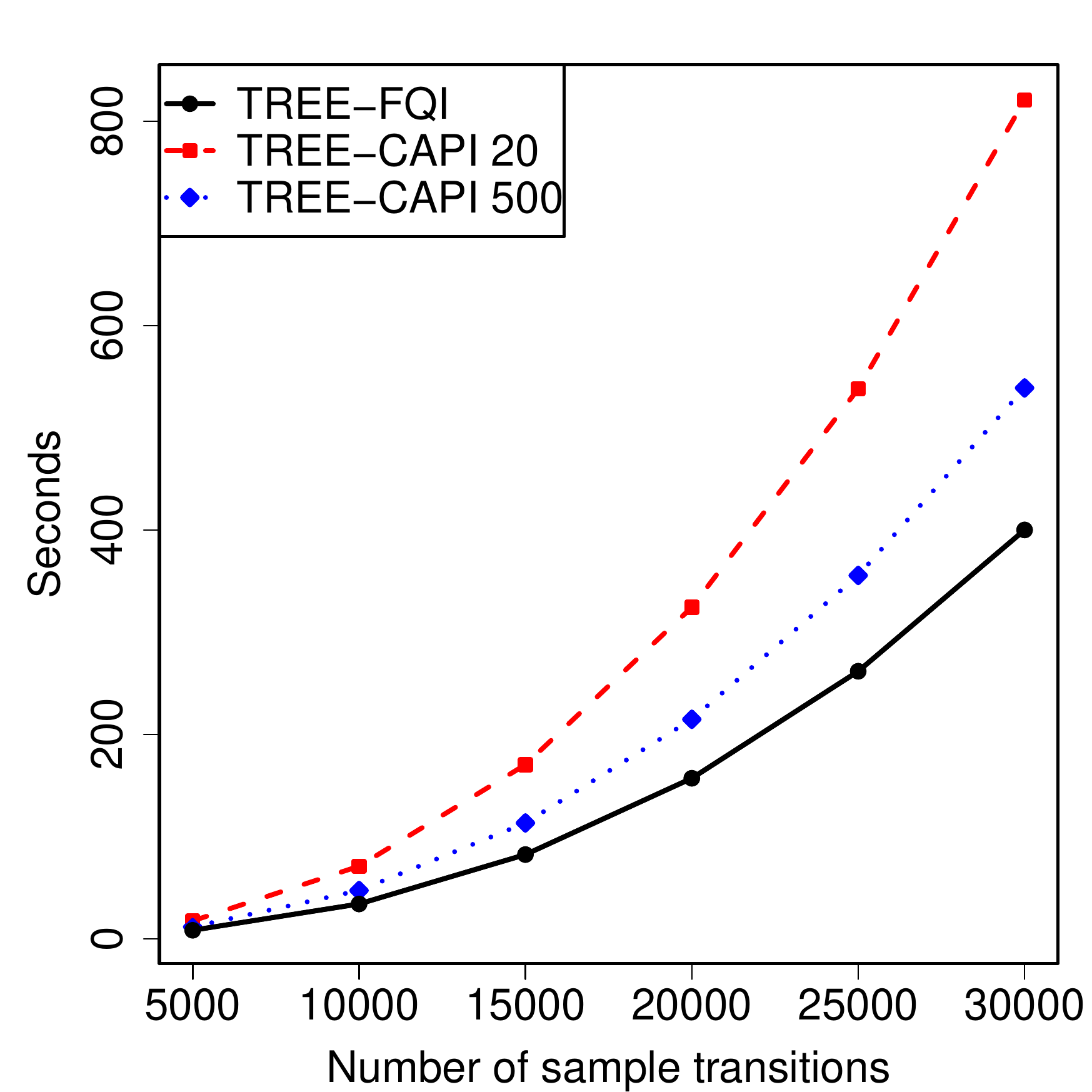}
   }
%    \vspace*{-0.5cm}  

% DPI vs CAPI   
\subfigure[Number of steps]{ %[$\mnea = 20$]{ 
\centering
   \includegraphics[width= 0.3\linewidth]{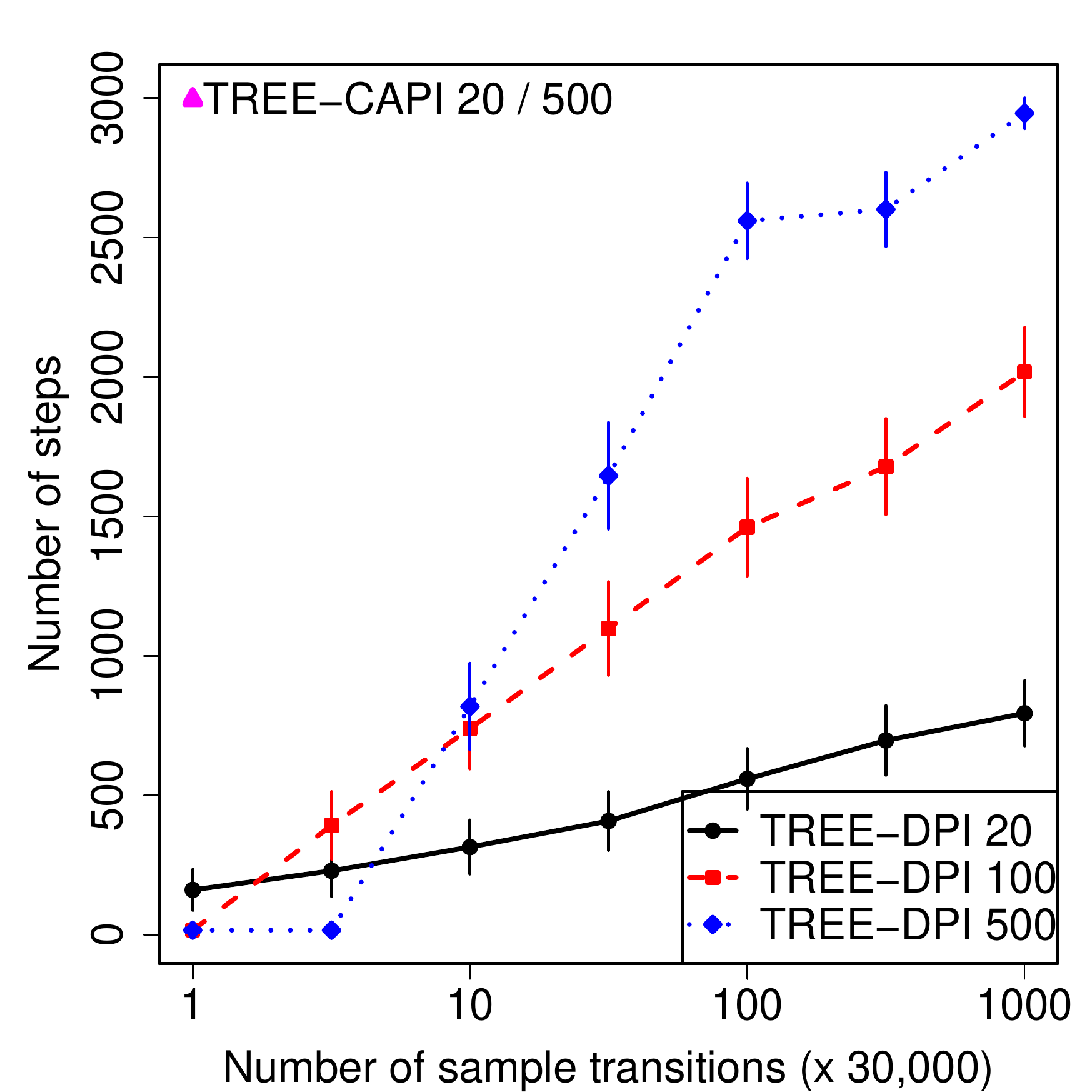}
    }
\subfigure[Return]
   {
   \centering
%   \label{fig:ConcentrabilityComparison}
   \includegraphics[width= 0.3\linewidth]{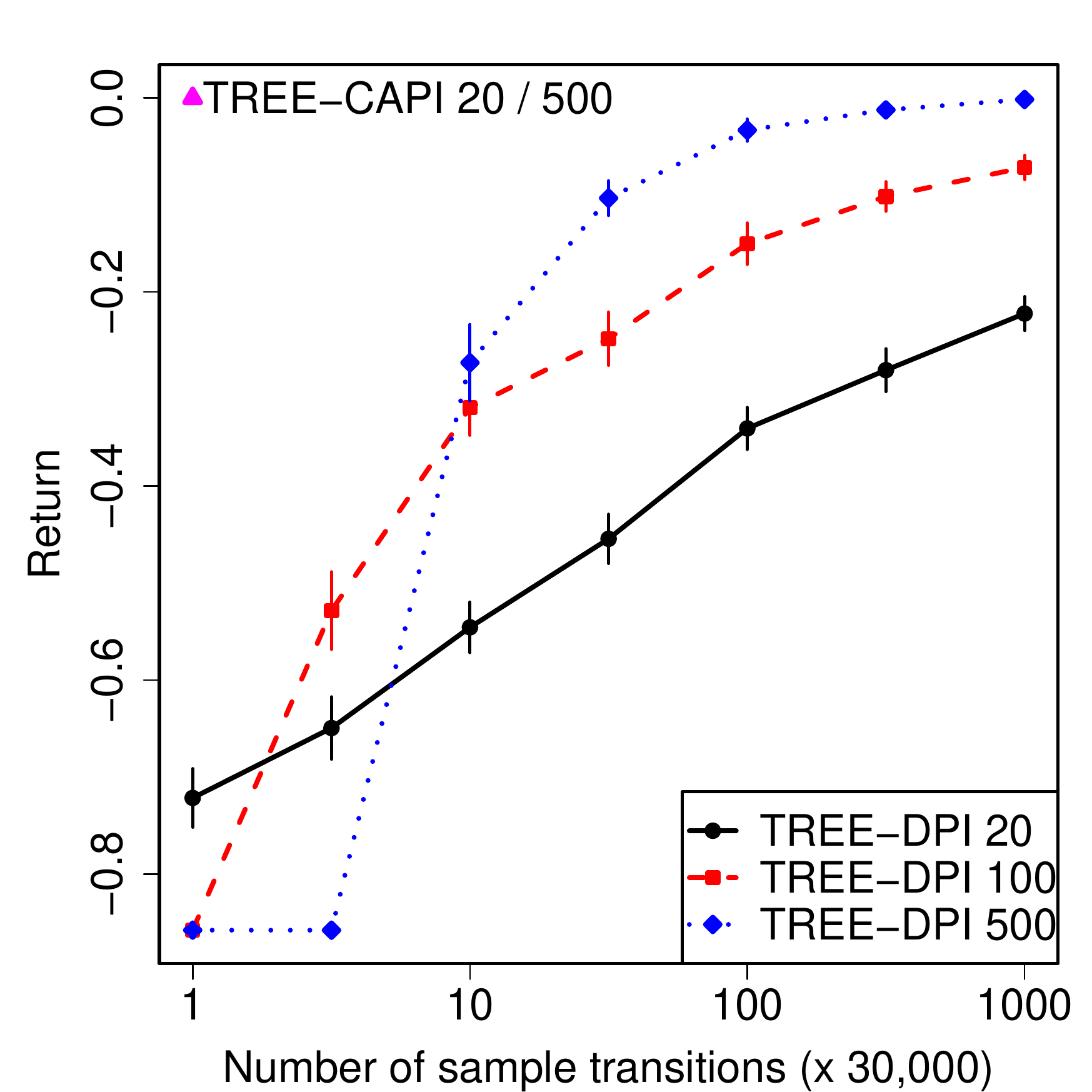}
   }	
   \subfigure[Time]
   {
   \centering
%   \label{fig:ConcentrabilityComparison}
   \includegraphics[width= 0.3\linewidth]{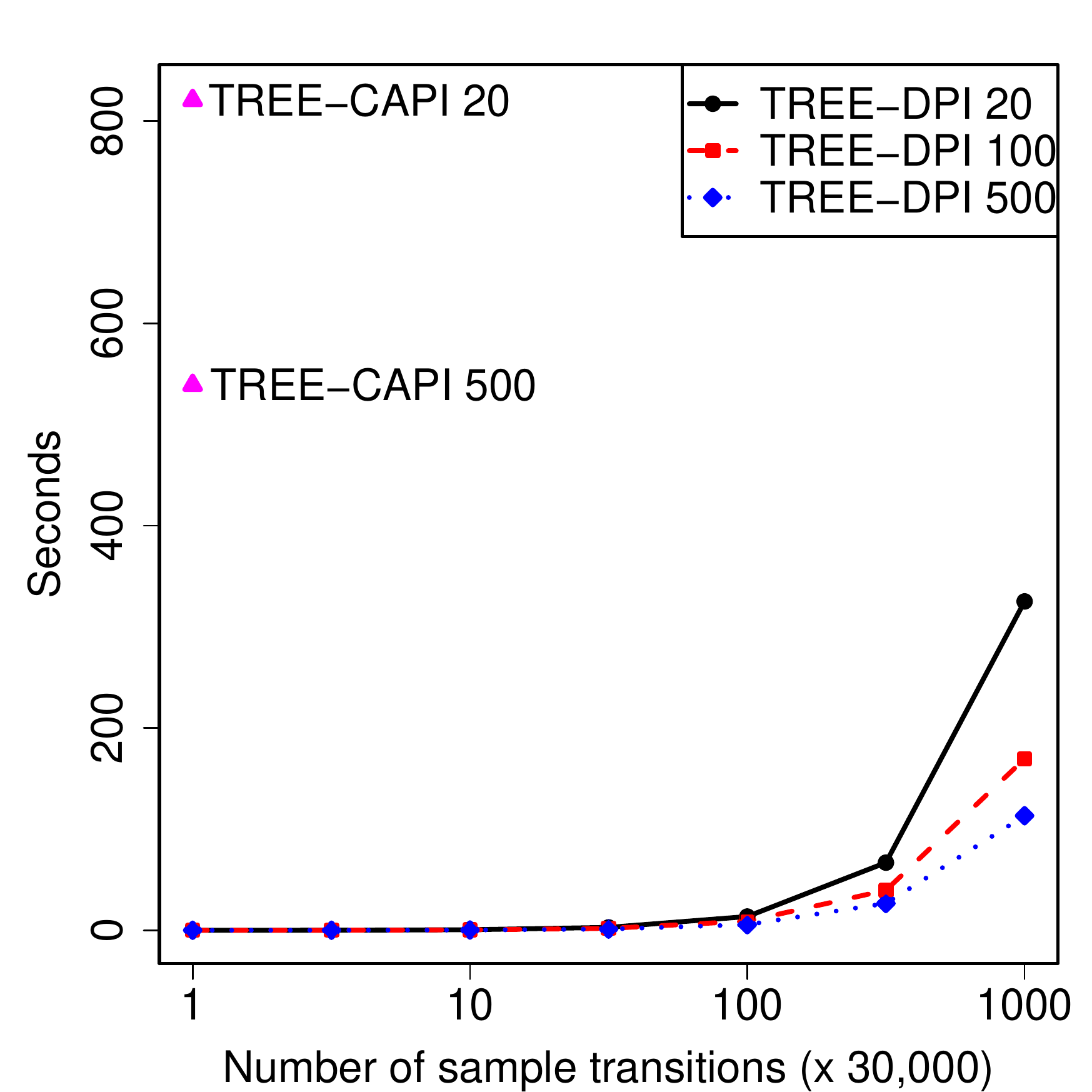}
   }
% \vspace*{-0.5cm}  
   
 \caption{(Pole Balancing) Comparing the expected number of (a) number of steps, (b) return in each episode, and (c) computation time for Tree-based CAPI vs. Tree-based Fitted Q-Iteration as a function of sample size.
Graphs (d), (e), and (f) compare Tree-based DPI (rollouts for PolicyEval) with Tree-based CAPI (Tree-FQI for PolicyEval). The error bars show one standard error over 50 runs.}
\label{fig:CAPI-PoleBalancing-TreeCAPIvsTreeFQI-Main}
\end{figure*}

We now turn to the pole-balancing problem, which is particularly interesting in the context of CAPI as it is possible to achieve good performance in this task with relatively simple policies~\cite{Wieland1991}.

For the pure value-based approach, we use Extra Trees Fitted Q-Iteration~\citep{Ernst05} (denoted by Tree-FQI), a state-of-the-art approximate value iteration algorithm, with the choice of 30 trees and
the minimum number of points required to split a node as $\mnec=20$.
We compare this value-based method with Tree-CAPI that uses the same Extra Tree Fitted Q-Iteration algorithm for policy evaluation (which is used, as before, only for policy evaluation, without improvement) and represents policies by an ensemble of $30$ trees.
%The value function used by CAPI is estimated by the same Extra Tree Fitted Q-Iteration algorithm (which is used, as before, only for policy evaluation, without improvement). The policies are represented by an ensemble of $30$ trees. 
The Extra Trees algorithm was adopted to build the trees, but in this case using the estimated action-gap-weighted loss function. %  We denote this instantiation of CAPI by Tree-CAPI.
We chose the minimum number of points required to split a node from the set of $\mnea \in \{20, 500 \}$; note that this number controls the complexity of the policy space.
We also compare performance with an instantiation of the DPI algorithm, which we call Tree-DPI.  It uses rollouts for policy evaluation (as any DPI algorithm) and Extra Trees to represent the policy (with $\mnea \in \{20, 100, 500 \})$.
The rollouts are single trajectories of length at most $50$. We also run Tree-DPI for $5$ iterations. The number of trajectories are chosen such that Tree-DPI uses the same amount of data as Tree-CAPI. Hence the difference between Tree-DPI and Tree-CAPI is only in the policy evaluation step.  In particular, Tree-DPI does not generalize the data using a value function, so it can only exploit policy regularities.
% The detail of experiments is presented in the appendix.

Figure~\ref{fig:CAPI-PoleBalancing-TreeCAPIvsTreeFQI-Main}
shows the expected number of steps balancing the pole (note that we stop the simulation after $3000$ successful steps). We also show the expected return per episode for all algorithms (averaged over $50$ independent runs) and the running time. Tree-CAPI clearly outperforms Tree-FQI (top row), especially when the number of samples is small and $\mnea$ is not very large.
Note that for the sample sizes larger than $15000$, Tree-CAPI with any choice of $\mnea$ we tried solved the task perfectly, in all experiments.
These results confirm that exploiting policy structure can lead to better and more stable results. 
The difference between the sample efficiency of both Tree-CAPI and Tree-FQI compared to Tree-DPI is dramatic: Tree-DPI-500 takes about $3\times 10^7$ samples to achieve the same performance that Tree-CAPI-20 achieved with about $10^4$ samples and Tree-FQI achieved with about $3 \times 10^4$ samples.
This comparison shows that value function regularities can also be exploited, especially when the number of samples is small.
Note that Tree-CAPI is computationally more costly than Tree-DPI. In practice, the choice of algorithm depends on the cost of collecting new samples vs. the cost of computation.  In problems when samples are expensive but computation can be done off-line, CAPI-style algorithms are a powerful choice, as illustrated in these results.

\fi

% use section* for acknowledgement
\section*{Acknowledgment}
This work is financially supported by the Natural Sciences and Engineering Research Council of Canada (NSERC).
%The authors would like to thank...

% Can use something like this to put references on a page
% by themselves when using endfloat and the captionsoff option.
\ifCLASSOPTIONcaptionsoff
  \newpage
\fi

% trigger a \newpage just before the given reference
% number - used to balance the columns on the last page
% adjust value as needed - may need to be readjusted if
% the document is modified later
%\IEEEtriggeratref{8}
% The "triggered" command can be changed if desired:
%\IEEEtriggercmd{\enlargethispage{-5in}}

% references section
\bibliography{MyBib(abbr)}

\end{document}